\documentclass{article}
\usepackage[margin=2cm]{geometry}
\usepackage{graphicx}

\usepackage[colorlinks=true, allcolors=black]{hyperref}
\usepackage{amsmath,amsthm,amsfonts,amssymb,amscd,mathrsfs}
\usepackage{enumerate}
\usepackage{bbm}
\usepackage{float}
\usepackage{fancyvrb}
\usepackage{booktabs}
\usepackage{subcaption}
\usepackage[T1]{fontenc}
\usepackage{xcolor}
\usepackage{enumitem}
\usepackage{setspace}

\usepackage[ruled,vlined]{algorithm2e}
\def\algbackskip{\hskip-\ALG@thistlm}
\newtheorem{theorem}{Theorem}[section]
\newtheorem{definition}[theorem]{Definition}

\newtheorem{proposition}[theorem]{Proposition}
\newtheorem{lemma}[theorem]{Lemma}

\newtheorem{notation}[theorem]{Notation}
\newtheorem{assumption}[theorem]{Assumption}
\newtheorem{example}[theorem]{Example}
\newtheorem*{theorem*}{Problem}
\theoremstyle{remark}
\newtheorem{remark}[theorem]{Remark}

\newcommand{\R}{{\mathbb R}}
\newcommand{\E}{{\mathbb E}}
\newcommand{\N}{{\mathbb N}}
\newcommand{\Z}{{\mathbb Z}}

\newcommand{\g}{{\mathcal L}}

\DeclareMathOperator{\err}{err}

\newcommand{\inprod}[2]{\left\langle #1, #2 \right\rangle}
\newcommand{\nor}[1]{\left\lVert#1\right\rVert}
\newcommand{\modu}[1]{\left\lvert #1 \right\rvert}
\newcommand{\paren}[1]{\left( #1 \right)}
\newcommand{\brac}[1]{\left[ #1 \right]}
\newcommand{\curlbrac}[1]{\left\{ #1 \right\}}

\let\oldtabular\tabular 
\renewcommand{\tabular}{\footnotesize\oldtabular}
\let\oldcaption\caption 
\renewcommand{\caption}{\footnotesize\oldcaption}

\title{Deep Hilbert--Galerkin Methods for Infinite-Dimensional PDEs and Optimal Control}

\date{\today}
\author{\begin{large}
Samuel N. Cohen\footnote{Mathematical Institute, University of Oxford. Email: cohens@maths.ox.ac.uk.},\quad Filippo de Feo\footnote{Principal author; Institut für Mathematik, Technische Universität Berlin. Email: defeo@math.tu-berlin.de.},\quad Jackson Hebner\footnote{Principal author; Mathematical Institute, University of Oxford. Email: hebner@maths.ox.ac.uk.},\quad Justin Sirignano\footnote{Mathematical Institute, University of Oxford. Email: justin.sirignano@maths.ox.ac.uk.}\end{large}}
\begin{document}

\maketitle

\begin{abstract}
We develop deep learning-based  approximation methods for fully nonlinear second-order PDEs on separable Hilbert spaces, such as HJB equations for infinite-dimensional control, by parameterizing solutions via Hilbert--Galerkin Neural Operators (HGNOs).  We prove the first Universal Approximation Theorems (UATs) which are sufficiently powerful to address these problems, based on novel topologies for Hessian terms and corresponding novel continuity assumptions on the fully nonlinear operator. These topologies are non-sequential and non-metrizable, making the problem delicate. In particular, we prove UATs for functions on Hilbert spaces, together with their Fréchet derivatives up to second order, and for unbounded operators applied to the first derivative, ensuring that HGNOs are able to approximate all the PDE terms. For control problems, we further prove UATs for optimal feedback controls in terms of our approximating value function HGNO.

We develop numerical training methods, which we call Deep Hilbert--Galerkin and Hilbert Actor-Critic (reinforcement learning) Methods, for these problems by minimizing the $L^2_\mu(H)$-norm of the residual of the PDE on the whole Hilbert space, not just a projected PDE to finite dimensions. This is the first paper to propose such an approach. The models considered arise in many applied sciences, such as functional differential equations in physics and Kolmogorov and HJB PDEs related to controlled PDEs, SPDEs, path-dependent systems, partially observed stochastic systems, and mean-field SDEs. We numerically solve examples of Kolmogorov and HJB PDEs related to the optimal control of deterministic and stochastic heat and Burgers' equations, demonstrating the promise of our deep learning-based approach.
\end{abstract}



\begin{small}
\tableofcontents
\end{small}
\section{Introduction}
\subsection{Model problem}
\paragraph{PDEs on Hilbert spaces.}
Let $(H,\inprod{\cdot}{\cdot})$ be a separable Hilbert space and consider the second-order fully-nonlinear partial differential equation (PDE) on an open domain  $O\subset H$
\begin{align}\label{eq:PDE_intro}
\begin{cases}
    \inprod{LDv}{x} + F(x,v,Dv,D^2v)= 0, \quad  x \in O,\\
    v(x)=g(x),\quad x \in \partial O,
\end{cases}
\end{align}
where $L \colon D(L)\subset H \to H$ is a closed, densely-defined, possibly unbounded linear operator; $F \colon  H \times \mathbb R^k \times H^k \times S(H)^k\to \mathbb R^k$;  $Dv:H\to H^k$ and $D^2v:H\to S(H)^k$ are Fréchet derivatives (here $S(H)$ denotes symmetric bounded operators on $H$), and $\langle L (\cdot),x \rangle$ is applied to each component of $Dv$.

In this manuscript, we develop the first numerical methods  for classical solutions of  infinite-dimensional fully-non linear PDEs of the form \eqref{eq:PDE_intro} via Derivative-Informed Operator Learning. Let $\{e_i\}_{i=1}^\infty \subset D(L)\subset  H$ be an orthonormal basis. Then we can represent $x=\sum_{i\in \mathbb N}x_ie_i$, where $x_i:=\langle x,e_i \rangle$. Taking for simplicity $k=1$, $O=H$, and given a classical solution\footnote{That is, $v \in C^{2}(H)$, $LD v \in C^0(H;H)$, and $v$ satisfies \eqref{eq:PDE_intro}.} $v:H\to \mathbb R$, we can write $v(x)=\tilde v((x_i)_{i=1}^\infty)$. The Fréchet derivatives are then
\begin{small}
\begin{align*}
&Dv(x)=\sum_{i=1}^\infty \partial_i \tilde v\left((x_i)_{i=1}^\infty\right) e_i,\quad D^2v(x)h=\sum_{i=1}^\infty\left[\sum_{j=1}^\infty\partial^2_{ij}  \tilde v\left((x_i)_{i=1}^\infty\right)h_j \right ]e_i.
\end{align*}
\end{small}
With this expansion, equations of the form \eqref{eq:PDE_intro} can formally be seen as PDEs in infinitely many variables which are highly challenging equations due to their intrinsic infinite-dimensionality and the unboundedness of $L$. In our approach, we will parameterize the solution $v$ of these PDEs with a Hilbert--Galerkin Neural Operator (HGNO).  

  \paragraph{Neural Operators.}  We represent non-linear operators between separable Hilbert spaces via a neural operator. We choose a Hilbert--Galerkin Neural Operator (HGNO)\footnote{Castro \cite{castro2022} refers to this architecture as a DeepONet or Deep-H-ONet. However, to highlight the difference with the DeepONet as introduced in \cite{lu2020deeponet}, we prefer the term Hilbert--Galerkin Neural Operator. Indeed, the DeepONet in \cite{lu2020deeponet} is designed for operators on spaces of continuous functions and uses point-evaluations at fixed sensor points $z^1, \ldots, z^d$ in the spatial domain as inputs. By contrast, in \cite{castro2022}, the inputs are $(\inprod{x}{e_i})_{i=1}^d$ for some fixed orthonormal basis $\{e_i\}_{i=1}^\infty$ of a Hilbert space, a crucial feature for our analysis.} \cite{castro2022}, i.e.~an encoder-decoder type architecture of the form
\begin{equation*}
       f^{d,\theta,p} : H_1 \to H_2,\quad      f^{d,\theta,p}(x) = \left(\hat{\mathcal E}_p^{H_2} \circ \tilde f^{d,\theta,p} \circ \mathcal{E}_d^{H_1}\right)(x)=\sum_{j=1}^p \tilde f^{d,\theta,p}_j\Big((\inprod{x}{e_i})_{i=1}^d\Big) g_j ,
    \end{equation*}
where $\{e_i\}_{i=1}^\infty \subset H_1$ and $\{g_i\}_{i=1}^\infty \subset H_2$ are orthonormal bases of $H_1$, $H_2$, respectively, $\mathcal{E}_d^{H_1} : H_1 \to \R^d$, $ \mathcal{E}_d^{H_1}(x) =  (\inprod{x}{e_i})_{i=1}^d$ is  the coordinate operator from $H_1$,
        $\hat{\mathcal E}_d^{H_2} : \R^d \to H_2 $, $\hat{\mathcal{E}}_d^{H_2}((x_i)_{i=1}^d) = \sum_{i=1}^d x_i g_i$ is  the  embedding operator onto $H_2$, and $\theta $ is the parameters of a trainable  neural network $\tilde f^{d,\theta,p}:\mathbb R^d\to \mathbb R^p$.
        \begin{itemize}[leftmargin=*,nosep]
            \item We parameterize the PDE solution $v\in C^2(H)$ via an HGNO $v^{d,\theta}=v^{d,\theta,1} : H\to \mathbb R$. In order to accurately represent the solution of the PDE, we need to simultaneously learn $v$ and its Fréchet derivatives $Dv\in C(H;H)$, $D^2v\in C(H;S(H))$ via a single neural operator. This is the setting of Derivative-Informed Operator Learning because it relies in particular on accurately representing $D^2v$, the Fréchet derivative of the non-linear operator $Dv$.
            \item In optimal control problems (see Subsection \ref{subsec:Motations}), we also  parameterize the optimal control operator $u : H \to \tilde{U}$ between Hilbert spaces via an HGNO $u^{d,\theta,p} : H \to \tilde{U}$. The neural operator $u^{d,\theta,p}$ will be trained, i.e.~informed, via $v^{d,\theta}, Dv^{d,\theta},D^2v^{d,\theta}.$ 
        \end{itemize}
        The use of neural operators enables us to solve \eqref{eq:PDE_intro} directly on $H$ rather than a projected PDE on $P_d(H)$ (cf.~Remark \ref{rem:def_N}).
  \paragraph{Contributions.}  In brief, the contributions of this paper are:
\begin{enumerate}
    \item In general, we cannot expect the $D^2v^{d,\theta}$ to approximate $D^2v$ in the operator norm (see Remark \ref{rem:density_neural_compact_open_topology}). Therefore, we identify the natural topologies on $C^2(H)$ to achieve universal approximation in full generality (Theorems \ref{th:UAT_K} and \ref{th:UAT_L2_noL}), the first general results of this kind. These topologies are non-sequential and non-metrizable in general, making the problem delicate.
    \item We prove the first universal approximation results for solving the PDE \eqref{eq:PDE_intro} on the whole space $H$. In particular, we show that HGNOs are capable of solving \eqref{eq:PDE_intro} to arbitrarily low $L^2_\mu$-norm of the residual (Theorems \ref{th:DGM_residual} and \ref{th:l2_convergence}). Since continuity assumptions on $F$ are usually given in operator norm, this requires a new continuity assumption on $F$ and universally approximating the term involving the unbounded operator $L$.
    \item We specialize our results to optimal control problems on Hilbert spaces, by showing that HGNOs $v^{d,\theta}$ can solve the HJB equation and provide universal approximation of optimal  feedback controls in terms of $Dv^{d,\theta},D^2v^{d,\theta}$, the first results of this kind. 
    \item We leverage the theory to develop the first numerical schemes for these fully nonlinear second-order PDEs {directly} on $H$ {(i.e.~not a projected PDE on $P_d(H)$)}, which we call Deep Hilbert--Galerkin Methods. These work by minimizing the $L^2_\mu$-norm of the PDE residual with variants of gradient descent.
    \item We introduce Optimize-then-Learn methods and develop Hilbert Actor-Critic Methods for optimal control problems. These are the first methods to solve the PDE \eqref{eq:PDE_intro} on the whole Hilbert space $H$.
\end{enumerate}

We will not focus in this paper on the training-time convergence of the numerical approximation to the PDE solution when trained under gradient descent methods (for this in finite dimensions, see \cite{cohen2023, jiang2023global}). We note, however, that sufficiently strong universal approximation theorems are at the core of these convergence results. We will see that empirically, our proposed algorithms converge well when trained using standard choices of parameters.

\subsection{Motivation and examples}\label{subsec:Motations}
These fundamental PDEs have challenged many prominent scientists for nearly a century and remain a longstanding computational open problem. They arise throughout an extraordinary range of applied sciences, for instance, in the form of the celebrated families of
\begin{enumerate}
    \item \textbf{Kolmogorov PDEs}, associated to  stochastic evolution equations on $H$ \cite{da1992stochastic,da2002second},
    \item \textbf{Hamilton-Jacobi-Bellman (HJB) PDEs}, associated to optimal control problems of deterministic  and stochastic evolution equations on $H$ \cite{crandall_lions_infinite_dim_II,crandall1991viscosity,crandall_lionsVII,fabbri2017book,li2012optimal,lions1988viscosity,lions1989viscosity,lions2006viscosity} and in mean-field control \cite{gangbo_mayorga_swiech,swikech2025finite,defeo_gozzi_swiech_wessels},
    \item \textbf{Isaacs PDEs} or systems of HJB equations, associated to differential games of deterministic  and stochastic  evolution equations on $H$ \cite{crandall_lionsIII,fleming_nisio,nisio1998},
        \item \textbf{Functional differential equations (FDEs)} in  physics\footnote{In physics the terminology differs from that typically used in applied mathematics, where  a functional differential equation is typically a path-dependent differential equation.}, such as the Hopf equation in turbulence theory \cite{hopf1952statistical,monin2013statistical} or the Schwinger–Dyson equation in quantum field theory \cite{peskin2018introduction}.
\end{enumerate}
\begin{example}[Functional differential equations in physics]
One of the most famous FDEs  in physics is the \textbf{Hopf equation} \cite{hopf1952statistical,monin2013statistical} in turbulence theory, which encodes the statistical properties of velocity and pressure fields of the Navier--Stokes equations given statistical information of the random initial state. This is a $\mathbb C$-valued PDE in infinitely many variables on $H=L^2(\mathcal D;\mathbb R^3)$ of the form
    $$
\frac{\partial \Phi( t,x)}{\partial t}=\sum_{m=1}^3 \int_{\mathcal D} x^k(\xi)\left(i \sum_{j=1}^3 \frac{\partial}{\partial \xi^j} \frac{\delta^2 \Phi( t,x)}{\delta x^m(\xi) \delta x^j(\xi)}+v \nabla^2 \frac{\delta \Phi( t,x)}{\delta x^m(\xi)}\right) \mathrm{d} \xi, \quad (t,x)\in (0,T)\times H,
$$
which governs the dynamics of the characteristic functional
$
\Phi(t,x)=\mathbb{E}\left[\exp \left(i \int_{\mathcal D}  u(t,\xi) \cdot x(\xi) \mathrm{d} \xi\right)\right] ,
$
where  $u(t,\xi)$ solves the Navier--Stokes equation with random initial state $u(0,\xi)$. 
\end{example}
\paragraph{Kolmogorov PDEs and Hamilton--Jacobi--Bellman PDEs on Hilbert spaces.} Our methods will be developed for  general PDEs of the form \eqref{eq:PDE_intro}. Moreover, we will dedicate special attention to Kolmogorov PDEs and Hamilton--Jacobi--Bellman PDEs on $H$. 
To this purpose, consider the family of  controlled
stochastic evolution equations on $H$
\begin{align}\label{eq:stateSDE_H_intro}
    dX_t &= [AX_t+ b(X_t,u_t)]dt + \sigma(X_t,u_t)dW^Q_t, \quad  X_0 = x \in H,
\end{align}
where $A \colon D(A)\subset H \to H$ is the generator of a $C_0$-semigroup on $H$, $b$ and $\sigma$ are the drift and diffusion, and $u \in \mathcal U$ is an admissible control process. Here  $\Xi$ and $\tilde U$ are separable Hilbert spaces, the control takes values in $U\subset \tilde U$, and $W^Q$ is a Wiener process with covariance operator $Q$.

In optimal control theory a standard goal is to minimize, over all admissible controls $u \in \mathcal U$, a functional 
\begin{align}\label{eq:functional_intro}
 J(x;u) := \E\brac{\int_0^\infty e^{-\gamma t} l(X^{x,u}_t,u_t)dt },
    \end{align}
where $\gamma>0$, $l:  H \times U \to \mathbb R$. Following the  dynamic programming approach \cite{fabbri2017book}, we study the HJB equation, i.e.~the second-order fully-nonlinear  PDE on $H$ (with $k=1$)
\begin{align}\label{eq:HJB_control_intro}
    -\gamma v+  \inprod{A^*Dv}{x}  + \inf_{u\in U }\curlbrac{\inprod{Dv}{b(x,u)}+\frac{1}{2}\mathrm{Tr}[\sigma(x,u) Q \sigma^*(x,u)D^2v]+l(x,u)} &= 0,\quad x\in H.
\end{align}
This typically leads to a characterization of the value function as the unique solution of the HJB equation and to the construction of optimal feedback controls.
    The following are important special cases or extensions:
    \begin{itemize}
       \item when $Q=0$ the problem reduces to a \textbf{deterministic  control problem}.  In this case the HJB reduces to a first order PDE on $H$, see Remark \ref{rem:detertministic}. 
        \item when $\mathcal U$ is a singleton, then the HJB equation \eqref{eq:HJB_control_intro} reduces to a \textbf{Kolmogorov PDE}, see \eqref{eq:kolmogorov}.
        \item In the finite-horizon case, the HJB equation becomes a time dependent backward PDE.
        \item In
        zero-sum differential games of deterministic or stochastic evolution equation, the HJB equation is generalized to the  \textbf{Isaacs} equation \cite{crandall_lions_infinite_dim_II,fleming_nisio,nisio1998}.
    \end{itemize}
\begin{example}[PDEs and SPDEs]\label{ex:heat_eq_intro}
    A particularly informative  example that we will analyze is the optimal control of the deterministic and stochastic heat equation, i.e.~on a domain $\mathcal D\subset \mathbb R^n$, the SPDE
    \begin{align}
    \label{eq:stochastic_heat_intro}
        &\frac{\partial x}{\partial t}(t,\xi) = \Delta_\xi x(t,\xi) + u(t,\xi) + \sigma \frac{\partial^2 W^Q}{\partial t \partial \xi}(t,\xi) , \quad (t,\xi)\in \mathcal D,\end{align}
    with boundary conditions $x(t,\xi)  = 0, $ $x\in \partial \mathcal D$, initial conditions $ x(0,\cdot)= x \in L^2(\mathcal D)$, and control process $u \in \mathcal U$.
The goal is to minimize  a cost functional  $J(x;u) = \E\brac{\int_0^\infty e^{-\gamma s}\int_{\mathcal D}l^0(x(s,\xi)-\overline{x}(\xi), u(s,\xi))d\xi ds} ,$ where $\overline{x}(\xi)$ is the target temperature and $l^0:\mathbb R^2\to \mathbb R$ is a suitable cost density function depending on the objective of the controller. A standard choice is $l^0(y,a)=y^2+a^2$, but other choices are possible.  This problem can be rewritten on $H=L^2(\mathcal D)$ in the form \eqref{eq:stateSDE_H_intro}-\eqref{eq:functional_intro}, so the corresponding HJB equation is of the form \eqref{eq:HJB_control_intro}. When $Q=0$ the problem reduces to control of the heat equation.
\end{example}
\begin{example}[Path-dependent systems]\label{ex:volterra} Consider a controlled Volterra stochastic integral equation (SVIE) on $\mathbb R^n$
 $$y(t)=z(t)+\int_0^tK(t-s) b^0\left(y(s) ,u(s) \right)d s+\int_0^tK(t-s) \sigma^0 \left(y(s),u(s) \right)d W_s, 
$$
where 
$z(\cdot) $ is an initial curve, $K(\cdot)$ is a kernel, $W_s$ is a Wiener process, and $u:\Omega \times [0,\infty) \to U\subset \mathbb R^n$ is an admissible control process. The cost functional is of the form
$
 J(x;u) = \mathbb{E}\left[\int_0^\infty e^{-\gamma t}l^0\left(y(t), u(t)\right) d t\right].
$
Due to the non-Markovianity of the SVIE \cite{friesen2025failure}, the standard dynamic programming arguments on $\mathbb R^n$ do not apply in general. However,  the  problem can be rewritten  as the optimal control of a  Markovian SDE of the form \eqref{eq:stateSDE_H_intro}-\eqref{eq:functional_intro} on a suitable Hilbert space $H$ via {Markovian Lifts}, cf.~\cite{bolli_defeo,hamaguchi2005,possamai_mehdi} (and \cite{gasteratos_pannier} for the Kolmogorov PDE case). When $\sigma^0=0$ we have a deterministic control problem. Other path-dependent systems that fit into the above setting are stochastic control problems with \textbf{delays} in the state and/or in the control, cf.~\cite{defeo_phd,defeo_federico_swiech,fabbri2017book} and references therein. 
\end{example}
\begin{example}[Partially observed systems]
Consider the controlled system
    \begin{align*}
        dy(t) &= b^1(y(t),a(t))dt + \sigma^1(y(t),a(t))dW^1(t) + \tilde \sigma^1(y(t),a(t))dW^2(t), &&y(0) = \eta \in L^2( \mathcal F_0;\mathbb R^n) \\
        dz(t) &= b^2(y(t))dt + dW^2(t), &&y_1(0) = 0.
    \end{align*}
    where  $y(\cdot) \in \R^n$ is the hidden state and $z(\cdot) \in \R^m$ is the observation. The goal is to minimize, over all $\mathcal F^z_s$-adapted controls $a_s$, a functional of the form
   $I(t,\eta;a) = \E\brac{\int_0^\infty l_1(y(t),a(t))dt }.$
    The fact that controls are only adapted to the filtration $\mathcal F^z_s$ given by observations makes the problem extremely difficult \cite{CohenKnochenhauerMerkel2025}. One way of dealing with it is through the so-called ``separated'' problem where one is led to control the unnormalized conditional probability density of the state process $y(\cdot)$ given the observation $z(\cdot)$, leading to a (fully observable) optimal control problem of the Duncan--Mortensen--Zakai (DMZ) stochastic evolution equation on $H=L^2(\R^n)$, see  \cite[Section 2.6.6]{fabbri2017book}, \cite{lions2006viscosity,pardoux2006, cox2025measurevaluedhjbperspectivebayesian}.
\end{example}

\begin{example}[Mean-field control]\label{ex:MFC}
Consider the HJB equation on the Wasserstein space $\mathcal{P}_2\left(\mathbb{R}^n\right)$ arising in optimal control of particle systems and mean-field SDEs with common noise:
\begin{equation}\label{eq:wasserstein_intro}
    \left\{\begin{array}{l}\partial_t \mathcal{V}(t, \mu)+\frac{1}{2} \int_{\mathbb{R}^n} \operatorname{Tr}\left[\partial_y \partial_\mu \mathcal{V}(t, \mu)(y) \tilde \sigma(y, \mu) \tilde \sigma^{\top}(y, \mu)\right] \mu(\mathrm{d} y) \\ \quad+\frac{1}{2} \int_{\left(\mathbb{R}^n\right)^2} \operatorname{Tr}\left[\partial_\mu^2 \mathcal{V}(t, \mu)\left(y, y^{\prime}\right) \tilde  \sigma(y, \mu) \tilde \sigma^{\top}\left(y^{\prime}, \mu\right)\right] \mu(\mathrm{d} y) \mu\left(\mathrm{d} y^{\prime}\right) \\ \quad-\int_{\mathbb{R}^n} \mathcal F\left(y, \mu, \partial_\mu \mathcal{V}(t, \mu)(y)\right) \mu(\mathrm{d} y)=0, \qquad\qquad (t, \mu) \in(0, T) \times \mathcal{P}_2\left(\mathbb{R}^n\right),  \end{array}\right.
    \end{equation}
    with $\mathcal{V}(T, \mu)=\mathcal{G}(\mu), $ $\mu \in \mathcal{P}_2\left(\mathbb{R}^n\right)$.
A famous approach to deal with this PDE is via the ``lifting'' technique, due to P.L. Lions \cite{lions2007}. We lift the space $\mathcal{P}_2\left(\mathbb{R}^n\right) $ to a Hilbert space  of random variables $H:=L^2\left(\Omega ; \mathbb{R}^n\right)$, and then study the  `lifted HJB equation', i.e.~a PDE on $H$ of the form \eqref{eq:HJB_control_intro} (extended to the time dependent case) and the corresponding `lifted control problem' on $H$ of the form \eqref{eq:stateSDE_H_intro}-\eqref{eq:functional_intro}.
We refer to \cite{gangbo_mayorga_swiech,swikech2025finite} for  details of this procedure and to \cite{defeo_gozzi_swiech_wessels} for extensions to particle systems of stochastic evolution equations.
\end{example}
\subsection{Literature review}
\paragraph{PDEs and HJB on $H=\mathbb R^d$.} Finite-dimensional PDEs 
on  $H=\mathbb R^d$ have been extensively studied from both analytical and numerical perspectives. Classical numerical approaches are based on finite difference schemes or variational discretizations (most notably Galerkin methods).
Despite their success, these methods suffer from the well-known curse of dimensionality, as their computational complexity grows exponentially with the space dimension $d$, severely limiting their applicability to high-dimensional problems. In the past decade, machine learning (ML) methods have emerged as a prominent class of approaches for tackling high-dimensional PDEs. Their mathematical foundations can be traced back to the expressive power of neural networks, as formalized by Universal Approximation Theorems (UATs), i.e.~neural networks are capable of approximating smooth functions and their derivatives,  uniformly on compact sets and in Sobolev norms \cite{hornik1991approximation}. 

Arguably, the two most popular methods are the Deep Galerkin Method (DGM) \cite{sirignano2018} and Physics-Informed Neural Networks (PINNs) \cite{raissi2019physics}, both of which work by having a neural network serve as the PDE solution ansatz and training it to minimize the $L^2$-norm of the PDE residual. This approach has proven popular, and, for example, the more recent papers \cite{cohen2023} and \cite{al_aradi_2022} have introduced variants of DGM/PINNs specialized for $L^2$-monotone PDE operators, Fokker--Planck equations, and HJB equations. Showing convergence of these algorithms is a difficult matter: see \cite{cohen2023, jiang2023global}, which consider specific classes of PDEs with precompact domain. Beyond the DGM/PINNs framework, there also exist the Deep BSDE solver \cite{han_jentzen_e,han_jantzen_e2018} and Deep Backward dynamic programming \cite{hure_pham_warin2020} methods, which are based on backward stochastic differential equation (BSDE) formulations of the PDE. There are also  methods developed specially for solving HJB equations, such as \cite{cheridito2025deep, zhou2021, ito2021}. For a detailed review of ML applied to high-dimensional control problems and HJB equations on $H=\mathbb R^d$, see \cite{hu_lauriere}.

\paragraph{Operator Learning and Derivative Informed Operator Learning.} Over the last five years, operator learning has become a central paradigm for learning operators on infinite-dimensional spaces. Standard architectures include the DeepONet \cite{lu2020deeponet, lanthaler2022}, PCA-Net \cite{bhattacharya2021, hesthaven2018, lanthaler2023}, Hilbert--Galerkin Neural Operator (or Deep-H-ONet) \cite{castro2022}, Fourier Neural Operator \cite{kovachki2023neuraloperator}, Laplace Neural Operator \cite{cao2024}, Spectral Neural Operator \cite{fanaskov2024}, and Convolutional Neural Operator \cite{raonic2023cno}. These techniques have been extensively used to learn maps from spaces of PDE parameters and/or forcing functions to the solutions of the corresponding finite-dimensional PDEs. Their use for these tasks is justified theoretically by a series of UATs on infinite-dimensional spaces\footnote{We also refer to \cite{cuchiero_primavera_svaluto,ceylan_kwossek_promel,hager-harang-pellizari-tindel,lyons_nejad_perez} for signature learning and UATs, and to \cite{pham_warin} for operator learning on Wasserstein spaces.}.

Even more recently, the subfield of Derivative-Informed Operator learning has been rapidly growing \cite{cao_chen_brennan_olearyroseberry-marzouk-youssef-ghattas,cao-Roseberry-Ghattas,go_chen,luo-Roseberry-chen-Ghattas,olearyrosemberry-villa-chen-ghattas,Roseberry-peng-villa-ghattas,qiu_bridges_chen,yao-luo-cao-Kovachki-Roseberry-ghattas}. Here, the goal is to learn  an operator and its (infinite-dimensional) derivatives. This dramatically improves training efficiency and generation cost, and  directly controls errors in the derivative
approximation, thereby enhancing the performance in a wide variety of downstream tasks, such
as Bayesian inverse problems, optimal design under uncertainty, and optimal experimental design. 

The literature on Derivative-Informed Operator learning is mainly empirical and UATs for operators and their infinite-dimensional derivatives are still not  well-understood. One UAT result for derivatives in infinite-dimensional spaces can be found in \cite{luo-Roseberry-chen-Ghattas} for operators $f$ belonging to Gaussian Sobolev spaces $W^{\kappa,2}(H;\mu)$,  $\mu \sim N(0,Q)$, whose derivatives are taken along directions in the Cameron–Martin space\footnote{That is, $\mathcal E := Q^{1/2}(H)$, and $D_{\mathcal E}^\kappa f$ denotes the $\kappa$-th Gaussian Sobolev (Malliavin) derivative restricted to $\mathcal E$.}, a result directly obtained from the definition of the space. More recently\footnote{This paper appeared during the last phases of the writing of the present article},  \cite{yao-luo-cao-Kovachki-Roseberry-ghattas} proved UATs via Fourier Neural Operators up to first order Fréchet derivatives on spaces $H=H^s\left(\mathbb{T}^d ; \mathbb{R}^{d_a}\right), Y=H^{r}\left(\mathbb{T}^d ; \mathbb{R}^{d_a}\right)$, $s,r\geq 0$, by further assuming  $D f(x) \in  \mathcal L_2 \left(H_\delta, Y\right)$ to be  Hilbert--Schmidt, from some possibly smoother space $H_\delta:=H^{s+\delta}\left(\mathbb{T}^d ; \mathbb{R}^{d_a}\right)$, $\delta \geq 0$ and approximation in the $\|\cdot\|_{\mathcal{L}_2\left(H_\delta, Y\right)}$-norm. There also exists an even more recent paper \cite{gong-luo-roseberry-etal} with UATs for first-order Frechét derivatives for multi-input neural operators (again, $Df$ is assumed to be Hilbert--Schmidt valued).

The use of Derivative-Informed Operator Learning in this paper differs from the existing literature in that we are not learning families of finite-dimensional PDEs, but rather the solution of a single, but infinite-dimensional, PDE.

\paragraph{Numerical schemes for PDEs on infinite-dimensional spaces.}
PDEs of the forms \eqref{eq:PDE_intro} and \eqref{eq:HJB_control_intro} have been extensively studied theoretically using different notions of solutions; see \cite{da2002second,fabbri2017book,li2012optimal} for an account of the literature. However, the development of numerical methods faces deep difficulties due to their intrinsic infinite-dimensionality (with the consequent curse of dimensionality in any possible type of numerical scheme) and challenges arising from the presence of unbounded (and therefore discontinuous) operators. For these reasons, only a few partial results are available. 
\begin{itemize}[leftmargin=*,nosep]
    \item The papers \cite{venturi_dektor,rodgers_venturi, venturi2018numerical} consider
    linear FDEs over specific compact subsets of $H$, typically suitable Sobolev balls, see e.g.~\cite[Examples 1, 2, p. 9]{venturi_dektor}. They fix an orthonormal basis $\{e_i\}_{i=1}^\infty \subset H$ and project the FDE to a $d$-dimensional PDE over a hypercube via cylindrical approximation. They then estimate the value of the PDE solution at discrete points in time via two different spectral tensor methods and from this construct an FDE solution estimate. The theoretical justification for this comes from \cite{venturi_dektor} for \textbf{first-order linear} FDEs, which proves that the projected FDE and its solution converge to the FDE on those specific compact subsets of $H$, where the unbounded operator $L$ is continuous\footnote{The case of PDEs on Banach spaces is also discussed, with similar techniques.}. As remarked in \cite{miyagawa2024physics} the previous methods are  limited in handling high-dimensions due to the curse of dimensionality. The paper \cite{miyagawa2024physics} considers this deep computational issue by numerically solving the linear projected FDE (on specific compact sets, as in \cite{venturi_dektor,rodgers_venturi, venturi2018numerical}) with PINNs \cite{raissi2019physics}, as they can benefit from the full power of mesh-free ML methods.
    \item As mentioned previously, an important theoretical contribution  came from \cite{castro2022}, which developed  ML methods to learn classical solutions of non-linear Kolmogorov PDEs (i.e.~semilinear PDEs) on Hilbert spaces of the form
    \begin{equation}\label{eq:kolmogorov_castro}
  \partial_t u(t, x)+(\mathcal{L}u)(t, x)+\tilde F\left(t, x, u(t, x), B^*(t, x) \nabla u(t, x)\right)  =0,  \quad  (t, x) \in[0, T] \times D(A).
    \end{equation}
    Here $(\mathcal{L}v)(t, x)=\langle Dv(t, x), A x+b(t, x)\rangle+\frac{1}{2} \operatorname{tr}\left(D^2 v(t, x)\left(\sigma(t, x) Q^{1 / 2}\right)\left(\sigma(t, x) Q^{1 / 2}\right)^*\right)$ is the infinitesimal generator of
    an (uncontrolled) SDE on Hilbert space (i.e.~\eqref{eq:stateSDE_H_intro} with $b=b(t,x),\sigma=\sigma(t,x)$).  Note that, in this form, \eqref{eq:kolmogorov_castro} is only well-defined over $D(A)$.   
    Using the Hilbert--Galerkin Neural Operator and proving appropriate UATs, \cite{castro2022} extends the Deep-Backward-Dynamic-Programming method of \cite{hure_pham_warin2020} via forward-backward SDEs to the case  where $H$ is a Hilbert space and $A$ is an unbounded operator\footnote{This method exploits the representation of the solution of the PDE via the non-linear Feynman--Kac formula and the corresponding forward-backward system of SDEs and shows that there are Hilbert--Galerkin Neural Operators that approximately solve the PDE.}. 
    However, to handle a PDE only defined on $D(A)$, the method is based on the  assumption (see \cite[Assumption 3.1]{castro2022}) that there exists an (analytically) strong solution of the abstract SDE, i.e.
    \begin{align}\label{eq:state_caastro}
    X_r \in D(A)\quad \textit{and}\quad X_r &= x+\int_t^r [AX_s+ b(s,X_s)]ds + \int_t^r\sigma(s,X_s)dW^Q_s.
\end{align}
This assumption is very restrictive as it is typically not satisfied when $H$ is infinite-dimensional; indeed, general existence results are typically only available for weaker notions of solutions of SDEs on $H$ (such as mild, weak, or variational solutions), see e.g.~\cite{da1992stochastic,liu_rockner}. Moreover, \cite{castro2022} does not present precise numerical algorithms or simulations.
\end{itemize}
We  refer to \cite{fabbri2017book} for a theoretical finite dimensional approximation procedure for PDEs on Hilbert spaces via viscosity solutions (hence, a local procedure);  this seems very difficult to implement numerically as an abstract basis from the theory of viscosity solutions is chosen there\footnote{We also refer to \cite{pannier_salvi,saporito2021path} for recent papers addressing numerical methods for \textbf{linear} path-dependent PDEs, which are different types of infinite-dimensional PDEs}.
\paragraph{Numerical schemes for infinite-dimensional optimal control.} 
Infinite-dimensional optimal control problems of  PDEs \cite{crandall1991viscosity,manzoni_quarteroni_salsa,troltzsch2010,li2012optimal}, SPDEs \cite{defeo_swiech_wessels,fabbri2017book,gozzi_masiero}, path-dependent systems \cite{defeo_federico_swiech,fabbri2017book},  partially observed stochastic systems \cite{fabbri2017book}, and mean field systems \cite{carmona_delarue} have been  theoretically studied via many approaches, such as dynamic programming and HJB equations, BSDEs, Maximum principles, adjoint systems, and, in the linear-quadratic case (LQ), Riccati equations. 
However,  numerical methods are still not well understood. 

When dealing with infinite-dimensional control problems,  a popular approach is \textbf{Discretize-then-Optimize}, where
one first performs a spatial discretization of the state equation by reducing it to a high-dimensional system of standard ODEs or SDEs\footnote{For example, in the case of a PDE or SPDE, by using finite elements or finite difference schemes, or through a Galerkin discretization with respect to an orthonormal  basis.}. The optimization problem is then tackled  by dealing with the corresponding high-dimensional HJB equation, adjoint systems, etc\footnote{For example, see  \cite{khimin2024,manzoni_quarteroni_salsa} for optimal control of PDEs via adjoint systems, \cite{alla2016hjb,ferretti1997internal,kunisch2004hjb} for optimal control of linear evolution equations and PDEs via HJB equations,  
 \cite{atwell2001proper} for LQ optimal control of heat equation via Riccati equations
 \cite{sirignano2018} for linear quadratic optimal control of SPDEs {via DGM} for HJB equations, 
 \cite{evans-pereira-boutselis2020,pirmorad2021deep} for optimal control of SPDEs via reinforcement learning,  \cite{carmona_lauriere,pham_warin_2022}  for Mean-field optimal control (and games); see also \cite{qiu2026} and the detailed review \cite{lauriere_perrin_perolat_etal}. Specific  path-dependent optimal control problems with linear dynamics arising in finance were studied with signature methods in \cite{cuchiero_moller} via convex quadratic optimization techniques.}. However, this approach only solves an optimal control problem for the discretized system, which may differ  fundamentally from the original infinite-dimensional optimal control problem. This leads to the challenging problem of proving convergence of the scheme for value functions and optimal controls of the discretized problem to those of the limit problem, for which a general theory is not well-understood.  

Instead,  \textbf{Optimize-then-Discretize} methods \cite{manzoni_quarteroni_salsa} first derive  optimality conditions for infinite-dimensional  optimal control problems (i.e.~HJB equations of the form \eqref{eq:HJB_control_intro}, adjoint systems, etc.) and then discretize these conditions to  numerically solve them,
are more intrinsic, as they try to solve the right conditions for optimality.  These approaches were heavily used in \cite{manzoni_quarteroni_salsa} for optimal control of PDEs via infinite-dimensional adjoint systems, with the authors warning that the difference between 
the two methods is fundamental,  yielding different results in general. In \cite{mowlavi2023}, an approach via PINNs was investigated for PDE-constrained optimal
control problems. In \cite{lefebvre_miller}   Riccati equations associated to
LQ stochastic optimal control problems with delays were solved via PINNs.  Optimal control problems of SPDEs where $A$ is a densely defined self-adjoint, negative definite linear operator with compact inverse, the drift is of the form $b=\tilde b(x)+u$, and noise is additive, were considered in \cite{stannat2023neural,stannat2025approximation}. They parameterize
feedback controls of the SPDE for a fixed initial condition $x$ via neural operators and prove that controls can be approximated via finite-dimensional ansatzes akin to Hilbert--Galerkin Neural Operators,  proving that such approximations induce only controlled errors in the cost functional. The resulting networks are trained via adjoint-based gradient methods\footnote{We also refer to \cite{daprato_debussche2000,daprato-debousche-ns-2000} for existence and uniqueness of mild solutions of HJB PDEs related to the optimal control of stochastic Burgers and Navier--Stokes equations via a Galerkin approximation; to \cite{yao-luo-cao-Kovachki-Roseberry-ghattas} for constrained minimization problems via derivative informed FNO, motivated by inverse problems and optimal control of PDEs (but the method is not applied to control problems there); to \cite{furuya_kratsios_possamai} for novel approaches for solving 2BSDE families via neural operators; and to \cite{firoozi_krasios_yang} for simultaneously solving infinitely many LQ mean-field games In Hilbert Spaces.
Finally, we refer to \cite{park2024} for stochastic optimal control for diffusion bridges via function spaces, to \cite{hu2025model} for  learning feedback controls of SPDEs (but no optimal control problem is considered here), to
\cite{zhang2025} for learning optimal policies for large deterministic systems of agents, to \cite{fouque_zhang} for  deep learning algorithms for mean-field games with delays, and to 
\cite{lauriere_perrin_perolat_etal} for a survey on learning mean-field games and mean-field control problems.}. Specific  path-dependent optimal control problems with linear dynamics arising in finance were studied with signature methods in \cite{jaber_hainaut_motte} via Riccati equations.

\subsection{Our contributions}
\subsubsection{Our goals.} From the literature review above, it is evident that: 
\begin{itemize}[leftmargin=*,nosep]  
\item Developing rigorous and numerically implementable approximation schemes for fully non-linear PDEs on Hilbert spaces of the form \eqref{eq:PDE_intro} remains a longstanding open challenge. These approximation schemes will require novel topologies for the convergence of the Hessians and for the continuity of $F$ in the Hessian variable, as well as addressing the unbounded operator $L$. 
    \item Most numerical schemes proposed for optimal control of deterministic and stochastic evolution equations are of the type ``Discretize-then-Optimize'', with very few taking a ``Optimize-then-Discretize'' approach to directly address the original infinite-dimensional control problem. Among the latter,  none attempted to develop schemes for infinite-dimensional HJB equations, either in the deterministic or stochastic case. Here, led by the theoretical analysis, we will develop a new family of approaches, which we will call \textbf{``Optimize-then-Learn''}.
    This will allow us to develop novel universal approximation schemes for optimal feedback controls in terms of our approximate solution HGNO.
\end{itemize}
In this paper, we fill these critical gaps in the literature through novel theoretical analysis, which naturally lead us  to  develop general numerical methods for PDEs on $H$ of the form \eqref{eq:PDE_intro}, with particular focus on HJB equations \eqref{eq:HJB_control_intro}.  

We choose ML methods to benefit from the expressivity of deep neural networks and their mesh-free power for high-dimensional computations. 
We  develop  Deep Galerkin/PINNs-type methods for these PDEs, as these seem to be the only possible approaches to attack fully non-linear second order PDEs on unbounded subsets of $H$. Our analysis will be the first for fully non-linear second-order PDEs on $H$.
\begin{itemize}[leftmargin=*,nosep]
    \item Compared to the linear FDEs in \cite{miyagawa2024physics}, we develop methods for the PDE \eqref{eq:PDE_intro} over general domains $O\subset H$, possibly unbounded and non-compact, as this is crucial for applications; 
    \item Compared to the non-linear Kolmogorov PDEs in \cite{castro2022}, we work with weaker notions of solutions for the state SDEs \eqref{eq:stateSDE_H_intro}, as this is often needed for real-world applications. This allows us to have a well-defined PDE outside of $D(A)$. To achieve this, we consider mild SDE solutions, which allow us to cover many different families of problems. That is, we assume that $A$ is the generator of a $C_0$-semigroup $e^{At}$ (see Appendix \ref{app:semigroups}); mild solutions are given by the variation of constants formula:
   \begin{align}\label{eq:mild_intro}
        X_t &= e^{At} x+ \int_0^t e^{A(t-s)}b(X_s,u_s)ds +\int_0^t e^{A(t-s)}\sigma(X_s,u_s)dW^Q_s.
    \end{align}
Compared with the (analytically) strong solution \eqref{eq:state_caastro} used in \cite{castro2022},  the mild solution does not require $X_t\in D(A)$, which is extremely restrictive in applications \cite{da1992stochastic,liu_rockner,fabbri2017book}.
\end{itemize}

\medskip

\subsubsection{Our results.} 
In this paper,  we focus on classical solutions of the PDEs \eqref{eq:PDE_intro} with $O=H$ and $k=1$  and \eqref{eq:HJB_control_intro}, i.e.~$v \in C^{2}(H)$, $LD v \in C^0(H;H)$, and $v$ satisfies \eqref{eq:PDE}, see e.g. \cite{fabbri2017book}. We discuss extensions of the method to cover general domains in Remark \ref{rem:extensions_DHGM}.
The classical solution is the most regular notion of solution, making it suitable for one of the first investigations of numerical methods. Our method is also motivated by numerical schemes of mild solutions of the PDE (not to be confused with mild solutions of SDEs as in \eqref{eq:mild_intro}); see Remark \ref{rem:mild_sol_PDE}. We will put special emphasis on optimal control problems of deterministic and stochastic evolution equations and their corresponding HJB equations, leading us to derive the first universal approximation schemes for optimal feedback controls for fully-non linear infinite-dimensional control problems.

In Section \ref{sec:analytic_frame}, we formally introduce the technical framework needed for our analysis. The paper then proceeds as follows:

\paragraph{UATs for Fréchet derivatives (Section \ref{sec:UATFrechet}).} To show that HGNOs can accurately represent classical solutions of the PDEs, we prove universal approximation theorems (UATs) for the simultaneous approximations of functions $v\in C^{2}(H)$ and their Fréchet derivatives $Dv\in C(H;H)$, $D^2v\in C(H;S(H))$ through a single neural operator.  We prove the first UATs both on compact subsets of $H$ and in opportune weighted Sobolev norms for a given probability (or bounded) measure $\mu$ on $H$. To this purpose, we consider finite-dimensional cylindrical approximations of $v$, i.e.~$v^d(x):=v(P_dx)$ and then standard UATs \cite{hornik1991approximation} can be applied to networks approximating each $v^d,Dv^d,D^2v^d$, uniformly on compacts and in $L^w(H;\mu)$. However, we then need to estimate the error in the cylindrical approximation, i.e.~$\err(v-v^d)+\err(Dv-Dv^d)+\err(D^2v-D^2v^d)$ as $d\to\infty$.
As $P_d\to I$ uniformly on compact subsets of $H$,  we  have that $\err(v-v^d)+\err(Dv-Dv^d)\to0$ both on compact subsets of $H$ and in $L^w(H;\mu)$. However, the second-order term is more delicate: since $D^2v^d(x)=P_dD^2v(P_dx)P_d$ and   $\|P_d-I\|\equiv 1$, we cannot expect $\|D^2v(x)-D^2v^d(x)\|\to0$ in operator norm, in general. Hence, we {carefully  weaken the topology:} as $P_d\to I$  uniformly on compacts, we consider the compact-open topology on $C^0(H;S(H))$, where  $S(H)$ is also endowed with a compact-open topology,  see Definitions \ref{rem:compact_open_C_SH}, \ref{rem:compact_open_C2}; {endowing $S(H)$, e.g., with the weak or strong operator topology would lead to weaker statements.}  Similarly, we prove convergence of the Hessian in $L^w(H\times H;\mu\otimes\mu')$. To the best of our knowledge, these are  \emph{new natural topologies for cylindrical approximations} of second-order Fréchet derivatives and for UATs.
 These results are more general and  intrinsic than the very recent results on UATs for derivative informed neural operators in \cite{luo-Roseberry-chen-Ghattas,yao-luo-cao-Kovachki-Roseberry-ghattas,gong-luo-roseberry-etal}, {as we do not universally approximate $D^2v$ only in particular directions and we do not impose any Hilbert-Schmidt  assumption on $D^2v$. }
 
\paragraph{UATs under the action of unbounded operators (Section \ref{subsec:UAT_unbounded}).} The above UATs are unfortunately not enough for our purposes, as the unbounded/discontinuous operator $L$ excludes, in general, the convergence to zero of $\left|\left\langle L [D v(x)- D v^{d, \theta}(x)], x\right\rangle\right|$ (even point-wise in $x$). This is a crucial difference with the finite-dimensional case, where any linear operator is continuous. In Section \ref{subsec:UAT_unbounded}, we discuss many instances where this term converges to zero both uniformly on compact sets and in $L^2(H;\mu)$. In particular, we explain how the choice of basis and the regularity of $Dv$ are crucial for this convergence. To our knowledge, these are the first results of UATs of Fréchet derivatives under the action of unbounded operators. 

\paragraph{HGNOs can solve PDEs on $H$ (Section \ref{sec:HGNO_pde}).} Our next goal is to show that HGNOs can approximately solve PDEs on Hilbert spaces \eqref{eq:PDE_intro}. However, standard continuity assumptions on the operator $F$ in the variable $Z\in S(H)$ are with respect to the operator norm, and our UATs cannot guarantee this convergence. We therefore identify a \emph{new sequential continuity assumption} on $F$ in the variable $Z\in S(H)$, when $S(H)$ is endowed with the compact-open topology, see  Assumption  \ref{ass:F_seq_cont}. However, these topologies are non-metrizable, making the problem delicate.
To the best of our knowledge, this assumption is not available in the literature.

Using these observations, we show that HGNOs approximately solve PDEs of the form \eqref{eq:PDE_intro} uniformly on compacts and in $L^2(H;\mu)$ (see Theorem \ref{th:DGM_residual}), extending \cite[Theorem 7.1]{sirignano2018} to infinite-dimensions.
We emphasize that although $v^{d,\theta}$ is essentially a finite dimensional function, it approximately solves the original infinite-dimensional PDE, rather than a finite dimensional approximation of the PDE. To the best of our knowledge, these are the first results of this kind. In Theorem \ref{th:l2_convergence}, we also prove a bounded inverse-type result under suitable additional assumptions.

\paragraph{HGNOs can solve optimal control problems on $H$ (Section \ref{sec:NHO_control}).}
    Next, we specialize to optimal control of deterministic and stochastic evolution equations. Given a dynamic programming approach, our goal is to show universal approximation results for optimal feedback controls in terms of the Fréchet derivatives of our trainable ansatz $v^{d,\theta}$, uniformly on compact sets and in $L^2(H;\mu)$.  However, we need to be careful, as the compact-open topology of $S(H)$ is not sequential in general and we only  have sequential, not full, continuity of the current value Hamiltonian $F^{cv}$ in $S(H)$. Despite these difficulties, we are able to accomplish our goal,   see Theorem \ref{th:UAT_optimal_controls}.
    To our knowledge, these are the first statements of this kind for 
    optimal control problems on Hilbert spaces. 
    Comparing with the closely related results of \cite{stannat2023neural,stannat2025approximation}, both statements are valid for general state equations\footnote{These papers consider the case where when $A$ is a densely defined, self-adjoint, negative definite linear operator with compact inverse, the drift is of the form $b=\tilde b(x)+u$, and additive noise.}. Moreover, since they are obtained via an approximation of the  value function, they  are not tied to a specific choice of initial condition. 
\paragraph{Deep Hilbert--Galerkin Methods (Section \ref{sec:DHGM}).}
Next, we introduce novel PDE-solving algorithms. In Algorithm \ref{algo:fixed_control} we consider PDEs of the form \eqref{eq:PDE_intro}. To do this, we parameterize the solution via an HGNO ansatz $v^{d,\theta} : H \to \R$ and train the parameters $\theta$ to minimize the PDE residual norm $\nor{\mathcal{F}v^{d,\theta}}_{L^2(H;\mu)}.$ The parameters can be trained by direct gradient descent, leading to the \textbf{Deep Hilbert--Galerkin Method} (DHGM), inspired by DGM/PINNs \cite{sirignano2018,raissi2019physics}, or by a biased gradient, which we call \textbf{QHPDE}, inspired by the QPDE \cite{cohen2023} method, originally developed for monotone PDEs. These algorithms are designed to  solve the full PDE by sampling points in $H$, i.e.~we are not  simply applying the standard DGM/PINNs or QPDE method on the projected PDE on $P_d(H)$ (cf.~Remark \ref{rem:def_N}). To the best of our knowledge, these are the first numerical algorithms designed for  fully non-linear second-order PDEs on the whole $H$ with unbounded operators {and the first that attack directly the  full infinite-dimensional PDE \eqref{eq:PDE_intro} on the whole $H$ and not the corresponding projected PDE in finite dimensions (e.g.~as in \cite{miyagawa2024physics}). This has the advantage of evaluating the PDE residual of \eqref{eq:PDE_intro} accurately, even though HGNO depends on only the first $d$ components of the basis. The precise details of how we implement this are explained in Remark \ref{rem:def_N}.}

In Algorithm \ref{algo:ac}, we consider optimal control problems of the form \eqref{eq:stateSDE_H_intro}-\eqref{eq:functional_intro} via HJB equations of the form \eqref{eq:HJB_control_intro}. Led by the theory, we present a novel family of approaches, which we call \textbf{``Optimize-then-Learn''}. We note that these methods may  have extensions to other methods, such as BSDEs, etc.  In an ``Optimize-then-Discretize'' approach (discussed above) one would first analytically derive conditions for  optimality in the infinite-dimensional problem \eqref{eq:HJB_control_intro} and then numerically  solve a discretized/projected version of it on $P_d(H)$. In the ``Optimize-then-Learn'' approach, we still derive the HJB equation \eqref{eq:HJB_control_intro} on $H$, but then we solve the PDE  \eqref{eq:HJB_control_intro}  directly on $H$, not a  projected version of it. This is possible thanks to the use of Neural Operators, which allow us to learn the solution{, whence the name ``Optimize-then-Learn''}. With this in mind, we develop \textbf{Hilbert Actor-Critic Methods}, i.e.~{\textbf{Reinforcement Learning} algorithms in which} we train both a PDE solution ansatz $v^{d,\theta} : H \to \R$ and an optimal control ansatz $u^{d,\phi,p}: H \to U$ parameterized by HGNOs. {To the best of our knowledge, this is the first class of reinforcement learning algorithms on an entire infinite-dimensional Hilbert space, rather than e.g. a projection.}

\paragraph{Numerical tests (Section \ref{sec:numerics}).} We test Algorithms \ref{algo:fixed_control} and \ref{algo:ac} by solving Kolmogorov and HJB PDEs arising from the optimal control of stochastic
and deterministic heat and Burgers equations. The control of the heat equation serves a good example problem because it admits a classical solution and satisfies the necessary regularity conditions in our theory. As a closed-form solution can be derived, we can easily  benchmark our algorithms. We remark that, in this case, the state SDE does not admit a strong solution (i.e.~\eqref{eq:stateSDE_H_intro}), making the results of \cite{castro2022} not applicable. However, the unique  mild solution  \eqref{eq:mild} is standard \cite{da1992stochastic,fabbri2017book}.
The control of the stochastic and deterministic Burgers equation, on the other hand, is a much more challenging problem, whose HJB equation contains multiple nonlinear unbounded operators and does not admit a classical solution, but only a suitable mild solution. It is intended as a stress-test of the method. Nonetheless, we are able to train HGNOs and compare the values they learn against Monte Carlo finite difference estimates.

\section{The analytic framework}\label{sec:analytic_frame}
In this section, we introduce the analytic framework for our problems, i.e.~PDEs on Hilbert spaces and optimal control problems on Hilbert spaces with the corresponding HJB equations. For the sake of brevity in the main presentation, we specify the notation used throughout the paper in Appendix \ref{subsec:notation}.

\subsection{PDEs on Hilbert spaces}
Throughout the whole paper, let $(H,\langle \cdot,\cdot \rangle)$ be a separable Hilbert space.

Consider the second-order fully-nonlinear PDE on $H$
\begin{align}\label{eq:PDE}
\inprod{LDv}{x} + F(x,v,Dv,D^2v)= 0, \quad  x \in H.
\end{align}
where $L \colon D(L)\subset H \to H$ is a closed, densely-defined, possibly unbounded linear operator, and $F \colon  H \times \mathbb R \times H \times S(H)\to \mathbb R$.  Comparing with \eqref{eq:PDE_intro} we are restricting to the scalar codomain $k=1$ case; our analysis extends to $k\in \mathbb N$ without any additional difficulty beyond notation. In our theoretical analysis we further restrict to $O=H$, so we are considering PDEs on the whole space $H$, therefore we do not impose boundary conditions.

We assume the following:
\begin{assumption}\label{ass:F_estimates}
$F$ is continuous and there exists $C>0, k \geq 0$ such that
\begin{align}\label{eq:est_F}
    |F( x, v, p  , Z)-F( x, r, q, Y) |  \leq C\big (|v-r|+|p-q|+\|Z-Y\|\big) \big(1+|x|^k\big).
\end{align}
\end{assumption}
The estimate \eqref{eq:est_F} includes Hamiltonians $F$ arising in optimal control on Hilbert spaces, see \eqref{eq:est_F_control}.

We will see that continuity of $F$ with respect to the variable $Z\in S(H)$, in the operator norm topology, will be too weak for our needs due to intrinsic problems in infinite-dimensional spaces (see e.g.~Remark \ref{rem:density_neural_compact_open_topology}). Thus we need to strengthen the continuity requirement by considering weaker topologies, i.e.~Assumption \ref{ass:F_seq_cont} (which, to the best of our knowledge, has not been previously considered for a general fully-non linear second-order PDE on a Hilbert space). We stress that, in the following, we require sequential continuity and not continuity, when $S(H)$ is endowed with the compact-open topology. As the compact-open topology is non sequential in general (see Definition \ref{rem:compact_open_LX}), the two notions do not coincide, with
sequential continuity being a weaker notion. Nonetheless, we will show that sequential continuity is enough for our needs. Moreover, we will verify it (Lemma \ref{seq:continuity_Fcv_F_compact_open}) in many important cases for HJB equations. Full continuity would be a lot harder to verify, as one would need to check continuity with respect to converging nets of operators and our argument would not work in this case.
\begin{assumption}\label{ass:F_seq_cont}
Let $F \colon  H \times \mathbb R \times H \times S(H)\to \mathbb R$ be sequentially continuous when $S(H)$ is endowed with any of the following topologies (see Definition \ref{rem:compact_open_LX}):
\begin{enumerate}
    \item[(i)] compact-open topology, i.e.~we assume that, for $x^n ,x,p^n,p \in H$, $v^n,v \in \mathbb R $,  and $Y^n,Y \in  S(H)$ such that $|x^n -x|,|p^n-p|,|v^n-v|\to 0 $, and $\sup_{h \in K} |[Y^n-Y]h| \xrightarrow{n \to \infty}0$ for all compact subsets $K \subset H$, we know $\modu{F(x^n, v^n, p^n, Y^n)-F\paren{x, v, p, Y}} \xrightarrow{n \rightarrow \infty } 0$.
    \item[(ii)] strong operator topology, i.e.~we assume that, for $x^n ,x,p^n,p \in H$, $v^n,v \in \mathbb R $, and $Y^n,Y \in  S(H)$ such that $|x^n -x|,|p^n-p|,|v^n-v|\to 0 $, and $ |[Y^n-Y]h| \xrightarrow{n \to \infty}0$ for all $h\in H$, we know $\modu{F(x^n, v^n, p^n, Y^n)-F\paren{x, v, p, Y}} \xrightarrow{n \rightarrow \infty } 0.$
\end{enumerate}
Sequential continuity with respect to the topologies in (i) and (ii) is equivalent, since, by Lemma \ref{lemma1:uniformcompactsfromstrongandequibounded},  a sequence in $S(H)$ converges in the compact-open topology if and only if it converges in the strong operator topology.
\end{assumption}
As pointed out by A. Święch in a correspondence with one of the authors, this assumption seems related to a crucial assumption  in the theory of viscosity solutions on Hilbert spaces, introduced by P.L. Lions in \cite[Equations (6), (7)]{lions1989viscosity} when $L=0$  (see also \cite[Assumption 3.47]{fabbri2017book}), which is used in typical viscosity perturbation arguments in the Hessian variable.  Although this assumption differs from ours, both are satisfied in similar ways when $F$ is the Hamiltonian of a stochastic control problem, see \cite{lions1989viscosity,fabbri2017book} and Lemma  \ref{seq:continuity_Fcv_F_compact_open}, e.g.~if $Q$ is trace class and $U$ is compact, or when $\sigma$ is independent of $u$.  This suggests that our assumption is sharp.
\begin{remark}\label{rem:discussion_assumption}
In this paper, we will always prove two kinds of statements: suitable uniform convergence on compact sets,  and convergence in  $L^w_\mu$ for a Borel measure $\mu$ on $H$ with $\|\mu\|_w<\infty$. For the latter, we remark that Assumption \ref{ass:F_seq_cont} is a weaker requirement then a (sequential) continuity of $F$ in the Hessian variable in  $L^w_\mu$, since $Z^n\to Z$ in the strong operator topology (or in the compact-open topology) implies\footnote{By Banach--Steinhaus, we can apply the dominated convergence theorem.}  $\left (\int_H|(Z^n-Z)h|^w\mu(dh)\right)^{1/w}\to 0$.
\end{remark}
To achieve  continuity of $F$ with respect to cylindrical approximations of  Hessians $D^2v^d(x)$, uniformly for  $x$ over compacts, we will use the following lemma.
\begin{lemma}\label{rem:continuity_PhiF}
Let Assumptions \ref{ass:F_estimates} and \ref{ass:F_seq_cont} hold.  Consider\footnote{Note that $F^{op}$ is well defined under Assumption \ref{ass:F_estimates}.} \[F^{op}: C(H ; S(H)) \rightarrow C(H \times \mathbb{R} \times H ; \mathbb{R}),\quad  (F^{op}(Z))(x, v, p):=F(x, v, p, Z(x)).\] Then $F^{op}$ is sequentially continuous when $C(H \times \mathbb{R} \times H ; \mathbb{R})$ is endowed with the  compact-open topology and 
$C(H ; S(H))$ is endowed with the  compact-open topology generated by the family of seminorms $\mathcal{P}^{coco}$ in Definition \ref{rem:compact_open_C_SH}. That is, let $Z^n,Z \in C(H ; S(H))$ be such that $\sup_{x\in K,h \in K'} |[Z^n(x)-Z(x)]h| \xrightarrow{n \to \infty}0$ for all compact $K,K' \subset H$, then 
        \begin{equation*}
    \sup_{(x,v,p)\in \mathcal K} \modu{F^{op}(Z^n)(x,v,p)-F^{op}(Z)\paren{x, v, p}}\equiv    \sup_{(x,v,p)\in \mathcal K} \modu{F(x, v, p, Z^n(x))-F\paren{x, v, p, Z(x)}} \xrightarrow{n \rightarrow \infty } 0,
    \end{equation*}
    for all compact $\mathcal K \subset H \times \mathbb R \times H$.
    \end{lemma}
    \begin{proof}
For simplicity, we show this for  $F:S(H)\to \mathbb R$, as the critical variable is the one in $S(H)$ and all other variables act as parameters on compact sets.
   Let  $Z^n,Z \in C(H ; S(H))$ be such that $\sup_{x\in K,h \in K'} |[Z^n(x)-Z(x)]h| \xrightarrow{n \to \infty}0$ for all compact $K,K' \subset H$. Fix a compact $\mathcal K\subset H$. We apply the Banach--Steinhaus theorem 
   to $\{T^{x,n}\}_{x\in \mathcal K,n\in \mathbb N}\subset \mathcal L(H)$, defined by $T^{x,n}h:=[Z^n(x)-Z(x)]h$, to see that  $\sup_{x\in \mathcal K,n\in \mathbb N}\|Z^n(x)-Z(x)\|<\infty$. Then,
   there exists $M>0$ such that $Z^n(\mathcal K),Z(\mathcal K)\subset  S_M=\{Y\in S(H):\|Y\|\leq M\}$, for all $n\in \mathbb N$. The compact-open topology is metrizable on $S_M$ (see Definition \ref{rem:compact_open_LX}), so $F$ upgrades to a  continuous function there. Picking the metric $d_M$ on $S_M$,  we have $\sup_{x\in \mathcal K}d_M(Z^n(x),Z(x))\leq \sum_{k=1}^{\infty} 2^{-k} \sup_{x\in \mathcal K}\left|(Z^n(x)-Z(x)) h_k\right|\xrightarrow{n\to \infty} 0,$ 
   where we have used the dominated convergence theorem. The claim follows by Lemma \ref{rem:relatively-compact_metric}(2).
\end{proof}
\begin{definition}[Classical solution]\label{def:classical_sol}
  A function $v: H \to \R$ is called a classical solution of \eqref{eq:PDE} if $v \in C^{2}(H)$, $LD v \in C^0(H;H)$, and $v$ satisfies \eqref{eq:PDE}.
\end{definition}
\begin{remark}\label{rem:mild_sol_PDE_1}
The classical solution is the most regular notion of  solution for these PDEs, making it suitable for one of the first investigation of numerical methods. Existence results for these PDEs are available for particular Kolmogorov-type PDEs (e.g.~see \cite{da2002second}) and for other specific cases (e.g.~see \cite{fabbri2017book}). See also Remark \ref{rem:mild_sol_PDE}.
\end{remark}

\subsection{Optimal control of deterministic and stochastic evolution equations}\label{subsec:control_setup}
The above framework is relevant for various nonlinear PDEs as mentioned in the introduction. Although our theory and methods will cover general PDEs of the form \eqref{eq:PDE}, in this paper we will take a special focus on the  HJB equation, which arises in optimal control problems of deterministic and stochastic differential equations on Hilbert spaces. These include optimal control of PDEs and SPDEs, path-dependent (S)DEs (e.g.~stochastic delay equations or stochastic Volterra integral equations), or partially observed stochastic systems. The HJB equation is a particular instance of the general PDE~\eqref{eq:PDE} which we will analyze in greater detail in what follows.

More concretely, let $H, \Xi, \tilde U$ be separable Hilbert spaces. Let $\tau=\left(\Omega, \mathcal{F}, \{\mathcal{F}_{t}\}_{t\ge 0}, \mathbb{P}, W^Q)\right)_{t \geq 0}$ be a generalized reference probability space (i.e.~$\big(\Omega, \mathcal{F}, \{\mathcal{F}_{t}\}_{t\ge 0}, \mathbb{P}\big)_{t\geq 0}$ is a complete filtered probability space with complete, right-continuous filtration $\{\mathcal{F}_{t}\}_{t\ge 0}$ and $W^Q_t$ is a generalized Wiener process on $\Xi$ \cite{da1992stochastic,fabbri2017book} with covariance operator $Q \in \mathcal L(\Xi)$). Let $A:D(A)\subset H \to H$ be the generator of a $C_0$-semigroup on $H$, and define $\mathcal U^\tau:=\{u: \Omega \times [0,\infty) \to U \text{ progressively measurable}\}$ where $U$ is a closed convex subset of  $\tilde U$; we define the proper class of admissible controls in the weak formulation by $\mathcal U=\bigcup_\tau \mathcal U^\tau$, where the union is taken over the class of all reference probability spaces $\tau$ \cite{defeo_swiech,defeo2025}.
Consider the state equation
\begin{align}\label{eq:stateSDE_H}
    dX_t &= [AX_t+ b(X_t,u_t)]dt + \sigma(X_t,u_t)dW^Q_t, \quad  X_0 = x \in H,
\end{align}
where $b \colon H \times U \to H$ and $\sigma \colon  H \times U  \to \g(\Xi;H)$, and $u \in \mathcal U.$ 
The goal is to minimize, over all admissible controls $u \in \mathcal U$, a functional of the form
\begin{align*}
 J(x;u) := \E\brac{\int_0^\infty e^{-\gamma t} l(X^{x,u}_t,u_t)dt },
    \end{align*}
where $\gamma>0$ is large enough to ensure finiteness, and $l:  H \times U \to \mathbb R$. Define the value function 
    $$V:H\to \mathbb R,\quad V(x):=\inf_{u \in \mathcal U}  J(x;u). $$
The associated HJB equation is the following second-order fully non-linear PDE on $H$ \cite{fabbri2017book}
\begin{align}\label{eq:HJB_control}
    -\gamma v +  \inprod{A^*Dv}{x}  + \inf_{u\in U }\curlbrac{\inprod{Dv}{b(x,u)}+\frac{1}{2}\mathrm{Tr}[\sigma(x,u) Q \sigma^*(x,u)D^2v]+l(x,u)} &= 0,
\end{align}
i.e.~\eqref{eq:PDE} with $L: = A^*$, where $A^*:D(A^*)\subset H \to H$ is the adjoint of $A$, and 
\begin{align}
    F(x,v,p,X)&:=-\gamma v+\inf_{u\in U}F^{cv}(x,p,X,u),\\
    F^{cv}(x,p,X,u)&:=\inprod{p}{b(x,u)}+\frac{1}{2}\mathrm{Tr}[\sigma (x,u)Q\sigma^*(x,u)X]+l(x,u).\label{eq:Fcv}
\end{align}
\begin{remark}[Deterministic case]\label{rem:detertministic}When $Q=0$ the problem reduces to an optimal control problem of an evolution equation, i.e.
\begin{align*}
    V(x)=\inf_{u \in \mathcal U}  J(x;u) =\inf_{u \in \mathcal U}   \int_0^\infty e^{-\gamma t} l(X^{x,u}_t,u_t)dt ,\quad X_t' &= AX_t+ b(X_t,u_t), \quad  X_0 = x \in H
    \end{align*}
and the HJB equation reduces to a first order PDE on $H$ \cite{crandall1991viscosity,li2012optimal}.
\end{remark}
\begin{remark}[Kolmogorov PDEs]When $\mathcal U$ is a singleton made of a feedback control (i.e.~$\mathcal U=\{u\}$, with $u:H\to U$), the problem reduces to the computation of the functional
\begin{align*}
    V(x)=  \int_0^\infty e^{-\gamma t} l(X^{x,u}_t,u(X^{x,u}_t))dt .
    \end{align*}
In this case the HJB equation reduces to a Kolmogorov PDE on $H$, i.e.
\begin{align}\label{eq:kolmogorov}
    -\gamma v+  \inprod{A^*Dv}{x}  + \inprod{Dv}{b(x,u(x))}+\frac{1}{2}\mathrm{Tr}[\sigma(x,u(x)) Q \sigma^*(x,u(x))D^2v]+l(x,u(x)) &= 0.
\end{align}
\end{remark}
\begin{assumption}\label{ass:coefficients_control}
    Assume that $Q\in \mathcal L_1^+(\Xi)$, $b: H \times U \to H$, $\sigma: H \times U \to \mathcal L(\Xi;H)$, and $l: H \times U \to \R$ are continuous and there exists $C>0$ such that for all $x,x' \in H$, we have
    \begin{align}
     |b(x,u)- b(x',u)|&\leq C(|x'-x|), &|b(x,u)|&\leq C(1+|x|)\label{eq:estimates_b},\\ 
     \|\sigma(x,u)- \sigma(x',u)\|_{\g(\Xi;H)}&\leq C(|x'-x|),
     & \|\sigma(x,u)\|_{\g(\Xi;H)}&\leq C(1+|x|),\label{eq:estimates_sigma}\\
     |l(x,u)- l(x',u)|&\leq \omega_R(|x'-x|),&|l(x,u)|&\leq C(1+|x|^m).
    \end{align}
\end{assumption}
\begin{remark}\label{rem:F_uniform_cont}
Under this assumption (see \cite{fabbri2017book}), $F$ is uniformly continuous on bounded sets of $H\times \mathbb R\times H\times S(H)$ and Assumption \ref{ass:F_estimates} is satisfied with $k = 2$, as there exists $C>0$ such that
    \begin{align}\label{eq:est_F_control}
        |F(x, v, p, Z)-F(x, r, q, Y)| \leq C\brac{|v-r|+(1+|x|)|p-q|+\left(1+|x|^2\right)\|Z-Y\|}.
    \end{align}
    \end{remark}
Under these conditions, by \cite{fabbri2017book} there exists a unique mild solution to the state equation \eqref{eq:stateSDE_H}, i.e.
    \begin{align}\label{eq:mild}
        X_t &= e^{At} x+ \int_0^t e^{A(t-s)}b(X_s,u_s)ds +\int_0^t e^{A(t-s)}\sigma(X_s,u_s)dW^Q_s.
    \end{align}
    In particular, recall that $X_t\not \in D(A)$ in general, hence this solution is not simply obtained by integrating \eqref{eq:stateSDE_H} as in the finite dimensional case ($H=\mathbb R^n$) but it is defined via the variation of constants formula exploiting the smoothness properties of the $C_0$-semigroup $e^{At}$.

    We  provide examples for which  Assumption  \ref{ass:F_seq_cont} holds.
\begin{lemma}\label{seq:continuity_Fcv_F_compact_open}
 Let Assumption \ref{ass:coefficients_control} hold.  let $\{\xi_i\}_{i=1}^\infty$ be an orthonormal basis of $\Xi$. Assume that for all $i \in \N$ and all $Z_n,Z\in S(H)$ such that $Z^n\to Z$ in the compact-open topology,  \begin{align}\label{eq:sup_Z-PdZPd_to_zero}
         \sup_{u\in U}\modu{\langle (Z^n-Z) \sigma(x,u) Q^{1/2}\xi_i, \sigma(x,u)Q^{1/2}\xi_i \rangle}\xrightarrow{n \to \infty}0.
    \end{align}
 Then Assumption  \ref{ass:F_seq_cont} holds.
\end{lemma}
\begin{remark}\label{rem:Fcv_seq_cont}
    Condition \eqref{eq:sup_Z-PdZPd_to_zero} holds if, for example:
   \begin{enumerate}
    \item The action space $U$ is compact. Indeed, in this case
    \begin{align*}
        &\sup_{u\in U}\left | \langle   (Z^n-Z) \sigma(x,u) Q^{1/2}\xi_i, \sigma(x,u)Q^{1/2}\xi_i \rangle\right |\leq C\sup_{y\in K_\sigma^i}\left |    (Z^n-Z) y\right | \xrightarrow{n \to \infty}0,
    \end{align*}
   where $K_\sigma^i:=\sigma(x,U)Q^{1/2}\xi_i \subset H$ is compact, for all $i$.
   \item The diffusion coefficient $\sigma=\sigma(x)$ is independent of $u$.
    \end{enumerate}
\end{remark}
\begin{proof}
 Let $x^n,x, p^n,p \in H$, $v^n,v \in \mathbb R$, $Z^n,Z\in S(H)$ such that  $x^n \to x$, $p^n \to p$, $v^n\to v$,  and 
    $\sup_{h\in K} |[Z^n-Z]h| \xrightarrow{n \to \infty}0, $ for all $K \subset H$ compact subsets;
then
\begin{align*}
&|F\left(x^n, p^n, Z^n\right)-F\left(x, p, Z\right) |  \leq  | F\left(x^n, p^n, Z^n\right)-F\left(x, p, Z^n\right)|+| F\left(x, p, Z^n\right)-F\left(x, p, Z\right)|.
\end{align*}
By the Banach--Steinhaus theorem we have  $\sup_{n \in \mathbb N}\|Z^n\|< \infty$. Then, the first term goes to zero by uniform continuity of $F$ on bounded sets of $X\times \mathbb R\times H\times S(H)$. For the second term, it  suffices to prove that 
\begin{align}\label{eq:Tr_sigma_Z-P-dZ_to_zero}
     \sup_{u \in U} |\mathrm{Tr}[\sigma(x,u)Q\sigma^*(x,u)(Z^n-Z)]| \xrightarrow{n \to \infty}0.
          \end{align}
We have
     \begin{align*}
        \sup_{ u \in U} |\mathrm{Tr}[\sigma(x,u)Q\sigma^*(x,u)(Z^n-Z)]|&= \sup_{u \in U} |\mathrm{Tr}[\left(\sigma(x,u)Q^{1/2}\right)^*(Z^n-Z)\sigma(x,u)Q^{1/2}]|\\
        &= \sup_{u \in U} \left | \sum_{i=1}^\infty\langle   (Z^n-Z) \sigma(x,u)Q^{1/2}\xi_i,  \sigma(x,u)Q^{1/2}\xi_i \rangle\right | .
     \end{align*}
Since $\sup_{n \in \mathbb N}\|Z^n-Z\|< \infty$ and  by \eqref{eq:estimates_sigma}, we have, for $C_x>0,$ 
     \begin{align*}
        &\sup_{u \in U}  \left | \langle   (Z^n-Z) \sigma(x,u) Q^{1/2}\xi_i, \sigma(x,u)Q^{1/2}\xi_i \rangle\right |\leq C_x |Q^{1/2}\xi_i |^2,\quad \sum_{i = 1}^\infty|Q^{1/2}\xi_i |^2
        = \|Q\|_{\g_1(\Xi)}<\infty.
     \end{align*}
      Then we conclude that \eqref{eq:Tr_sigma_Z-P-dZ_to_zero} holds by  dominated convergence for series. 
\end{proof}
\begin{remark}\label{seq:continuity_Fcv_compact_open}
      Let Assumption \ref{ass:coefficients_control} hold.  Then,  $F^{cv}:H\times \mathbb R \times H\times S(H)\times U\to \mathbb R$ is sequentially continuous, when $S(H)$ is endowed with the compact-open topology. The proof is a simplification than the one in Lemma \ref{seq:continuity_Fcv_F_compact_open}.
    \end{remark}   
    
    \paragraph{Examples.} The above  is a very general setup, able to handle many infinite-dimensional deterministic and stochastic optimal control problems, as discussed in the introduction. For instance, under suitable standard conditions, it covers deterministic and stochastic optimal control problems related to 
    \begin{itemize}[leftmargin=*,nosep]
        \item heat equations (see Example \ref{ex:heat_eq_intro} and Section \ref{sec:numerics}), e.g.~\cite{fabbri2017book},
        \item wave equations, e.g.~\cite{fabbri2017book},
        \item first-order transport equations, e.g.~\cite{faggian_gozzi,defeo_gozzi_swiech_wessels},
        \item delay equations (in the state and/or in the control), e.g.~\cite{defeo_federico_swiech,defeo_phd,defeo_DEAF,fabbri2017book},
        \item Volterra integral equations with regular kernels, e.g.~\cite{possamai_mehdi} (Example \ref{ex:volterra}),
        \item  particle systems and mean-field problems, leading to lifted limit HJB equations on $H$ (Example \ref{ex:MFC}), e.g.~\cite{gangbo_mayorga_swiech,swikech2025finite} and generalizations to particle systems and mean-field problems of stochastic evolution equations, e.g.~\cite{defeo_gozzi_swiech_wessels}.
    \end{itemize}

\subsection{Hilbert--Galerkin Neural Operators (HGNOs)}\label{sec:neural_network}

We use a Hilbert--Galerkin Neural Operator (HGNO) \cite{castro2022} to represent the PDE solution $v : H \to \R$ and, if relevant to the problem, the optimal control $u : H \to U$. The HGNO is an encoder-decoder type architecture for learning nonlinear operators between separable Hilbert spaces that works by representing elements of each Hilbert space via coordinates in a truncated orthonormal basis, where the map between coordinates in these bases is a trainable finite-dimensional approximation scheme such as a deep neural network. In particular, we assume for each domain and codomain dimension $d, p \in \N$, there exists a class of finite-dimensional functions $\{\tilde{f}^{d,\theta,p}  : \R^d \to \R^p\}_{\theta \in \Theta}$, where $\theta \in \Theta$ represents the trainable parameters, that universally approximate in the following senses.
\begin{assumption}[Universal approximation on compacts]\label{ass:finite_dim_approx}
    For all $d,p \in \N$, the set of functions $\{\tilde f^{d,\theta,p}: \mathbb{R}^d \rightarrow \mathbb{R}^p\}_{\theta \in \Theta}$ is dense in $C^2(\R^d;\ \R^p)$ endowed with the compact-open topology. That is, for any $h\in C^2(\R^d; \R^p),$ $K \subset \R^d$ compact, $\epsilon > 0$, there exists $\theta \in \Theta$ such that 
    \begin{equation*}
      \modu{h - \tilde f^{d,\theta,p}}_{C^2(K;\R^p)}:=  \sum_{i=1}^p \brac{\sup_{x \in K}\modu{h_i(x)-\tilde f^{d,\theta,p}_i(x)} + \sup_{x \in K}\modu{\nabla_x(h_i - \tilde f^{d,\theta,p}_i)(x)} + \sup_{x \in K}\modu{\mathrm{Hess}_x(h_i-\tilde f^{d,\theta,p}_i)(x)}} < \epsilon.
    \end{equation*}
\end{assumption}
Given a finite measure $\mu$ on $\R^d$, let $C^{2}_{\mu}(\mathbb R^d;\mathbb R^p):=\left\{f \in C^2\left(\mathbb R^d;\mathbb R^p\right):\|f\|_{H^2(\mathbb R^d;\mathbb R^p;\mu)}<\infty\right\}$.

\begin{assumption}[{Universal approximation in Sobolev norms}]\label{ass:finite_dim_sobolev_approx}
    For all $d,p \in \N$ and finite measures $\mu$ on $\R^d$, the set of functions $\{\tilde f^{d,\theta,p}: \mathbb{R}^d \rightarrow \mathbb{R}^p\}_{\theta \in \Theta}$ is such that for any $h \in  C^{2}_{\mu}(\mathbb R^d;\mathbb R^p)$ and $\epsilon > 0$, there exists $\theta \in \Theta$ such that $\|h - \tilde f^{d,\theta,p}\|_{H^2(\R^d;\R^p;\mu)} < \epsilon.$
\end{assumption}

Our motivating class of functions with these approximation capabilities are neural networks.

\begin{definition}[Neural networks]\label{def:neural_structure}
    Let $\mathfrak{m} \in C^0(\R)$ be an `activation' function. Define the set of neural network parameters of an arbitrary number $\mathfrak{L} \in \N$ of hidden layers by $\Theta := {\Theta}_{d, p}:=\bigcup_{\mathfrak{L} \in \mathbb{N}} {\Theta}_{\mathfrak{L},d,p}$, where 
    $$\begin{aligned} & {\Theta}_{\mathfrak{L}, d, p}:=\left\{\theta=\left(A_1, A_2, \ldots, A_\mathfrak{L}\right): A_j: \mathbb{R}^{d_j} \rightarrow \mathbb{R}^{d_{j+1}} \text { is an affine function}, j=1, \ldots, \mathfrak{L}-1, d_1=d, d_{\mathfrak{L+1}}=p  \right \}.
    \end{aligned}$$ For each $\theta \in \Theta$, the corresponding deep neural network is the function $\tilde f^{d,\theta,p}: \mathbb{R}^d \rightarrow \mathbb{R}^p$ given by $$\tilde{f}^{d, \theta, p}(y):=A_{\mathfrak{L}} \circ \mathfrak{m} \circ A_{\mathfrak{L}-1} \circ \cdots \circ \mathfrak{m} \circ A_1(y),$$ where $\mathfrak{m}$ is understood to apply componentwise.
\end{definition}
Fully connected neural networks typically obey Assumptions \ref{ass:finite_dim_approx} and \ref{ass:finite_dim_sobolev_approx}. In particular, the subclasses of
\begin{itemize}
    \item wide neural networks with a single hidden layer ($\mathfrak{L} = 1$, $d_2$ arbitrarily large) with $\mathfrak{m} \in C^2_b(\R)$ nonconstant; this is the classic result of Hornik \cite[Theorems 3 and 4]{hornik1991approximation},
    \item deep neural networks of finite width ($\mathfrak{L}$ arbitrarily large, $d_2, \dots, d_\mathfrak{L} = d+p+1$) with $\mathfrak{m} \in C^3(\R)$ non-affine \cite{kidger2020universal, JMLR:v24:22-1191}\footnote{The stated claim follows for fully connected networks from a slight extension of the results of \cite{kidger2020universal} (by way of Nachbin's \cite{Nachbin1949} extension of the Stone--Weierstrass theorem to obtain approximation of a function and its derivatives in finite dimensions, rather than simply approximation in $C^0$). For other deep neural network architectures, see \cite{JMLR:v24:22-1191} for bounds on the required dimensions.}
\end{itemize}
are each individually rich enough to satisfy the two assumptions.  Other architectures, such as some transformer models \cite{Yun2020Are}, convolutional neural networks \cite{HWANG2026101833}, and kernel methods \cite{JMLR:v7:micchelli06a}, can also be shown to have results of this type, under appropriate assumptions, along with the classical examples of approximation in Fourier or polynomial bases.

We now state the definition of a Hilbert--Galerkin Neural Operator.
\begin{definition}[HGNOs]\label{def:deep_H_net}
    Let $H_1$, $H_2$ be separable Hilbert spaces and let $\{e_i\}_{i=1}^\infty \subset H_1$ and $\{g_i\}_{i=1}^\infty \subset H_2$ be orthonormal bases of $H_1$, $H_2$, respectively. Let $d,p \in \N$ be such that $d \leq \dim H_1$ and $p \leq \dim H_2$. Let $\{\tilde f^{d,\theta,p} : \R^d \to \R^p\}_{\theta \in \Theta}$ satisfy Assumptions \ref{ass:finite_dim_approx} and \ref{ass:finite_dim_sobolev_approx}. A $(H_1, \{e_i\}_{i=1}^\infty, d, \theta, H_2, \{g_i\}_{i=1}^\infty, p)$-HGNO, or simply HGNO, is a map
    \begin{equation}
   f^{d,\theta,p} : H_1 \to H_2,\quad    f^{d,\theta,p}(x) = \left(\widehat{\mathcal E}_p^{H_2} \circ \tilde f^{d,\theta,p} \circ \mathcal{E}_d^{H_1}\right)(x)=\sum_{j=1}^p \tilde f^{d,\theta,p}_j\Big((\inprod{x}{e_i})_{i=1}^d\Big) g_j 
    \end{equation}
    where $\mathcal{E}_d^{H_1}:H_1\to \mathbb{R}^d$ and $\widehat{\mathcal E}_p^{H_2}:\mathbb{R}^p\to H_2$ are the coordinate and embedding operators and $\tilde f^{d,\theta,p}_j$ represents the $j$-th component. We denote the class of HGNOs of this form by $\mathcal {HGNO}(H_1,H_2)=\{f^{d,\theta,p}:d,p\in \mathbb N,\theta\in \Theta\}$.
\end{definition}

We will prove in Theorems \ref{th:UAT_K} and \ref{th:UAT_L2_noL} that the HGNO given above has approximation properties analogous to finite-dimensional networks when $H_2 = \R$ and, by trivial extension, when $H_2 = \R^k$. We will apply these theorems to solving PDEs on Hilbert spaces. We remark that in the theorems stated in the present paper, every statement regarding existence of HGNOs satisfying uniform properties on compact sets will only require Assumption \ref{ass:finite_dim_approx}, while statements regarding properties in $L^2_\mu$ or in Sobolev norms will only require Assumption \ref{ass:finite_dim_sobolev_approx}, i.e.~we will not need Assumptions \ref{ass:finite_dim_approx} and \ref{ass:finite_dim_sobolev_approx} to hold simultaneously. However, for brevity, we include both assumptions in the definition of an HGNO.

\section{Universal approximation  of Fréchet derivatives on $H$}\label{sec:UATFrechet}
In this section, we prove new universal approximation theorems for the simultaneous approximation of a function $v \in C^2(H)$ and its Fréchet derivatives uniformly on compact subsets of $H$ and in opportune Sobolev norms. We also extend these results to hold  under the actions of unbounded operators. These results are of independent interest but will be used later to show that HGNOs can approximately solve PDEs on Hilbert spaces.

Let $H$ be a Hilbert spaces with orthonormal basis $\{e_i\}$, and let $P_d:H\to \mathrm{span}\{e_1,...,e_d\}\subset H$ denote the orthogonal projection map (see Definition \ref{def:embedding_op}). 
We introduce the cylindrical approximation of $v :H\to \mathbb R$. 
\begin{notation}[Cylindrical approximation]\label{def:cylindrical_approx}
  Given  $v \in C^0(H)$ we denote its cylindrical approximation of degree $d \in \mathbb N$ by
\begin{equation}\label{eq:def_vd}
    v^d \colon H \to \R, \quad   v^d(x):=v(P_d x).
\end{equation}
Then we can define the corresponding finite dimensional function  
$$\tilde v^d : \mathbb R^d \to \mathbb R, \quad \tilde v^d(x_1,...,x_d):=v^d(\widehat {\mathcal E}_d(x_1,...,x_d))= v(\widehat{\mathcal{E}}_{d}(x_1,...,x_d)).$$
\end{notation}
If $v \in C^2(H)$ then $\tilde v^d \in C^2(\mathbb R^d)$ with $\partial_i \tilde v^d(x_1,...,x_d)=(Dv(\widehat{\mathcal{E}}_{d}(x_1,..,x_d))_i$, for $i\leq d$, and $\partial_{ij}^2 \tilde v^d(x_1,...,x_d)=\langle  D^2v(\widehat{\mathcal{E}}_{d}(x_1,...,x_d)) e_i,e_j\rangle$, for $i,j \leq d$. Notice that 
\begin{small}
\begin{align}\label{eq:Dvd}
&Dv^d(x)=P_d Dv(P_dx)=\sum_{i=1}^d \partial_i \tilde v^{ d}\left(\mathcal{E}_{ d}(x)\right) e_i,\qquad D^2v^d(x)h=P_d D^2 v\left(P_d x\right) P_dh=\sum_{i=1}^d\left[\sum_{j=1}^d\partial^2_{ij} \tilde v^{d}\left(\mathcal{E}_{ d}(x)\right)h_j \right ]e_i.
\end{align}
\end{small}

Let $L:D(L)\subset H\to H$ be a closed  linear operator with dense domain. Furthermore, assume that $\{e_i\}\subset D(L)$. Then $Dv^{d,\theta} :H \to D(L)$ and $LDv^{d,\theta}\in C^0(H;H)$  with \begin{align}\label{eq:LDvd}
LDv^{d}(x)=\sum_{i=1}^d \partial_i \tilde v^{ d}\left(\mathcal{E}_{ d}(x)\right) Le_i.
\end{align}
\begin{remark}
Since $D(L)$ is dense in $H$ (which is separable), there exists a countable set $\{y_n\}\subset D(L)$ dense in $H$. By the Gram--Schmidt algorithm, we construct an orthonormal basis $\{e_i\}$ of $H$ such that $e_i\in D(L)$, and assume that our approximation is taken in such a basis.
\end{remark}
\paragraph{Parametrization of the solution of the PDE.} We parameterize  $v\in C^2(H)$ via a HGNO (Definition \ref{def:deep_H_net}) with $H_1=H,H_2=\mathbb R$, i.e., $$v^{d,\theta}\equiv v^{d,\theta,1}\colon H \to \mathbb R, \quad v^{d,\theta}(x)= \tilde v^{d,\theta} ( \mathcal{E}_{ d}(x)),$$
where $\tilde v^{d,\theta} \equiv \tilde v^{d,\theta,1} : \R^d \to \R$ is a deep neural network (Definition \ref{def:neural_structure}).
Notice that $v^{d,\theta}\in C^2(H)$ with
\begin{align}\label{eq:Dvdtheta}
&Dv^{d,\theta}(x)=\sum_{i=1}^d \partial_i \tilde v^{d,\theta}\left(\mathcal{E}_{ d}(x)\right) e_i,\quad D^2v^{d,\theta}(x)h=\sum_{i=1}^d\left[\sum_{j=1}^d\partial^2_{ij} \tilde v^{d,\theta}\left(\mathcal{E}_{ d}(x)\right)h_j \right ]e_i.
\end{align}
Furthermore, if $\{e_i\}\subset D(L)$, then $Dv^{d,\theta} :H \to D(L)$ and $LDv^{d,\theta}\in C^0(H;H)$  with \begin{align}\label{eq:LDvd_theta}
LDv^{d,\theta}(x)=\sum_{i=1}^d \partial_i \tilde v^{ d, \theta}\left(\mathcal{E}_{ d}(x)\right) Le_i.
\end{align}
\subsection{Universal approximation on compact subsets of $H$}\label{sec:UAT}
\begin{theorem}\label{th:UAT_K}
The set $\mathcal{HGNO}(H,\mathbb R)$ is dense in $C^2(H)$, when this space is endowed with the compact-open topology generated by  the directed family of seminorms $\mathbf P ^{COCO}$ described in Definition \ref{rem:compact_open_C2}. That is for all $v \in C^2(H)$, $K,K'\subset H$ compact subsets,  for every $\epsilon>0$, there exist $d \in \mathbb N$ and $\theta \in \Theta$ such that $\mathbf p^{COCO}_{K,K'}(v-v^{d,\theta})<\epsilon,$ i.e.
    \begin{equation}\label{eq:convergence_Dvtheta}
        \sup_{x \in K} |v(x)-v^{d,\theta}(x)|<\epsilon, \quad \sup_{x \in K}|D v(x)-D v^{d,\theta}(x)|< \epsilon, \quad \sup_{x \in K, h \in K'}|[D^2v(x)-D^2v^{d,\theta}(x)]h|<\epsilon .
    \end{equation}
\end{theorem}
\begin{remark}\label{rem:density_neural_compact_open_topology}
In general, we cannot expect that $D^2v^{d,\theta}$ approximates $D^2v$ in operator norm. The main issue is that the projection operator $P_d$ does not converge to the identity operator in the operator norm as $d \to \infty$, as $\|I-P_d\|\equiv 1$. Therefore, without any additional regularity of $D^2v$, we cannot expect that standard universal approximation results in finite dimensions (Assumption \ref{ass:finite_dim_approx}) extend to infinite dimensions with a density statement when $S(H)$ is endowed with the operator norm (i.e.~in the standard
compact-open topology generated by  the family of seminorms $\mathbf P ^{CO}$ described in Definition \ref{rem:compact_open_C2}).
However, the above UAT proves the density of the set  $\mathcal{HGNO}(H,\mathbb R)$ in $C^2(H)$ when $S(H)$ is also endowed with the (weaker) compact-open topology (i.e.~when this space is endowed with the compact-open topology generated by  the family of seminorms $\mathbf P ^{COCO}$). This is a natural topology for the universal approximation, as it exploits the fact that $P_d\to I$ in the compact-open topology. Endowing $S(H)$, for example, with the weak or strong operator topology would lead to weaker statements in Theorem \ref{th:UAT_K}. Moreover, we see the clear role of the compact-open topology on $S(H)$ for verifying Assumption \ref{ass:F_seq_cont} in Remark \ref{rem:Fcv_seq_cont}.  
\end{remark}
\begin{proof} Let $K,K'\subset H$ be compacts and $\epsilon>0$. We denote the compact sets $\tilde  K:=\overline{\bigcup_{d \in\mathbb N}P_d(K)}$, $\tilde K':=\overline{\bigcup_{d \in\mathbb N}P_d(K')}$ (see Remark \ref{rem:relatively-compact_Pd-I}). We fix a modulus of continuity $\omega$ of $v$ and its derivatives over the compact set $\tilde  K$ (so $\omega$ is independent of $d$) We will prove approximation of the function and its derivatives in turn.
\paragraph{Approximating $v$.} As $v^d(x) = v(P_d x)$, we have 
  $$\sup_{x \in K} |v(x)-v^{d,\theta}(x)|\leq \sup_{x \in K} |v(x)- v(P_dx)|+\sup_{x \in K} | v^{d}(x)-v^{d,\theta}(x)|=:I_1^{d}+I_2^{d,\theta}.$$
  \begin{itemize}[leftmargin=*,nosep]
      \item For $I_1^{d}$, we have $I_1^{d}\leq \sup_{x \in K} \omega(|(I-P_d) x|) \leq  \omega(\sup_{x \in K}|(I-P_d) x|)$.Thanks to  Lemma \ref{lemma1:uniformcompactsfromstrongandequibounded},  this term can be made arbitrarily small for $d$ large enough.
      \item For $I_2^{d,\theta}$, we have $I_2^{d,\theta}=\sup_{x \in K} | \tilde v^{d}(\mathcal E_d(x))-\tilde v^{d,\theta}(\mathcal E_d(x))|=\sup_{y \in \mathcal E_d(K)} | \tilde v^{d}(y)-\tilde v^{d,\theta}(y)|$. For fixed $d$, this term  can  be made arbitrarily small  by Assumption \ref{ass:finite_dim_approx}  (note that $\mathcal E_d(K) \subset \mathbb R^d$ is a compact set, as the continuous image of the compact set $K$). 
  \end{itemize}
\paragraph{Approximating $Dv$.} Recall \eqref{eq:Dvd}, then
$$
\sup_{x \in K}|D v(x)-D v^{d,\theta}(x)|\leq \sup_{x \in K}|(I-P_d)Dv(x)|+\sup_{x \in K} |P_d[Dv(x)-Dv(P_dx)]|+\sup_{x \in K}|Dv^d(x)-D v^{d,\theta}(x)|=I_3^{d}+I_4^{d}+I_5^{d,\theta}
$$
\begin{itemize}[leftmargin=*,nosep]
    \item For $I_3^d$,  we have $I_3^d\leq \sup_{x \in K}|(I-P_d)Dv(x)|= \sup_{y \in Dv(K)}|(I-P_d)y|$, where    $Dv(K) \subset H$ is compact (as continuous image of the compact set $K\subset H$). As above,  thanks to  Lemma \ref{lemma1:uniformcompactsfromstrongandequibounded},  this term can be made arbitrarily small for $d$ large enough.
\item We have $I_4^d\leq\sup_{x \in K} |[Dv(x)-Dv(P_dx)]| \leq  \sup_{x \in K} \omega(|(I-P_d) x|)\leq  \omega(\sup_{x \in K} |(I-P_d) x|)$.
\item Using  \eqref{eq:Dvd}, \eqref{eq:Dvdtheta}, we have $I_5^{d,\theta}\leq  \sup_{x \in K} \sum_{i=1}^d |\partial_i \tilde  v^d(\mathcal{E}_{ d}( x))-\partial_i \tilde v^{d,\theta}\left(\mathcal{E}_{ d}(x)\right)|=\sup_{y \in \mathcal E_d(K)}\sum_{i=1}^d |\partial_i \tilde  v^d(y)-\partial_i \tilde v^{d,\theta}\left(y\right)|$. For fixed $d$, since $\mathcal{E}_{ d}(K)\subset \mathbb R^d$ is compact, this term  can be made arbitrarily small by Assumption \ref{ass:finite_dim_approx} (simultaneously with our approximation of $v$).
\end{itemize}

\paragraph{Approximating $D^2v$.}
Recall \eqref{eq:Dvd}, then
    \begin{align*}\sup_{x \in K, h \in K'}|[D^2v(x)-D^2v^{d,\theta}(x)]h|&\leq \sup_{x \in K, h \in K'}|D^2 v(x) h-P_d D^2 v\left(P_d x\right) P_d h | +\sup_{x \in K, h \in K'}|[D^2v^d(x)-D^2v^{d,\theta}(x)]h|\\
    &=:I_6^d+I_7^{d,\theta}.
    \end{align*}
\begin{itemize}[leftmargin=*,nosep]
  \item For $I_6^d$, we have:
    \begin{align*}
        I_6^d&\leq\sup_{x \in K, h \in K'}|\left[D^2 v(x)-D^2 v\left(P_d x\right)\right] h|+\sup_{x \in K, h \in K'}|D^2 v\left(P_d x\right)\left(I-P_d\right) h|+\sup_{x \in K, h \in K'}|\left(I-P_d\right) D^2 v\left(P_d x\right) P_d h|\\
        &\leq \sup_{x \in K} C_{K'} \|D^2 v(x)-D^2 v\left(P_d x\right)\|+\sup_{x \in K}\|D^2 v\left(P_d x\right)\| \sup_{ h \in K'}|\left(I-P_d\right) h|+\sup_{y \in Z}|\left(I-P_d\right) y|\\
        &\leq C_{K'} \omega \left(\sup_{x \in K}|(I-P_d)x|\right)+C_{K}\sup_{ h \in K'} |\left(I-P_d\right) h|+\sup_{y \in \overline Z}|\left(I-P_d\right) y|,
            \end{align*}
            where $C_{K'}>0$ depends only on $K'$,  $C_{ K}>0$  bounds $\|D^2 v\left(\cdot\right)\|$ over the compact set $\tilde K $ (independent of $d$), 
            and 
            \begin{align*}
                Z:=\{D^2 v\left(P_d x\right) P_d h: \ (x,h) \in K \times K', d \in \mathbb N\}\subset \overline Z:=\left \{D^2 v\left(r\right) z: \ (r,z) \in \tilde K \times \tilde K'\right \},
            \end{align*}
            so that $\overline Z$ is a compact set as the continuous image of the compact set $\tilde K  \times \tilde K'$ via the map $D^2v(\cdot)(\cdot)$. As usual,  this term can be made arbitrarily small for $d$ large enough.
    \item Using  \eqref{eq:Dvd}, \eqref{eq:Dvdtheta}, for $I_7^{d,\theta},$ we have 
    \begin{align*}
        I_7^{d,\theta}&\leq  \sup_{x \in K,h \in K'}\sum_{i,j=1}^d|\partial_{ij} \tilde v^d(\mathcal E_d(x))- \partial_{ij}\tilde v^{d,\theta} (\mathcal E_d(x))| |h_j| \leq C_{K'} \sup_{x \in K}\sum_{i,j=1}^d|\partial_{ij} \tilde v^d(\mathcal E_d(x))- \partial_{ij}\tilde v^{d,\theta} (\mathcal E_d(x))|\\
        &=C_{K'} \sup_{y \in \mathcal E_d(K)}\sum_{i,j=1}^d|\partial_{ij} \tilde v^d(y)- \partial_{ij}\tilde v^{d,\theta} (y)|,
    \end{align*}
where  $C_{K'}>0$. Again, for fixed $d$,  $I_7^{d,\theta}$ can be made arbitrarily small by Assumption \ref{ass:finite_dim_approx} (simultaneously with the approximation of $v$ and $Dv$).
\end{itemize}
\paragraph{Completion of the proof.}
By Lemma \ref{lemma1:uniformcompactsfromstrongandequibounded}, there exists $\bar d\in\mathbb N$ such that  $I_1^d,I_3^d+I_4^d,I_6^d < \epsilon/2$, for all $d\geq \bar d$. Then, for such $d$, using Assumption \ref{ass:finite_dim_approx},  there exists $\theta \in \Theta$ such that $I_2^{d,\theta},I_5^{d,\theta},I_7^{d,\theta} < \epsilon/2$. 
\end{proof}
\begin{remark}\label{eq:operatornormD2vdtheta_K}
    By the proof of Theorem \ref{th:UAT_K} we see that 
    \begin{enumerate}[leftmargin=*,nosep]
    \item  $v^d\to v $  in the compact-open topology of $C^2(H)$ generated by the family of seminorms ${\mathbf {\mathcal {\mathbf P}}}^{COCO}$ in Definition \ref{rem:compact_open_C2}. That is, for all compacts $K,K' \subset H$, for every $\epsilon>0$, there exist $\bar d \in \mathbb N$ such that for all $d\geq \bar d$, it holds that $\mathbf p^{COCO}_{K,K'}(v-v^{d})<\epsilon,$ i.e.
    \begin{equation}
        \sup_{x \in K} |v(x)-v^{d}(x)|<\epsilon, \quad \sup_{x \in K}|D v(x)-D v^{d}(x)|< \epsilon, \quad \sup_{x \in K, h \in K'}|[D^2v(x)-D^2v^{d}(x)]h|<\epsilon .
    \end{equation}
        \item For  $d\in \mathbb N$ fixed, as $v^d$ acts as a $d$-dimensional function,
   we  have density in the standard compact-open topology of $C^2(H)$ generated by the family of seminorms ${\mathbf {\mathcal {\mathbf P}}}^{CO}$ in Definition \ref{rem:compact_open_C2}, i.e.~for every $K \subset H$ compact, for every $\epsilon>0$, there exist $\theta \in \Theta$ such that
   \begin{equation}
        \sup_{x \in K} |v^d(x)-v^{d,\theta}(x)|<\epsilon, \quad \sup_{x \in K}|D v^d(x)-D v^{d,\theta}(x)|< \epsilon, \quad \sup_{x \in K}\|D^2v^d(x)-D^2v^{d,\theta}(x)\|<\epsilon ,
    \end{equation}
   where Hessians are endowed with the operator norm. This is a direct consequence of Assumption \ref{ass:finite_dim_approx}.
    \end{enumerate}
\end{remark}
\subsection{Universal approximation in weighted Sobolev norms}
We now prove uniform approximation in opportune weighted Sobolev-type norms.  We recall the notation $\|\mu\|_{q}:= \left(\int_H |x|^{q}\mu(dx)\right)^{1/q}$ for a Borel probability measure $\mu$  on $H$   and the definition of  $C^{2}_{w,q}(H)$ and its weighted Sobolev-type norm, given in \eqref{eq:growth_v_UATsobolev}.
\begin{theorem}\label{th:UAT_L2_noL}
Let  $q,w\geq 1$ and let $\mu,\mu'$ be Borel probability measures on $H$ such that  $\|\mu\|_{q},\|\mu'\|_{w} < \infty$. Then $\mathcal {HGNO}(H,\mathbb R)$ is dense in $C^{2}_{w,q}(H)$. That is
for every $v\in C^{2}_{w,q}(H)$, for all $\epsilon>0$, there exist $d \in \mathbb N,\theta \in \Theta$ such that $ \|v-v^{d,\theta}\|_{\mathcal  W^{2,w}_{\mu,\mu'}(H)}^w<\epsilon,$ i.e.
 \begin{small}
\begin{equation}\label{eq:convergence_Dvtheta_L1}
    \int_H |v(x)-v^{d,\theta}(x)|^w\mu(dx)<\epsilon, \quad \int_H |D v(x)-D v^{d,\theta}(x)|^w\mu(dx)< \epsilon, \quad  \int_{H\times H} |[D^2v(x)-D^2v^{d,\theta}(x)]h|^w\mu(dx)\mu'(dh)<\epsilon .
\end{equation}
\end{small}
\end{theorem}
\begin{remark}\label{rem:density_neural_Sobolev}
    Similarly to Remark \ref{rem:density_neural_compact_open_topology}, in general, we cannot expect to approximate $D^2v$ in the Sobolev norm defined via the operator norm,  i.e.~$\|\cdot\|_{\mathcal W^{2,w}_{\mu}(H)}$ defined in \eqref{eq:sobolev_norm}.
\end{remark}
\begin{proof}
Let $\epsilon>0$. By dominated convergence, there exists $\bar d>0$ such that for all $d\geq \bar d$,
\begin{equation*}\begin{split}
    \int_H |v(x)-v^{d}(x)|^w\mu(dx)&<\epsilon/2,\\ \int_H |D v(x)-D v^d(x)|^w\mu(dx)&< \epsilon/2, \\  \int_{H\times H} |[D^2v(x)-D^2v^d(x)]h|^w\mu(dx)\mu'(dh)&<\epsilon /2. 
    \end{split}
\end{equation*}
Then for such $d$, denoting $\mu^d(\cdot) :=\mu (\mathcal E_d^{-1} (\cdot))$ the pushforward measure on $\mathbb R^d$, we have
\begin{align*}
    \int_H |v^d(x)-v^{d,\theta}(x)|^w\mu(dx)&=\int_{\mathbb R^d} |\tilde v^d(y)-\tilde v^{d,\theta}(y)|^w\mu^d(dy)
    <\epsilon/2,\\
    \int_H |D v^d(x)-D v^{d,\theta}(x)|^w\mu(dx)&=\int_{\mathbb R^d}|D \tilde v^d(y)-D \tilde v^{d,\theta}(y)|^w \mu^d(dy)< \epsilon/2, \\  
    \int_{H\times H} |[D^2v^d(x)-D^2v^{d,\theta}(x)]h|^w\mu(dx)\mu'(dh)&\leq C_d\int_{H} \|D^2v^d(x)-D^2v^{d,\theta}(x)\|^w\mu(dx)\\
    & = C_d \int_{\mathbb R^d} \|D^2\tilde v^d(y)-D^2\tilde v^{d,\theta}(y)\|^w\mu^d(dy)\\
    &\leq C_d \int_{\mathbb R^d} |D^2\tilde v^d(y)-D^2\tilde v^{d,\theta}(y)|^w\mu^d(dy)<\epsilon/2 ,
\end{align*}
where we have used Assumption \ref{ass:finite_dim_sobolev_approx} to find a parameter $\theta\in \Theta$ such that the above hold. The claim follows.
\end{proof}
\begin{remark}\label{eq:operatornormD2vdtheta}
    For  $d\in \mathbb N$ fixed,    $v^d$ acts as a $d$-dimensional function; then,  from the above proof,  we can pick $\theta$ such that 
    $$\int_{H} \|D^2 v^d(x)-D^2 v^{d,\theta}(x)\|^w\mu(dx)\leq C_d\int_{\mathbb R^d} |D^2\tilde v^d(y)-D^2\tilde v^{d,\theta}(y)|^w\mu^d(dy)<\epsilon,$$
i.e.~the approximation holds in the usual Sobolev norm $\|\cdot\|_{W^{2,w}_{\mu}(H)}$ defined by \eqref{eq:sobolev_norm},  a direct consequence of Assumption \ref{ass:finite_dim_sobolev_approx}.
\end{remark}
\subsection{Universal approximation under the action of unbounded operators}\label{subsec:UAT_unbounded}
In our PDE \eqref{eq:PDE}, unbounded operators  naturally appear (see e.g.~Section \ref{subsec:control_setup} where $L=A^*$). To successfully handle this case in Theorem \ref{th:DGM_residual}, we need to prove UATs under the action of unbounded operators.

Let $L:D(L)\subset H \to H$ be a closed linear operator with dense domain and let $\{e_i\}$ be an orthonormal basis of $H$ such that $e_i\in  D\left(L\right)$ for all $i\in \mathbb N$.
\begin{theorem}\label{th:UAT_L}Let $v \in C^2(H)$  and assume that $Dv:H \to D(L)$ is such that  $LDv\in C^0(H;H)$. 
\begin{enumerate}
\item[(i)] Let $K,K'\subset H$ be compact and assume  
    \begin{align}\label{ass:APdDv}
        \sup_{x \in K}\modu{\langle L[Dv(x) - P_dDv(P_dx)],x\rangle} \xrightarrow{d \to \infty}0.
    \end{align} 
    Then for every $\epsilon>0$, there exist $d \in \mathbb N$ and $\theta \in \Theta$ such that \eqref{eq:convergence_Dvtheta} holds and \begin{equation}\label{eq:convergence_ADvtheta}
        \sup_{x \in K}|\langle L[D v(x)-D v^{d,\theta}(x)],x\rangle|< \epsilon.
    \end{equation}
    \item[(ii)] Let  $q\geq 1$, $w:=4$, $v\in C^{2}_{w,q}(H)$, and let  $q,w\geq 1$ and let $\mu,\mu'$ be Borel probability measures on $H$ such that  $\|\mu\|_{q},\|\mu'\|_{w} < \infty$. Assume  
\begin{align}\label{ass:APdDv_measure}
      \int_{ H}|\langle L[D v(x)-P_dD v(P_d x)],x\rangle|^2\mu(dx) \xrightarrow{d \to \infty}0.
\end{align}
Then there exist $d\in \mathbb N,\theta \in \Theta$ such that \eqref{eq:convergence_Dvtheta_L1} holds for $w=4$ and 
\begin{equation}\label{eq:convergence_ADvtheta_measure}\int_H |\langle L[D v(x)-D v^{d,\theta}(x)],x\rangle|^2\mu(dx)< \epsilon.
\end{equation}
\end{enumerate}
\end{theorem}
\begin{proof}
We only  prove point (ii), as the proof of (i) follows similar steps. Recall \eqref{eq:LDvd} and \eqref{eq:LDvd_theta}.
By \eqref{ass:APdDv_measure} there exists $\bar d\in \mathbb N$ such that, for all $d \geq \bar d,$
\begin{align*}
      \int_{ H}|\langle L(D v(x)-Dv^{d}(x)),x\rangle|^2\mu(dx) < \epsilon/2.
\end{align*}
Choose $d\geq \bar d$ such that the estimates in the proof of Theorem \ref{th:UAT_L2_noL} hold. 
For such a $d$, denoting as usual $\mu^d(\cdot) :=\mu (\mathcal E_d^{-1} (\cdot))$ the pushforward measure on $\mathbb R^d$, we  estimate 
\begin{align*}
   \int_H \Big|\inprod{L(Dv^{d}-Dv^{d,\theta})(x)}{x}\Big|^2\mu(dx)&\leq \|L(Dv^{d}-Dv^{d,\theta})\|_{L^{4}(H;H;\mu)}^{2} \|\mu\|_4^{2}\\
   &=  \left(\int_{ H} \left | \sum_{i=1}^d[ \partial_i \tilde  v^d(\mathcal{E}_{ d}( x))-\partial_i \tilde v^{d,\theta}\left(\mathcal{E}_{ d}(x)\right)]Le_i \right|^{4} \mu(dx)\right)^{1/2}\|\mu\|_4^{2}\\
   &
    \leq C_d \left(\sum_{i=1}^d \int_H \Big|\partial_i \tilde  v^d(\mathcal{E}_{ d}( x))-\partial_i \tilde v^{d,\theta}\left(\mathcal{E}_{ d}( x)\right)\Big|^{4} \mu(dx)\right)^{1/2}
    \\
    &=C_d \left(\sum_{i=1}^d \int_{\mathbb R^d} \big|\partial_i \tilde  v^d(y)-\partial_i \tilde v^{d,\theta}\left(y\right)\big|^{4} \mu^d(dy)\right)^{1/2},
\end{align*}
 with  $C_d=\tilde C_d\max_{i=1,...,d} |Le_i|^2$ for $\tilde C_d>0$ (changing from line to line). Hence, again by Assumption \ref{ass:finite_dim_sobolev_approx}, we can find parameters $\theta \in \Theta$ such that this term is smaller than $\epsilon/2$ and \eqref{eq:convergence_Dvtheta_L1} holds for $w=4$. 
\end{proof}

Conditions \eqref{ass:APdDv}, \eqref{ass:APdDv_measure}  are a consistency condition for cylindrical approximations of $D v$ under the action of the unbounded operator $L$. This condition is crucial, since the unboundedness of $L$ prevents continuity with respect to convergence in $H$, so convergence of $P_d D v\left(P_d x\right)$ alone does not imply convergence of $L P_d D v\left(P_d x\right)$.
Nevertheless, there are numerous cases when \eqref{ass:APdDv} or \eqref{ass:APdDv_measure} is satisfied. We first discuss when \eqref{ass:APdDv_measure} holds and leave the discussion of \eqref{ass:APdDv} to Remark \ref{rem:LDv_K}. Write
\begin{align*}
\int_{ H}|{\langle L[D v(x)-P_dD v(P_dx)],x\rangle}|^2\mu(dx)& \leq \int_{ H}|{\langle L(I-P_d)D v(x),x\rangle}|^2\mu(dx)\\
&\quad+ \int_{ H}|{\langle LP_d[D v(x)-D v(P_d x)],x\rangle}|^2\mu(dx)=:E_1^d+E_2^d.
\end{align*}
Thus, \eqref{ass:APdDv_measure} is satisfied, for example, in the following cases:
\begin{enumerate}[leftmargin=*,nosep]
\item The operator $L\in \g(H)$ is bounded, as then $|L(I-P_d)x|\leq \|L\||(I-P_d)x |\xrightarrow{d \to \infty}0$ for all $x \in H$.
\item Assume that $v=v^{\bar d}$ for some $\bar d \in \mathbb N$, i.e.~$v$ acts as a $d$-dimensional cylindrical function. In this case, for $d \geq \bar d$, we have $E_1^d=E_2^d=0$.
\item We have $LP_dx=P_dLx$ for all $x \in D(L)$ (note that this is the case when the operator $L$ is diagnonalized by the orthonormal basis $\{e_i\}_{i=1}^\infty$, i.e.~there exist scalars $\{\lambda_i\}_{i=1}^\infty$ such that $Le_i = \lambda_ie_i$ for each $e_i$). Indeed, assume moreover that there exists $C>0$ such that
$|LDv(x)|^4\leq C(1+|x|^q),$ for all $  x \in H$. Under these conditions, by dominated convergence, we have 
\begin{align*}
    E_1^d=\int_{ H}|\langle (I-P_d)LD v(x),x\rangle|^2\mu(dx)&\leq \|\mu\|^2_4 \left(\int_{ H}|(I-P_d)LD v(x)|^4\mu(dx)\right)^{1/2} \xrightarrow{d\to \infty}0,\\
    E_2^d=\int_{ H}|\langle P_d[LD v(x)-LD v(P_d x)],x\rangle|^2\mu(dx)&\leq \|\mu\|^2_4 \left( \int_{ H}|LD v(x)-LD v(P_d x)|^4\mu(dx)\right)^{1/2} \xrightarrow{d\to \infty}0,
\end{align*}
where we have used the dominated convergence theorem and, for $E_2^d$,  we have used the continuity of $LDv$.

\item The following examples are particularly important as they give standard implementable bases in numerical analysis. Let $H=L_{\nu}^2(\mathcal O)$, where $\mathcal O \subset \mathbb R^d$ and $\nu:\mathcal O \to (0,\infty)$ is a suitable weight function. Assume there exists $\alpha \geq 0$ such that  $D(L)\subset H_\nu^{\alpha}(\mathcal O)$ and $L$ is continuous in the Sobolev norm $|\cdot|_{H_\nu^{\alpha }}$, i.e.~there exists $C>0$ such that $|Lx| \leq C\|x\|_{H_\nu^{\alpha}}$ for all $ x \in D(L)$. This is typical in many important differential operators. Notice that
 \begin{align*}
    E_1^d&\leq\int_{ H}|L(I-P_d)D v(x)|^2 |x|^2\mu(dx)\leq C \int_{ H}|(I-P_d)D v(x)|_{H_\nu^{\alpha }}^2 |x|^2\mu(dx)\\
    &\leq C \|\mu\|^2_4 \left( \int_{ H}\|(I-P_d)D v(x)\|_{H_\nu^{\alpha }}^4 \mu(dx)\right)^{1/2} \\
    E_2^d&=\int_{ H}|LP_d[D v(x)-D v(P_d x)]|^2|x|^2\mu(dx)\leq C \int_{ H}|P_d[D v(x)-D v(P_d x)]|_{H_\nu^{\alpha }}^2|x|^2\mu(dx)\\
    &\leq C \|\mu\|^2_4\left(\int_{ H}\|P_d[D v(x)-D v(P_d x)]\|_{H_\nu^{\alpha }}^4\mu(dx)\right)^{1/2}.
\end{align*}
Therefore, if in addition $Dv$ is continuous from $H$ to a suitable fractional Sobolev space of order $s>\alpha$ large enough and an appropriate integrability condition on $|Dv(x)|_{H_\nu^{s}}$ with respect to $\mu$ is satisfied, then we can apply standard results in numerical analysis for $\{e_i\}_{i=1}^\infty \subset D(L)$ given by standard classes of $L^2_\nu$-orthogonal bases to show $E_1^d,E_2^d\xrightarrow{d\to \infty} 0$.

As an example, we discuss the case of the Fourier system using \cite[Theorem 1.1]{canuto_quarteroni}, but a similar discussion can be done, for e.g.~the Chebyshev Spectral Projection System  \cite[Theorem 2.2]{canuto_quarteroni}, the Legendre Spectral Projection System  \cite[Theorem 2.4]{canuto_quarteroni}, or Laguerre polynomials \cite[Theorem 12.3]{bernardi_mayday}.
Let ${\mathcal O}=I^n$, $I=(-\pi, \pi)$, and let the weight be $\nu=1$. For multi-integer $\mathbf{k} \in \mathbb{Z}^d$, we set $|\mathbf{k}|_{\infty}=\max _{1 \leqslant j \leqslant d}\left|k_j\right|$.
We consider the set $\left\{e_{\mathbf{k}} : \mathbf{k} \in \mathbb{Z}^d\right\}$ with $e_{\mathbf{k}}(\boldsymbol{\theta})=(2 \pi)^{-d / 2} \exp (i \mathbf{k} \cdot \boldsymbol{\theta})$, which forms a complete orthonormal system in $L^2({\mathcal O})$. Set $\tilde{S}_d=\operatorname{span}\left\{e_{\mathbf{k}} : \left|\mathbf{k}\right|_{\infty}<d\right\} \subset L^2(\mathcal O)$ and denote by $\tilde P_d$ the  projection operator over $\tilde{S}_d$. Then by \cite[Theorem 1.1]{canuto_quarteroni},  for any real $0<\alpha<s$ there exists  $C>0$ such that
\begin{align} \label{eq:canuto_bound}
\left\|(I-\tilde {P}_d) y\right\|_{H^\alpha}<C d^{\alpha-s}|y|_{H^s} \quad \forall y \in H_{\mathrm{per}}^s({\mathcal O}) .
\end{align}
This estimate and the inverse triangle inequality imply
\begin{align}
\left\|\tilde {P}_d y\right\|_{H^\alpha}<\left\|y\right\|_{H^\alpha}+C d^{\alpha-s}|y|_{H^s} \quad \forall y \in H_{\mathrm{per}}^s({\mathcal O}) .
\end{align}
Hence, for the case with periodic boundary conditions (see Appendix \ref{subsec:notation} for notation), if there exists $s>\alpha$ such that $Dv:H\to  H_{\mathrm{per}}^s(\mathcal O)$, $Dv \in C(H;H_{\mathrm{per}}^s(\mathcal O))$ and $\|Dv(x)\|_{H^s}^4 \leq C(1+|x|^q)$ for all $x \in H$, then by the dominated convergence theorem we have $E_1^d,E_2^d\xrightarrow{d\to \infty} 0$ for the basis $\{e_i\}_{i=1}^\infty$ (up to a potential relabeling). Finally, we observe that  bounds analogous to \eqref{eq:canuto_bound} hold  when $\{e_\mathbf{k} : \mathbf{k} \in \Z^d\}$ instead represents the eigenmodes of the heat operator with Dirichlet or Neumann boundary conditions, so that the same argument can be made  if $Dv \in C(H; H^s_\mathrm{Dir}(\mathcal{O}))$ or $Dv \in C(H; H^s_\mathrm{Neu}(\mathcal{O}))$, respectively.

We remark that definition of a classical solution of \eqref{eq:PDE} (Definition \ref{def:classical_sol}) implies that $Dv\in C(H;D(L))$, where $D(L) \subset H^{\alpha}(\mathcal O)$ is endowed with the graph norm. Here, we are using slight additional regularity of $Dv$ and a growth bound.

\item  Assume that $\mu([D(L^*)]^c)=0$ and $\int_{ D(L^*)}  |L^*x|^p\mu(dx)<\infty$ for some $p>2$. Assume also that $| D v(x)|^{\frac {2p}{p-2}}\leq C(1+|x|^q).$ 
Then \eqref{ass:APdDv_measure} is satisfied as we have
\begin{align*}
      &\int_{ H}|\langle LD v(x)-LP_dD v(P_d x),x\rangle|^2\mu(dx) =      \int_{D(L^*)}|\langle D v(x)-P_dD v(P_d x),L^*x\rangle|^2\mu(dx) \\
      &\qquad \leq \left(\int_{D(L^*)}| D v(x)-P_dD v(P_d x)|^{\frac {2p}{p-2}}\mu(dx) \right)^{\frac {p-2}p}
      \left (\int_{ D(L^*)}  |L^*x|^p\mu(dx)\right)^{2/p} \xrightarrow{d\to \infty}0, 
\end{align*}
with dominated convergence implying convergence of the first term. 
\item Suppose the operator $L$ is maximally dissipative and there exists a compact operator $B \in \g(H)$ satisfying the strong $B$-condition with $c_0>0$ \cite[Definition 3.10]{fabbri2017book} (i.e.~$B$ is strictly positive, self-adjoint, $LB \in \g(H)$, and $-L B+c_0 B \geq I$ for some $c_0>0$). Indeed, by \cite[Lemma 3.17 (i, ii)]{fabbri2017book}, we have that $D\left(L\right) = D\left(B^{-1}\right)$ and the operator $S:=-L B+c_0 B \in \g(H)$ is invertible with $S^{-1} \in \g(H)$. Now pick an orthonormal basis $\{e_i\}_{i=1}^\infty \subset H$ of eigenvectors of $B$ such that $e_i=\frac{1}{\lambda_i}Be_i\in \mathrm{Range}(B)= D(L)$. Then $P_dB=BP_d$ on $H$ and $P_dB^{-1}=B^{-1}P_d$ on $D(B^{-1})$. By the definition of $S$, we have $L=-S B^{-1}+c_0 I,$ so $$LP_dx=- S B^{-1} P_dx+c_0 P_d x=-S P_dB^{-1} x+c_0 P_d x, \quad \forall x \in D(B^{-1})=D(L).$$
Then for all $x \in D(B^{-1}) = D(L)$, it follows that
\begin{align*}
|LP_dx|&\leq \|S\| |B^{-1} x|+c_0 | x|,\\
|L(I-P_d)x|&= |-S (I-P_d) B^{-1} x+c_0 (I-P_d) x|\leq \|S\| |(I-P_d) B^{-1} x|+c_0 | (I-P_d) x|.
\end{align*}
Next assume that there exists $C>0$ such that
$|Dv(x)|^4\leq C(1+|x|^q)$ and $|LDv(x)|^4\leq C(1+|x|^q),$ for all $  x \in H$; since $L=-S B^{-1}+c_0 I,$ the latter is equivalent to  $|B^{-1}Dv(x)|^4\leq C(1+|x|^q),$ for all $  x \in H$.
Similarly $LDv$ is continuous if and only if $B^{-1}Dv$ is continuous. Then we can proceed as for point 3 to prove $E_1^d,E_2^d\xrightarrow{d\to \infty}0.$
\item Assume there exists a compact positive self-adjoint operator $B\in L(H)$ such that  $LB^{1/2}\in L(H)$\footnote{This is a stronger condition than the weak  $B$-condition  in \cite[Chapter 3]{fabbri2017book}, see \cite{defeo_federico_swiech}.}. Pick an orthonormal basis $\{e_i\}$ made of eigenvectors of $B$ (which are such that $BP_d=P_dB$ and such that, for all $i$,  $e_i=\frac 1{\sqrt \lambda_i} B^{1/2}e_i\in R(B^{1/2})\subset D(L)$, where $\lambda_i$ is the  corresponding eigenvalue). Let $H_{-1}$ be the completion of  $H$ under the weaker norm $|x|_{-1}=|B^{1/2}x|$ and note that $P_d$ extends to a  orthogonal projection  operator on $H_{-1}$ as a bounded operator on $L(H_{-1})$. Moreover assume $v \in C^1({H_{-1}})$ (for some conditions such that this is true see \cite{defeo_swiech_wessels,defeo_gozzi_swiech_wessels}, covering HJB from control of stochastic differential delay equations or first order SPDEs), so that $Dv(x)=BD_{-1}v(x)$, for all $x \in H$. Then, we have
\begin{align*}
|L(I-P_d)Dv(x)|&=|L(I-P_d)B^{1/2}B^{1/2}D_{-1}v(x)|= | LB^{1/2}B^{1/2}(I-P_d)D_{-1}v(x)|\\
   &\leq C | (I-P_d)D_{-1}v(x)|_{-1}\\
| LP_d[Dv(x)-Dv(P_d x)]|&= | LP_dB^{1/2}B^{1/2}[D_{-1}v(x)-D_{-1}v(P_d x)]|\\
&=| LB^{1/2}B^{1/2} P_d[D_{-1}v(x)-D_{-1}v(P_d x)]|\\
   &\leq \|LB^{1/2}\|  |P_d[D_{-1}v(x)-D_{-1}v(P_d x)]|_{-1}\leq C  |D_{-1}v(x)-D_{-1}v(P_d x)|_{-1}.
\end{align*}
Then if $|Dv(x)|_{-1}^4\leq C(1+|x|^q),$ since $v\in C^1({H_{-1}})$, by dominated convergence we have $E_1^d,E_2^d\xrightarrow{d\to \infty} 0$.
\end{enumerate}
\begin{remark}\label{rem:LDv_K}
Similar considerations can be repeated to show that \eqref{ass:APdDv} holds (without any growth conditions on $Dv$). For instance, in an analogous setting to point 3, a term of the form
    $\tilde E_1^d=\sup_{x\in K}|\langle (I-P_d)LD v(x),x\rangle|\leq C\sup_{y\in LDv(K)} |(I-P_d)y|$ goes to zero, thanks to Lemma \ref{lemma1:uniformcompactsfromstrongandequibounded} as the set $LDv(K)$ is compact by continuity of $LDv$. Instead,  $\tilde E_2^d=\sup_{x\in K}|\langle P_dL[D v(x)-D v(P_d x)],x\rangle|\leq C\sup_{x\in K} |LD v(x)-LD v(P_d x)| $ goes to zero by uniform continuity of $LDv$ over $\tilde K$ (recall the notation $\tilde K$ in Remark \ref{rem:relatively-compact_Pd-I}).
\end{remark}

\section{HGNOs can solve PDEs on Hilbert spaces}\label{sec:HGNO_pde}
In this section, we prove that can HGNOs approximate solutions to our class of PDEs on Hilbert spaces, in the sense that the associated PDE residual can be made arbitrarily small, both uniformly on compact sets and in $L^2_\mu$.

Consider our second-order fully nonlinear PDE \eqref{eq:PDE} and denote the PDE residual by 
\[(\mathcal F v)(x):= \inprod{LDv(x)}{x} + F(x,v(x),Dv(x),D^2v(x)),\quad x \in H.\]
Let $\{e_i\}\subset D(L)$ be an orthonormal basis of $H$.
\begin{theorem}\label{th:DGM_residual} Let Assumptions \ref{ass:F_estimates}, \ref{ass:F_seq_cont} hold. Assume there exists a classical solution $v$ of  \eqref{eq:PDE}. 
\begin{enumerate}
    \item[(i)] {Let $K,K'\subset H$ be compacts and assume \eqref{ass:APdDv} holds. Then, for any $\epsilon>0$ there exists $d\in \mathbb N, \theta \in \Theta$ such that \eqref{eq:convergence_Dvtheta}, \eqref{eq:convergence_ADvtheta} hold and
    $$\sup_{x\in K} |(\mathcal F v^{d,\theta})(x)|<\varepsilon.$$}
    \item[(ii)] Let $k$ be as in Assumption \ref{ass:F_estimates} and assume that  there exists $C>0$ such that
\begin{equation}\label{eq:est_classical_sol}
    |v(x)|+|Dv(x)|+\nor{D^2v(x)} \leq C(1+|x|^k),\quad \forall x \in H.
\end{equation}
Let $\mu,\mu'$ be Borel probability measures on $H$ such that $\|\mu\|_q,\|\mu'\|_w< \infty$ for $q=4k$, $w=4$, and assume that \eqref{ass:APdDv_measure} holds. Then, for any $\epsilon>0$ there exists $d\in \mathbb N, \theta \in \Theta$ such that \eqref{eq:convergence_Dvtheta_L1} holds for $w=4$, \eqref{eq:convergence_ADvtheta_measure} holds,  and
\begin{equation}
    \nor{\mathcal{F} v^{d,\theta}}_{L^2(H;\mu)}^{2} := \int_H |(\mathcal{F} v^{d,\theta})(x)|^2 \mu(dx) < \epsilon.
\end{equation}
\end{enumerate}
\end{theorem}
\begin{proof}Proof of (ii). Consider $v^{d,\theta}$ for $d\in \mathbb N, \theta \in \Theta$. Then, since $v$ is a classical solution of \eqref{eq:PDE}, i.e.~$(\mathcal Fv)(x)\equiv 0,$ 
 \[
\begin{aligned} 
\|\mathcal{F} v^{d,\theta}\|_{L^2(H;\mu)}^2
&=\|\mathcal{F} v^{d,\theta}-\mathcal{F} v\|_{L^2(H;\mu)}^{2}\\
&\leq \|\inprod{L(Dv^{d,\theta}-Dv )}{\cdot}\|_{L^2(H;\mu)}^{2} +\|F(\cdot,v^{d,\theta},Dv^{d,\theta},D^2v^{d,\theta})-F(\cdot,v,Dv^{d,\theta},D^2v^{d,\theta})\|_{L^2(H;\mu)}^{2} \\
&\quad +\|F(\cdot,v,Dv^{d,\theta},D^2v^{d,\theta})-F(\cdot,v,Dv,D^2v^{d,\theta})\|_{L^2(H;\mu)}^{2}\\
&\quad +\|F(\cdot,v,Dv,D^2v^{d,\theta})-F(\cdot,v,Dv,D^2v)\|_{L^2(H;\mu)}^{2}=:E_1^{d,\theta}+E_2^{d,\theta}+E_3^{d,\theta}+E_4^{d,\theta}.
\end{aligned}
\]
By Assumption \ref{ass:F_estimates}, we have
    \begin{align*}
       & E_2^{d,\theta}=\int_H |F(x,v^{d,\theta}(x),Dv^{d,\theta}(x),D^2v^{d,\theta}(x))-F(x,v(x),Dv^{d,\theta}(x),D^2v^{d,\theta}(x))|^2 \mu(dx)\\
      &\quad \quad 
      \leq C \int_H |v^{d,\theta}(x)-v(x)|^2 (1+|x|^{2k})\mu(dx) \leq C(1+\|\mu\|_{4k})^{2k}\|v^{d,\theta}-v\|_{L^{4}(H;\mu)}^{2},\\
      &E_3^{d,\theta}=\int_H |F(x,v(x),Dv^{d,\theta}(x),D^2v^{d,\theta}(x))-F(x,v(x),Dv(x),D^2v^{d,\theta}(x))|^2 \mu(dx) \\
      &\quad \quad 
      \leq C\int_H |Dv^{d,\theta}(x)-Dv(x)|^2(1+|x|)^{2k} \mu(dx)\leq C (1+\|\mu\|_{4k})^{2k} \|Dv^{d,\theta}-Dv\|_{L^{4}(H;H;\mu)}^{2}.
    \end{align*}
    \item Consider $E_4^{d,\theta}$.   Recalling Notation \ref{def:cylindrical_approx} and using  Assumption \ref{ass:F_estimates} and \eqref{eq:est_classical_sol}, we have
         \begin{align*}
    E_4^{d,\theta}&\leq \int_{H} |F(x,v(x),Dv(x),D^2v^{d,\theta}(x))-F(x,v(x),Dv(x),D^2v^d(x))|^2 \mu(dx)\\
   &\quad  +\int_{H} |F(x,v(x),Dv(x),D^2v^d(x))-F(x,v(x),Dv(x),D^2v(x))|^2 \mu(dx)\\
   &\leq \int_{H} \|D^2v^{d,\theta}(x)-D^2v^d(x)\|^2(1+|x|^{2k}) \mu(dx)\\
   &\quad  +\int_{H} |F(x,v(x),Dv(x),D^2v^d(x))-F(x,v(x),Dv(x),D^2v(x))|^2 \mu(dx)\\
&\leq  (1+\|\mu\|_{4k})^{2k} \|D^2v^{d,\theta}-D^2v^d\|_{L^{4}(H;\mathcal L(H);\mu)}^{2}\\
   &\quad  +\int_{H} |F(x,v(x),Dv(x),D^2v^d(x))-F(x,v(x),Dv(x),D^2v(x))|^2 \mu(dx) =:E_{4,1}^{d,\theta}+E_{4,2}^{d}.
      \end{align*}
By Remark  \ref{eq:operatornormD2vdtheta_K}(1) and Assumption \ref{ass:F_seq_cont}, we have $|F(x,v(x),Dv(x),D^2v^d(x))-F(x,v(x),Dv(x),D^2v(x))|^2\xrightarrow{d\to \infty} 0$, for every fixed $x\in H$. By  Assumption \ref{ass:F_estimates} and \eqref{eq:est_classical_sol}, we can apply the  
the dominated convergence theorem, to find $\bar d\in \mathbb N$ such that for all $d\geq  \bar d$ it holds that $E_{4,2}^{d}<\frac \epsilon {8}$.  Finally, using  Theorem \ref{th:UAT_L}(ii),  there exist $d\geq \bar d$ and $\theta \in \Theta$ such that  $E_1^{d,\theta}<\frac \epsilon {4}$, $E_2^{d,\theta}<\frac \epsilon {4}$, $E_3^{d,\theta}<\frac \epsilon {4}$,  and $E_{4,1}^{d,\theta}<\frac \epsilon {8}$ (recall also Remark \ref{eq:operatornormD2vdtheta} for $E_{4,1}^{d,\theta}$) and such that \eqref{eq:convergence_Dvtheta_L1} holds for $w=4$.

The proof of (i) follows similar steps. However, for the critical term
$$\sup_{x\in K}|F(x,v(x),Dv(x),D^2v^d(x))-F(x,v(x),Dv(x),D^2v(x))|,$$ we use Lemma \ref{rem:continuity_PhiF}  (which can be applied thanks to Remark  \ref{eq:operatornormD2vdtheta_K}(1) ) over the compact  $\mathcal K:=K\times v(K) \times Dv(K)$.
 \end{proof}
 \begin{remark}\label{rem:mild_sol_PDE}
The previous theorem holds for classical solutions of the PDE. However, it is also motivated by mild solutions of the PDE, as we now explain. The mild solution is another very popular notion of solution.
In the elliptic case, this is typically defined when $F$ is of the form $F=-\gamma v+\mathcal A v+\tilde F(x,v,p)$, where $\mathcal{A} v=\frac{1}{2} \operatorname{Tr}\left[\sigma(x)\sigma^*(x) D^2 v\right]+\langle A x+b(x), D v\rangle$ is the infinitesimal generator of the corresponding stochastic evolution equation. The mild solution of the PDE\footnote{Not to be confused with the mild solution of the stochastic evolution equation.} exploits the  Markov semigroup $P_s[v](x):=\mathbb{E}[v(X^{x}_s)]$ related to the stochastic evolution equation, i.e.~
\begin{equation}\label{eq:HJB_mild_rem}
    v(x)=\int_0^{+\infty} e^{-\gamma s} P_s[\tilde F(\cdot, v, D v)](x) d s.
\end{equation}
Powerful existence, uniqueness and regularity results can be obtained for mild solutions, making it a popular choice for these PDEs. Although \eqref{eq:HJB_mild_rem} seems very inefficient to implement numerically, mild solutions can typically be obtained as the limit of classical solutions of slightly perturbed PDEs \cite{fabbri2017book}. We speculate that our method in Section \ref{sec:DHGM} can be applied in the form presented in this paper (i.e.~minimizing the $L^2_\mu$-residual of the PDE without considering the form \eqref{eq:HJB_mild_rem}) to mild solutions, with a theoretical justification that may come by means of suitable perturbation arguments\footnote{The training of deep neural networks with stochastic gradient descent is well known to result in (implicit) regularization of the function approximation method. We hypothesize that this will typically lead to solutions which are stable under perturbations, supporting the application of our techniques to mild solutions of these PDEs. However, a full analysis of this setting is beyond the scope of this current paper.}.
While a full theory of this extension and its numerical tests will require a full future investigation, this motivates us  even more to consider classical solutions in the present paper.
\end{remark}

Under stronger conditions, we next seek to prove a complementary result to Theorem \ref{th:DGM_residual}. Namely, we will bound the solution ansatz error $\nor{v - v^{d,\theta}}_{L^2(H;\mu)}$ in terms of the PDE operator error $\nor{\mathcal{F}v^{d,\theta}}_{L^2(H;\mu)}$. This is helpful for verifying the convergence of numerical algorithms, as we can accurately estimate $\nor{\mathcal{F}v^{d,\theta}}_{L^2(H;\mu)}$ even for {HGNOs} with relatively low dimension $d$ by sampling points from $\mu$ with dimension $N \gg d$.

\begin{assumption}\label{ass:kolmogorov}
    Suppose the PDE \eqref{eq:PDE} is of the form
    \begin{equation}\label{eq:PDE_converse_est}
        (\mathcal{F}v)(x) := - \gamma v(x)+ \inprod{LDv}{x} + \Lambda(x)(Dv(x),D^2v(x)) + f(x) = 0\quad \forall x\in H,
    \end{equation}
    where $\Lambda : H \to \g(H\times S(H); \R)$. Suppose further that there exists a Borel probability measure $\mu$ on $H$ and an $\omega > 0$ such that
    \begin{equation}\label{eq:quasi_diss}
        \inprod{\mathcal{F}v}{v}_{L^2(H;\mu)} \leq -\omega \nor{v}_{L^2(H;\mu)}^2, \quad \forall v \in D(\mathcal{F}). 
    \end{equation}
\end{assumption}
\begin{example}
    The structure of $\mathcal{F}$ in Assumption \ref{ass:kolmogorov} accommodates the Kolmogorov equation \eqref{eq:kolmogorov} by setting $$\Lambda(x)(Dv(x), D^2v(x)) :=\mathcal Av(x): =\inprod{LDv}{x} +\inprod{Dv}{b(x,u(x))} + \frac{1}{2}\mathrm{Tr}\brac{\sigma(x,u(x))Q\sigma^*(x,u(x))D^2v(x)}$$ and $f(x)=l(x,u(x))$, where $b: H \times U \to H$, $\sigma: H \times U \to L(\Xi;H)$, and $l: H \times U \to \R$. 
    
   Further, define the Markov semigroup $P_t[v](x):=\mathbb{E}[v(X^{x}_t)]$, where  $X^{x}_t$ is the mild solution of \eqref{eq:stateSDE_H}.
   Under suitable conditions on the coefficients in \eqref{eq:stateSDE_H} and assuming that $\mu$ is a stationary distribution for $P_t$, we have that $P_t$ is a $C_0$-semigroup of contractions on $L^2(H;\mu)$  with generator $\mathcal A:D(\mathcal A)\subset L^2(H;\mu)\to L^2(H;\mu)$ \cite[Proposition 5.9 and Subsection 5.2.3]{fabbri2017book}. Thus, by the Lumer--Phillips Theorem, the generator $\mathcal{A}$ is maximally dissipative and, therefore, Assumption \ref{ass:kolmogorov} is satisfied with $\omega = \gamma > 0$.
\end{example}
\begin{theorem}\label{th:l2_convergence}
    Suppose Assumption \ref{ass:kolmogorov} holds and let $v \in C^2(H)$ be a classical solution of \eqref{eq:PDE_converse_est}. Then
    \begin{equation}\label{eq:bounded_inverse}
        \nor{v - w}_{L^2(H;\mu)} \leq \frac{1}{ \omega} \nor{\mathcal Fw}_{L^2(H;\mu)},\quad \forall w \in D(\mathcal{F})=D(\mathcal A):=\{v\in C^2(H):LDv\in C^0(H;H)\}.
    \end{equation}
\end{theorem}
\begin{remark}
Theorem \ref{th:l2_convergence} implies that if, for some $d\in \mathbb N,\theta \in \Theta$, we have  $\nor{\mathcal Fv^{d,\theta}}_{L^2(H;\mu)}<\epsilon$, then 
    $$\nor{v - v^{d,\theta}}_{L^2(H;\mu)}
    < \frac{\epsilon}{\omega}.$$
    
\end{remark}
\begin{proof}By the linearity of $\mathcal{F}$ and $\omega$-dissipativity, we have
    \begin{align*}
         \inprod{\mathcal{F}w}{w-v}_{L^2(H;\mu)} &= \inprod{\mathcal{F}w-\mathcal{F}v}{w-v}_{L^2(H;\mu)}= \inprod{\mathcal{F}(w-v)}{w-v}_{L^2(H;\mu)}\leq -\omega \nor{w-v}_{L^2(H;\mu)}^2.
    \end{align*}
Thus $\omega \nor{w-v}^2  \leq \nor{\inprod{\mathcal F w}{w-v}}_{L^2(H;\mu)} 
        \leq \nor{\mathcal Fw}_{L^2(H;\mu)}\nor{w-v}_{L^2(H;\mu)}.$
 The claim follows.
\end{proof}
In certain cases, such as infinite-dimensional Ornstein--Uhlenbeck processes, we can explicitly solve for and immediately sample from the stationary distribution to utilize Theorem \ref{th:l2_convergence}. In fact, this is what we do in Section \ref{sec:stochastic_heat} for the stochastic heat equation. In others, such as overdamped Langevin diffusions, we may simulate the forward process until convergence to get samples.

\section{HGNOs can solve optimal control problems on Hilbert spaces}\label{sec:NHO_control}
In this section, we specialize our results and provide further analysis of optimal control problems on Hilbert spaces (introduced in Section \ref{subsec:control_setup}). In particular, we show how to obtain universal approximation of optimal feedback controls in terms of our approximate value function HGNO.
\paragraph{Synthesis of optimal feedback controls.} Consider the control problem of Section \ref{subsec:control_setup}. We recall how classical solutions of the HJB equation \eqref{eq:HJB_control} are used to synthesize optimal feedback controls \cite[Corollary 2.43]{fabbri2017book}. Let $v:H \to \mathbb R$ be a classical solution of \eqref{eq:HJB_control}.
Define the multivalued maps
\begin{align}
&\Psi: H\times H\times S(H) \rightarrow \mathcal{P}(U), \quad \Psi(x,p,Z):=\mathop{\mathrm{arg\,min}} _{u \in U} F^{cv}\left(x, p, Z, u\right),\label{eq:mapPsi}\\
& \Phi: H \rightarrow \mathcal{P}(U), \quad
\Phi(x):=\Psi(x,Dv(x),D^2 v( x))=\mathop{\mathrm{arg\,min}}_{u \in U} F^{cv}\left(x, D v(x), D^2 v( x), u\right).\label{eq:mapPhi}
\end{align}
\begin{proposition}[Optimal Feedback Controls]\label{prop:feedbacks}Let Assumption \ref{ass:coefficients_control} hold.
Assume that $v$, $D v$, and $D^2 v$ are uniformly continuous on bounded subsets of $H$. Moreover, let $D v: H \rightarrow D\left(A^*\right)$ and suppose $A^* D v$ is uniformly continuous on bounded subsets of $H$, and that there exists $C>0$ such that
$$
|v(x)|+|D v(x)|+\left|A^* D v(x)\right| +\left\|D^2 v(x)\right\|\leq C(1+|x|)^k,\quad \forall x \in H.
$$
Assume that $\gamma >(k+2)\left(C+\frac{1}{2}(k+1) C^2\right)$, where $C$ is the constant from the growth of $b,\sigma$ in \eqref{eq:estimates_b}, \eqref{eq:estimates_sigma}. Assume that $\Phi$ has a Borel measurable selection\footnote{That is, there is a Borel measurable function $\varphi(x)\in \Phi(x)$ for all $x\in H$.} $\varphi:H \to U$ such that the corresponding closed loop equation (CLE)
$$
d X_s = A X_s d t+b(X_s , \varphi(X_s )) d s+\sigma(X_s , \varphi(X_s )) d W^Q_s,\quad X_0=x
$$
has a weak mild solution (see \cite[Definition 1.121]{fabbri2017book}) $X_s$ in some generalized reference probability space, for all $x\in H$. Then the pair $\left(X,u \right)$, where the control $u$ is defined by the feedback law $u_s=\varphi\left(X_s\right)$, is admissible and optimal at $x$ and $v$ is the value function.
\end{proposition}

\paragraph{HGNOs can solve optimal control problems on Hilbert spaces.}
Let Assumption \ref{ass:coefficients_control} hold. Then Assumption \ref{ass:F_estimates} holds, i.e.~we have \eqref{eq:est_F_control}, and Remark \ref{seq:continuity_Fcv_compact_open} applies. Furthermore, assume that the control set $U$ is compact so that Lemma  \ref{seq:continuity_Fcv_F_compact_open} applies. Let $\{e_i\}\subset D(L)$ be an orthonormal basis of $H$, for $L:=A^*$.
Then, under the corresponding assumptions, we can apply Theorems \ref{th:UAT_L} and \ref{th:DGM_residual} to the HJB equation \eqref{eq:HJB_control}, i.e.~by denoting by $v$ a classical solution of the HJB equation \eqref{eq:HJB_control}, we can find $d\in \mathbb N, \theta \in \Theta$ such that Theorems \ref{th:UAT_L} and \ref{th:DGM_residual} hold.  Moreover,  let the assumptions of Proposition \ref{prop:feedbacks} hold; then by the proposition, $v$ is equal to the value of the control problem. 

In addition, assume that the multivalued map \eqref{eq:mapPhi}  (which is non-empty valued thanks to the compactness of $U$) is single-valued, i.e.~$\Psi: H \times H \times S(H)\rightarrow U$. 
Then the optimal feedback map $\Phi$, defined in \eqref{eq:mapPhi}, is  also single-valued, $\Phi:H\to U$. 
\begin{remark}
    If $U$ is convex and for all $x,p,Z,$ the map $U\ni u\mapsto F^{cv}(x,p,Z,u)$ is strictly convex\footnote{A function $f:U\to \mathbb R$ is  strictly convex if
$
f\left((1-\lambda) x_1+\lambda x_2\right)<(1-\lambda) f\left(x_1\right)+\lambda f\left(x_2\right), $ $ 0<\lambda<1,
$
for any  $x_1\neq x_2$ in $U$.}, then its minimizer is unique \cite[Proposition 1.2]{ekeland_temam}. In this case, $\Psi$ (and $\Phi$) are single-valued.
\end{remark}
In this setting, we have the following theorem.
\begin{theorem}\label{th:UAT_optimal_controls}
    Let $K,K' \subset H$ be compacts and $\mu,\mu'$ be Borel probability measures on $H$. Define $\Phi^{d,\theta}: H \rightarrow U$ by
\begin{equation} \label{eq:phi_dt_def}
\Phi^{d,\theta}(x):=\Psi(x, Dv^{d,\theta}(x), D^2v^{d,\theta}(x))=\mathop{\mathrm{arg\,min}}_{u \in U}  F^{cv}\left(x, Dv^{d,\theta}(x), D^2v^{d,\theta}(x), u\right),\quad x\in H. \end{equation}
Under the respective assumptions of points (i), (ii) in Theorem \ref{th:DGM_residual}, for every $\epsilon>0$, there exist $d\in \mathbb N, \theta \in \Theta$ such that
    \begin{enumerate}
        \item[(i)]  {the conclusions of Theorem \ref{th:DGM_residual}(i) hold and}
$\sup_{x \in K}|\Phi(x)-\Phi^{d,\theta}(x)|_{\tilde U}<\epsilon$,
\item[(ii)]{the conclusions of Theorem \ref{th:DGM_residual}(ii) hold and} $ \|\Phi-\Phi^{d,\theta}\|_{L^2(H;\tilde U;\mu)}^2<\epsilon$. 
    \end{enumerate}
\end{theorem}
\begin{proof}Let $\epsilon>0$. Let $M>0$, let $\tilde K$ be defined with the notation of  Remark \ref{rem:relatively-compact_Pd-I}. By Remark \ref{seq:continuity_Fcv_compact_open}, $F^{cv}:H\times   H\times S(H)\times U\to \mathbb R$ is sequentially continuous, when $S(H)$ is endowed with the compact-open topology.  Therefore, when we consider its restriction  to $H\times H\times S_M\times U $, where $S_M=\{Z\in S(H):\|Z\|\leq M\}$, $M>0$,  it upgrades to  a continuous map, since $S_M$, endowed with the compact-open topology, is a metrizable topological space (see Definition \ref{rem:compact_open_LX}). By compactness of $U$, we apply \cite[Berge Maximum Theorem 17.31 and Lemma 17.6]{aliprantis2006infinite} to have that $\Psi:H\times H\times S_M\to U$ is continuous. Choose $M$ such that $\sup_{x\in \tilde K}\|D^2v(x)\|\leq M$ (so we also have $\sup_{x\in  K}\|D^2v^d(x)\|\leq M$) and define $\Phi^{d}: H\rightarrow U$,
\begin{equation}
\Phi^{d}(x):=\Psi(x, Dv^{d}(x), D^2v^{d}(x))=\mathop{\mathrm{arg\,min}}_{u \in U}  F^{cv}\left(x, Dv^{d}(x), D^2v^{d}(x), u\right),\quad x\in H.
\end{equation}
(i) Using Remark \ref{eq:operatornormD2vdtheta_K}(1), the continuity of $\Psi:H\times H\times S_M\to U$, and a similar argument to the one in the proof of Lemma \ref{rem:continuity_PhiF},  there exists $\tilde d\in \mathbb N,$ such that for all $d\geq \bar d$,
$$\sup_{x\in K}|\Phi(x)-\Phi^{d}(x)|_{\tilde U}=\sup_{x\in K}|\Psi(x, Dv(x), D^2v(x))-\Psi(x, Dv^{d}(x), D^2v^{d}(x))|_{\tilde U}<\epsilon/2.$$
We take a possibly larger $d\geq \tilde d,$ such that all estimates in the proof of Theorem \ref{th:DGM_residual}(i) are valid.

By Assumption \ref{ass:coefficients_control},  we apply \cite[Berge Maximum Theorem 17.31 and Lemma 17.6]{aliprantis2006infinite}, $\Psi:H\times H\times S(H)\to U$  is continuous, when $S(H)$ is endowed with the operator norm. By Remark \ref{eq:operatornormD2vdtheta_K}(2), there exists $ \theta \in \Theta$ such that the conclusions of Theorem \ref{th:DGM_residual} (i) hold and
$$\sup_{x \in K}|\Phi^d(x)-\Phi^{d,\theta}(x)|_{\tilde U}=\sup_{x\in K}|\Psi(x, Dv^{d}(x), D^2v^{d}(x))-\Psi(x, Dv^{d,\theta}(x), D^2v^{d,\theta}(x))|_{\tilde U}<\epsilon/2.$$

(ii) As above, 
$|\Phi(x)-\Phi^{d}(x)|_{\tilde U}=|\Psi(x, Dv(x), D^2v(x))-\Psi(x, Dv^{d}(x), D^2v^{d}(x))|_{\tilde U}\xrightarrow{d\to \infty}0$, for any $x\in H$. As the codomain of $\Psi$ is the bounded set  $U$, by the dominated convergence theorem, there exists $\tilde  d\in \mathbb N$ such that for all $d\geq \tilde d,$
$$\int_H|\Phi(x)-\Phi^{d}(x)|^2_{\tilde U}\mu(dx)<\epsilon/2.$$
We pick $d\geq \tilde d$ so that the estimates in the proof of Theorem \ref{th:DGM_residual}(ii) are valid. By Remark \ref{eq:operatornormD2vdtheta} there exists $\{\theta_n\}\subset \Theta:$   
$$\int_{H} | v^d(x)-v^{d,\theta_n}(x)|^4+|D v^d(x)-D v^{d,\theta_n}(x)|^4+|L(D v^d(x)-D v^{d,\theta_n}(x))|^4+ \|D^2 v^d(x)-D^2 v^{d,\theta_n}(x)\|^4\mu(dx) \xrightarrow{n\to \infty} 0.$$ Up to a subsequence, we have convergence  of the integrand to zero $\mu-$a.e. Since, by the proof of (i) $\Psi:H\times H\times S(H)\to U$  is continuous, when $S(H)$ is endowed with the operator norm, 
 we apply again the dominated convergence theorem to find $n$ such that 
\begin{align*}
    \int_H|\Phi^d(x)-\Phi^{d,\theta_n}(x)|^2_{\tilde U}\mu(dx)=  \int_{H}|\Psi(x, Dv^{d}(x), D^2v^{d}(x))-\Psi(x, Dv^{d,\theta_n}(x), D^2v^{d,\theta_n}(x))|^2_{\tilde U}\mu(dx)<\epsilon/2
\end{align*}
and such that the estimates in the proof of Theorem \ref{th:DGM_residual}(ii) hold for $\theta=\theta_n$.
The statement follows.
\end{proof}

\paragraph{Parametrization of feedback controls via HGNOs.} Just as we have parameterized the PDE solution by the HGNO ansatz $v^{d,\theta} : H \to \R$, it is natural to parameterize\footnote{If $U \subsetneq \tilde U$, we may need to apply a final layer to $\tilde{u}^{d',\phi,p} : \R^d\to \R^p$ so that the HGNO $u^{d',\phi,p} : H \to U$ does not take values in $\tilde U \setminus U$. For example, if $U = B_{\tilde U}(0,1) \subset \tilde U$, then we can apply the map $\Pi: \R^p \to B_{\R^p}(0,1)$ given by $\Pi_{\R^p}(x) = \tfrac{x}{1+\modu{x}}$ to the outputs of $\tilde{u}^{d',\phi,p}$. Thus, the structure of the HGNO is $u^{d',\phi,p} = \hat{\mathcal E}_p^{\tilde U} \circ \Pi_{\R^p} \circ \tilde u^{d',\theta,p} \circ \mathcal{E}_{d'}^{H}.$ If $\Pi_{\R^p}$ is a homeomorphism, as is the case here, then $\Pi_{\R^p} \circ \tilde u^{d',\theta,p}$ universally approximates in $C_0$-norm on compacts, which is sufficient for the results in \cite{castro2022} and thus the purposes of this paper.} the optimal feedback control by an HGNO ansatz $u^{d',\phi,p} : H \to U$, for $d',p\in \mathbb N, \phi\in \Theta$. Since the activation function  $\mathfrak{m}\in C_b^{2}(\mathbb R)$, we have that $u^{d',\phi,p}$ is Lipschitz and thus an admissible control, i.e.~$u^{d',\phi,p}\in \mathcal U$. Further, given some prior estimate $v^{d,\theta}$ of the PDE solution, we can train the HGNO $u^{d',\phi,p}$ to achieve the argmin in \eqref{eq:phi_dt_def}. By standard universal approximation theorems in $C^0$-norm on compact sets $K \subset H$ and in $L_{\mu}^2(H)$ (see e.g.~\cite{castro2022}), the HGNO $u^{d',\phi,p}$ can be trained to approximate the continuous function $\Phi^{d,\theta}: H \to U$. Thus, if the conditions of both theorems are fulfilled, then given a compact set $K \subset H$, resp.~a Borel probability measure $\mu$, there exist $d',p \in \N$ with $\phi \in \Theta$ trainable from knowing only $v^{d,\theta}$ such that
$$\sup_{x\in K} \modu{\Phi^{d,\theta}(x) - u^{d',\phi,p}(x)}_{\tilde{U}} < \epsilon, \quad \text{resp.} \quad \nor{\Phi^{d,\theta} - u^{d',\phi,p}}_{L^2(H;\tilde{U};\mu)} <\epsilon.$$
This idea guides the development of Hilbert Actor-Critic methods in the next section. In particular, to solve the HJB equation \eqref{eq:HJB_control}, we alternate between training the \textbf{critic} function $v^{d,\theta} : H\to \R$ to solve the Kolmogorov problem given the current \textbf{actor} $u^{d',\phi,p}$ and training the actor $u^{d',\phi,p}: H \to U$ to try to learn the argmin in \eqref{eq:phi_dt_def} given the current critic $v^{d,\theta}$.

\section{Deep Hilbert--Galerkin Methods and Hilbert Actor-Critic Methods}\label{sec:DHGM}

Guided by the previous theoretical analysis, we now propose \textbf{Deep Hilbert--Galerkin Methods} for solving PDEs on $H$ via HGNOs. Certain problems, such as the Kolmogorov equation \eqref{eq:kolmogorov}, only require learning $v^{d,\theta}: H \to \R$.  Others, such as control problems like \eqref{eq:HJB_control}, necessitate initializing an additional auxiliary network $u^{d,\phi,p}: H \to U \subset \tilde{U}$ to learn the optimal control.

For problems which only require learning $v^{d,\theta} : H \to \R$, we propose Algorithm \ref{algo:fixed_control}, which has two variants: the DHGM gradient, inspired by DGM \cite{sirignano2018}, and the QHPDE gradient, inspired by QPDE \cite{cohen2023}. When using the DHGM gradient, $v^{d,\theta} : H \to \R$ is made to learn the solution of the PDE \eqref{eq:PDE} by training $\theta \in \Theta$ to minimize the $L_{\mu}^2(H)$-norm of the PDE residual with stochastic gradient descent (cf.~\eqref{algo:dgm_critic_loss}). When using the QHPDE gradient, $\theta \in \Theta$ is trained using a biased gradient of this same residual norm that assumes the PDE operator is monotone (cf.~\eqref{algo:qpde_critic_loss}). We remark that the QHPDE gradient often outperforms the DHGM gradient even when the PDE is not truly monotone, but merely contains a time-discounting term like $-\gamma v$ in \eqref{eq:kolmogorov}.

In the case of control problems which require learning both $v^{d,\theta} : H \to \R$ and $u^{d',\phi,p} : H \to U$, we conceptualize new types of methods which we call \textbf{``Optimize-then-Learn''}. The fundamental difference with the literature is that we directly attempt to solve the infinite-dimensional PDE and optimality condition rather than projected versions of them (which would be a ``Discretize-then-Optimize'' or ``Optimize-then-Discretize'' method, depending on whether the PDE or optimality condition is discretized first).  Inspired by techniques from reinforcement learning and finite-dimensional HJB equations, we introduce \textbf{Hilbert Actor-Critic Methods}, i.e. \textbf{Reinforcement Learning} algorithms in which the actor is the learned optimal control $u^{d,\phi,p} : H \to U$ and the critic is the learned value function/PDE solution $v^{d,\theta}: H\to\R$. The critic can be trained with the most recent actor, then the actor with the most recent critic, and this process repeat until an PDE solution/optimal control pair is learned. In particular, in Algorithm \ref{algo:ac}, we propose alternating between training $v^{d,\theta} : H\to \R$ using either a DHDM or QHPDE gradient and training $u^{d',\phi,p}: H \to U$ to achieve the argmin in \eqref{eq:phi_dt_def} via stochastic gradient descent in the parameters $\phi$ on the integral $$\int_H F^{cv}(v^{d,\theta}(x),Dv^{d,\theta}(x),D^2v^{d,\theta}(x),u^{d',\phi,p}(x))\mu(dx).$$

{In the context of Algorithms \ref{algo:fixed_control} and \ref{algo:ac}, Theorem \ref{th:DGM_residual} guarantees that HGNOs are capable of representing functions which solve the PDE of interest to arbitrarily low $L^2_\mu(H)$-norm, and uniformly on compact sets, providing theoretical grounding for these approaches.}

\begin{remark}\label{rem:extensions_DHGM}
Algorithms \ref{algo:fixed_control} and \ref{algo:ac} are stated for solving \eqref{eq:PDE} and \eqref{eq:HJB_control} when the PDE domain is $H$. However, extensions to PDEs of the form \eqref{eq:PDE_intro} with  boundary conditions and $k  \in \mathbb N$ are possible. For example, to handle when $k > 1$, we can define an HGNO with codomain $\R^k$ and restrict to using a DHGM gradient. For parabolic PDEs with domain $[0,T] \times H$ and a terminal boundary condition $V(T,\cdot) : H \to \R$, we can define $v_{\mathrm{output}}^{d,\theta}(t,x) := (T-t)v_{\mathrm{HGNO}}^{d,\theta}(t,x) + V(T,x)$ and sample from $\mu_{1,\mathrm{extended}} = \mathrm{Unif}([0,T]) \times \mu$.
\end{remark}

\begin{remark}\label{rem:def_N}
    Algorithms \ref{algo:fixed_control} and \ref{algo:ac} are intended to solve PDEs on the entire Hilbert space $H$. This differs from other approaches in the literature \cite{miyagawa2024physics, venturi2018numerical, rodgers_venturi} which solve the a projected PDE on a $d$-dimensional vector space, e.g.
    \begin{align}\label{eq:PDE_Hd}
        \inprod{LDv^d}{y} + F(y,v^d,Dv^d,D^2v^d)= 0, \quad  y\in P_d(H) \not\approx H.
    \end{align}
     {Our approach has the advantage of evaluating the PDE residual of \eqref{eq:PDE_intro} accurately by not approximating the variable $x \in H$ appearing in the terms $F(x,v,Dv,D^2v)$ and $\inprod{LDv}{x}$ with $P_dx$, but rather using the full variable $x \in H$.} Of course, computers have finite memory and are still only capable of sampling points $x \in H$ to a finite number of basis elements. 
    However, we can cheaply and easily sample $N \gg d$. Thus, we can sample so that $\E_{\mu}\modu{x - P_Nx}$ is completely negligible while maintaining low computational cost (whereas increasing $d$ is expensive). Since $F$ and $\langle LDv^{d,\theta}(\cdot),\cdot \rangle$ are continuous in $x$ (in particular, our choice $\{e_i\}\subset D(L)$ makes $LDv^{d,\theta}\in C^0(H;H)$, cf.~\eqref{eq:LDvd_theta}), the PDE \eqref{eq:PDE} can be approximated well when $N$ is large enough. Note that $LDv^{d,\theta}(x)$ might be outside $P_d(H)$ even though $Dv^{d,\theta}(x)\in P_d(H)$.
\end{remark}

There are many other potential algorithms that could be written and experimented with. For example, one could establish an algorithm that uses information obtained from Monte Carlo simulations of the controlled SDE (cf.~e.g.~\cite{zhou2021} for a finite-dimensional analogue) rather than sampling $H$ directly.

\begin{figure}[p]
\centering

\begin{minipage}{0.98\textwidth}
\begin{algorithm}[H]\small
\SetAlgoLined
 \textbf{Parameters:} The PDE \eqref{eq:PDE}; choice of DHGM or QHPDE gradient; orthonormal basis $\{e_i\}_{i=1}^\infty$ of $H$; hyperparameters $d,\theta$ of the HGNO; reference probability measure $\mu$ on $H$; number of samples $M$ in $H$; learning rate $\alpha_t$; terminal time $T$\;
 \textbf{Initialize:} HGNO $v^{d,\theta}$; time $t = 0$\;
 Let $\mathcal{F}v(x) = \inprod{LDv(x)}{x} + F(x,v(x),Dv(x),D^2v(x)).$ \\
 ----------------------------------------------------------- \\
 \textit{While $t < T$}:\\
  ----------------------------------------------------------- \\
    \begin{enumerate}[leftmargin=*,nosep]
        \item Sample $\{x_m\}_{m=1}^M$ in $H$ according to $\mu$.
        \item \textit{If using DHGM gradient}: Calculate training (negative) gradient
        \begin{align}\label{algo:dgm_critic_loss}
            \mathrm{Grad} \leftarrow G_{\mathrm{DHGM}}(\theta; \{x_m\}_{m=1}^M) := -\frac{1}{M}\sum_{m=1}^M \mathcal{F}v^{d,\theta}(x_m)\nabla_\theta\mathcal{F}v^{d,\theta}(x_m).
        \end{align} \\
        \textit{Else if using QHPDE gradient}: Calculate training (negative) gradient
        \begin{align}\label{algo:qpde_critic_loss}
            \mathrm{Grad} \leftarrow G_{\mathrm{QHPDE}}(\theta; \{x_m\}_{m=1}^M) := -\frac{1}{M}\sum_{m=1}^M \mathcal{F}v^{d,\theta}(x_m)\nabla_\theta(-v^{d,\theta}(x_m)).
        \end{align}
        \item Update parameters $\theta$ with gradient $\mathrm{Grad}$ at rate $\alpha_t$ using stochastic gradient descent or ADAM \cite{kingma2014}.
        \item Update time $t \leftarrow t+1$\;
    \end{enumerate}
\caption{Deep Hilbert--Galerkin algorithm for PDE \eqref{eq:PDE} on $H$}\label{algo:fixed_control}
\end{algorithm}
\end{minipage}

\vspace{2em}

\begin{minipage}{0.98\textwidth}
\begin{algorithm}[H] \small
\SetAlgoLined
 \textbf{Parameters:} The PDE \eqref{eq:HJB_control}; choice of DHGM or QHPDE gradient; orthonormal bases $\{e_i\}_{i=1}^\infty$ and $\{g_i\}_{i=1}^\infty$ of $H$ and $\tilde U$; hyperparameters $d,\theta,\phi,p$ of the critic and actor HGNOs; reference probability measure $\mu$ on $H$; number of samples $M$ in $H$; learning rates $\alpha_t, \beta_t$; terminal time $T$\;
 \textbf{Initialize:} HGNOs $v^{d,\theta} : H \to \R$ and $u^{d,\phi,p}: H \to U$; time $t = 0$\; 
 Let $\mathcal{F}^{cv}(v,u)(x) =  -\gamma v(x)+ \inprod{A^*Dv(x)}{x} + \inprod{Dv(x)}{b(x,u(x))}+\frac{1}{2}\mathrm{Tr}[\sigma(x,u(x)) Q \sigma^*(x,u(x))D^2v(x)]+l(x,u(x)).$ \\
 ----------------------------------------------------------- \\
 \textit{While $t < T$}:\\
  ----------------------------------------------------------- \\
\textbf{Critic step}: \\
    \begin{enumerate}[leftmargin=*,nosep]
        \item Fix HGNO $u^{d,\phi,p}$. Sample $\{x_m\}_{m=1}^M$ in $H$ according to $\mu$.
        \item \textit{If using DHGM gradient}: Calculate training (negative) gradient
        \begin{equation}\label{algo:ac_dgm_critic_loss}
            \mathrm{Grad}_{\mathrm{Critic}} \leftarrow G_{\mathrm{DHGM}}(\theta, \phi; \{x_m\}_{m=1}^M) := -\frac{1}{M}\sum_{m=1}^M \mathcal{F}^{cv}(v^{d,\theta},u^{d,\phi,p})(x_m) \nabla_\theta \mathcal{F}^{cv}(v^{d,\theta},u^{d,\phi,p})(x_m).
        \end{equation} \\
        \textit{Else if using QHPDE gradient}: Calculate training (negative) gradient
        \begin{equation}\label{algo:ac_qpde_critic_loss}
            \mathrm{Grad}_{\mathrm{Critic}} \leftarrow G_{\mathrm{QHPDE}}(\theta, \phi; \{x_m\}_{m=1}^M) := -\frac{1}{M}\sum_{m=1}^M \mathcal{F}^{cv}(v^{d,\theta},u^{d,\phi,p})(x_m) \nabla_\theta (-v^{d,\theta}(x_m)).
        \end{equation}
        \item Update parameters $\theta$ with gradient $\mathrm{Grad}_{\mathrm{Critic}}$ at rate $\alpha_t$ using stochastic gradient descent or ADAM.
    \end{enumerate}
\textbf{Actor step}: \\
    \begin{enumerate}
        \item Fix HGNO $v^{d,\theta}$. Sample $\{x_m\}_{m=1}^M$ in $H$ according to $\mu$.
        \item Calculate
        \begin{equation}\label{algo:ac_actor_loss}
            \mathrm{Grad}_{\mathrm{Actor}} \leftarrow G_{\mathrm{Actor}}(\theta,\phi; \{x_m\}_{m=1}^M) := -\frac{1}{M} \sum_{m=1}^M \nabla_\phi \mathcal{F}^{cv}(v^{d,\theta}, u^{d,\phi,p})(x_m)
        \end{equation}
        Update parameters $\phi$ with gradient $\mathrm{Grad}_{\mathrm{Actor}}$ at rate $\beta_t$ using stochastic gradient descent or ADAM.
    \end{enumerate}
\textbf{Update time}: $t \leftarrow t + 1$\;
\caption{{Hilbert} actor-critic algorithm for HJB equation \eqref{eq:HJB_control} on $H$}\label{algo:ac}
\end{algorithm}
\end{minipage}
\end{figure}

\section{Numerical tests}\label{sec:numerics}
We now demonstrate using the Deep Hilbert--Galerkin and Hilbert Actor-Critic Methods  to solve  Kolmogorov and HJB equations on Hilbert spaces related to infinite-dimensional stochastic analysis and control problems.

\subsection{Optimal control of deterministic and stochastic heat equations}\label{sec:stochastic_heat}
Consider the controlled stochastic partial differential equation on  $[0,2\pi]$ and time domain $[0,\infty)$ 
    \begin{align} \begin{split}
    \label{eq:stochastic_heat}
        &\frac{\partial x}{\partial t}(t,\xi) = \frac{\partial^2x}{\partial \xi^2}(t,\xi) + u_t(\xi) + \frac{\partial^2 W^Q}{\partial t \partial \xi}, \quad x(t,0) = x(t,2\pi) = 0,\quad x(0,\xi)= x_0(\xi).
    \end{split} \end{align}
    Let $\lambda, \gamma > 0$ and $x_0 \in L^2([0,2\pi]).$ The objective to minimize is
    \begin{align*}
        &J(x_0;u) = \E\brac{\int_0^\infty e^{-\gamma s}\int_0^{2\pi}\paren{x(s,\xi)-\overline{x}(\xi)}^2 + \lambda u_s(\xi)^2d\xi ds}, \quad V(x_0) = \inf_{u \in \mathcal{U}}J(x_0;u).
    \end{align*}
    More precisely, we take $H = \Xi = \tilde{U} = U = L^2([0,2\pi])$ so that the controlled dynamics represent the unique mild solution of \eqref{eq:stochastic_heat} when properly interpreted. That is, $x: [0,\infty) \to L^2([0,2\pi]) = H$ is a stochastic process. The operator $A = \frac{\partial^2}{\partial\xi^2}$ should be understood as densely-defined on the domain $D(A) = H^2([0,2\pi]) \cap H^1_0([0,2\pi]) \subset L^2([0,2\pi])$, and $A = A^* = \frac{\partial^2}{\partial\xi^2}$ generates a strictly-contractive $C^0$-semigroup $\{e^{At}\}_{t\geq0}$ on $H$ (cf.~Example \ref{ex:laplacian} (1)). Then the unique mild solution is in the form of \eqref{eq:mild} with $b(x,u) = u$ and $\sigma(x,u) = I$. {We explicitly note that \eqref{eq:stochastic_heat} does not admit a strong solution in the sense of \eqref{eq:stateSDE_H_intro} (see \cite{da1992stochastic}), making the results of \cite{castro2022} not applicable.}
    
    The relevant HJB equation for this control problem is
    \begin{align}\label{eq:HJB_stochastic_heat}
        -\gamma v + \inprod{A^*Dv}{x} + \inf_{u \in H}\curlbrac{\inprod{Dv}{u} + \modu{x-\bar{x}}^2 + \lambda \modu{u}^2} + \frac{1}{2} \mathrm{Tr}[Q D^2v] = 0, \quad x \in H
    \end{align}
    and a family of Kolmogorov equations for $J(x;u)$ is similarly defined by fixing $u \in H$.
    
    The natural orthonormal basis of $H$ to use for analyzing this problem comes from the eigenmodes of $A^*$, i.e.~$\{e_n\}_{n=1}^\infty = \{\sin(n\xi/2)/\sqrt{\pi}\}_{n=1}^\infty$. This has two advantages. First, it satisfies the condition in Theorem \ref{th:DGM_residual} that $e_n \in D(L) = D(A^*)$ for each $n \in \N$. Further, $A^*$ is a diagonal operator in this basis with eigenvalues $\{-\lambda_n\}_{n=1}^\infty = \{-n^2/4\}_{n=1}^\infty$, allowing for ease of analysis and exact computations. The HJB equation and the Kolmogorov equation for any $u \in H$ both admit analytic solutions that are classical in the sense of Definition \ref{def:classical_sol} and satisfy the other conditions of Theorem \ref{th:DGM_residual}. The statements and derivations of these solutions are provided in Section \ref{sec:stochastic_heat_appendix}.

    \subsubsection{Numerical results}\label{sec:stochastic_heat_numerics}
    
    Algorithms \ref{algo:fixed_control} and \ref{algo:ac} perform very well in the fixed-actor (Kolmogorov) and actor-critic (HJB) problems, respectively. We consider three test problems (two stochastic and one deterministic) in both their Kolmogorov (by fixing $u = 0$) and HJB forms. Each HGNO $v^{d,\theta}$ uses the first $d = 25$ elements of the basis $\{\sin (n\xi/2)/\sqrt{\pi}\}_{n=1}^\infty$ and contains a single hidden layer with 600 neurons. Problems that have a learned control (i.e.~actor) use a HGNO $u^{d,\phi,p}$ with $p = d = 25$ which similarly contains 600 neurons in the hidden layer. Training is done for $T = 2\times 10^6$ iterations. When using DHGM for the critic training gradient, the rates are $\alpha_t = \frac{5}{20+t^{0.5}}$ and $\beta_t = \frac{5}{20+t^{0.75}}.$ When using QHPDE, they are $\alpha_t = \frac{0.05}{20+t^{0.5}}$ and $\beta_t = \frac{0.05}{20+t^{0.75}}.$ In both cases, the ADAM optimizer \cite{kingma2014} is used to schedule parameter updates in Algorithms \ref{algo:fixed_control} and \ref{algo:ac}. The quantities \eqref{algo:dgm_critic_loss} -- \eqref{algo:ac_actor_loss} are calculated with $M = 2000$ randomly sampled points at each step, and each point is sampled to its first $N = 250$ basis elements\footnote{We err on the side of caution in the selection of training hyperparameters so as to best exemplify the approximation capabilities of  HGNOs architecture. Users conscious of computational efficiency may find improved economy with reduced training time or other hyperparameter tweaks.}. All computations are done with CUDA-enabled PyTorch on an Nvidia H100 GPU. The code is available at \url{https://github.com/JacksonHebner/Deep-Hilbert-Galerkin-Methods}.

    We use four accuracy evaluations -- mean error (ME), root mean square error (RMSE), and two notions of relative error (RE1 and RE2). They are defined as
\begin{small}
\begin{align}\label{eq:accuracy_metrics}
    \textrm{ME}(Q,V) &= \frac{1}{K} \sum_{j=1}^K \modu{Q(x_j) - V(x_j)} &\textrm{RMSE}(Q, V) = \sqrt{\frac{1}{K}\sum_{j=1}^K\modu{Q(x_j) - V(x_j)}^2} \\
    \textrm{RE1}(Q, V) &= \frac{1}{K}\sum_{j=1}^K\frac{\modu{Q(x_j) - V(x_j)}}{\modu{V(x_j)}}
    &\textrm{RE2}(Q, V) = \sqrt{\frac{\sum_{j=1}^K\modu{Q(x_j) - V(x_j)}^2}{\sum_{j=1}^K\modu{V(x_j)}^2}}
\end{align}
\end{small}
    where $K = 10^6$ points are sampled from $\mu$ up to the first $N = 250$ basis functions and $\modu{\cdot}$ represents either the absolute value or $L^2([0,2\pi])$-norm for critic and actor evaluations, respectively.

    Additionally, the first test problem serves as examples where Theorem \ref{th:l2_convergence} applies. We report the $L^2(\mu;H)$-norm of the PDE residual and see empirically by comparing against the RMSE that the bound in \eqref{eq:bounded_inverse} is relatively tight. It thus serves as a genuinely helpful check on how well the PDE solution has been learned. We additionally report the accuracy of the gradient and Hessian.\footnote{Due to the high computational cost of calculating the operator norms of large matrices, for these evaluations, we set $K = 10^4$.} For the Hessian, we use the following norms:
    \begin{align*}
        \nor{D^2 v}_{(4;\mathrm{Op})} = \paren{\int_H \nor{D^2v(x)}_{\g(H)}^4 \mu(dx)}^{1/4}, \quad
        \nor{D^2v}_{(4;\mu,\mu)} = \paren{\int_{H\times H} |D^2v(x)h|_H^4\mu(dx)\mu(dh)}^{1/4}.
    \end{align*}
    From the analytic solution (Appendix \ref{sec:stochastic_heat_appendix}), we observe that $D^2v(x) \in S(H)$ is a trace-class, and therefore compact, operator for any $x \in H$. In this very particular case, unlike Remark \ref{rem:density_neural_Sobolev}, universal approximation might be theoretically possible when $S(H)$ is endowed with the operator norm. However, we empirically observe that convergence in operator norm is nevertheless either significantly slower or does not occur at all. An intuitive explanation of this effect may be that the Deep Hilbert--Galerkin Method minimizes the $L^2_\mu$-norm of PDE residual, which includes a trace term but does not directly depend on the operator norm.
    
    In all the problems below, we fix $\lambda  = \gamma = 1$ and $\overline x = 0$.
    
    \paragraph{Trace class covariance noise.}
    We set $Q = \mathrm{diag}(1/n^2)_{n=1}^\infty$ to be the diagonal\footnote{We use $\mathrm{diag}$ as shorthand to indicate how a linear operator $Q : H\to H$ acts in the Dirichlet basis $\{e_n\}_{n=1}^\infty = \{\sin(n\xi/2)\}_{n=1}^\infty$. That is to say, $Q = \mathrm{diag}(a_n)_{n=1}^\infty$ is equivalent to $Qe_n = a_ne_n$ for each $n \in \N$.} covariance kernel. With $u = 0$ fixed, the mild solution of \eqref{eq:stochastic_heat} has as its stationary distribution $\mu_{\mathrm{TCC}} = \mathcal{N}\paren{0, \mathrm{diag}(1/(2n^2\lambda_n))_{n=1}^\infty}.$
    We use $\mu = \mu_{\mathrm{TCC}}$ as the reference training and evaluation measure for all four simulations in this example. In particular, since $\mu_{\mathrm{TCC}}$ is the evaluation measure, Theorem \ref{th:l2_convergence} applies when solving for $J(\cdot; 0).$

    \begin{table}[H]
    \begin{center}
    \begin{tabular}{@{}llllll@{}}
    \toprule
     Gradient & Critic ME & Critic RMSE & Critic RE1 & Critic RE2 & $\nor{\mathcal{F}v^{d,\theta}}_{L^2(H;\mu)}$ \\ \midrule
     DHGM & 0.03481 & 0.03810 & 0.02895 & 0.01326 & 0.08631\\
     QHPDE & 0.01380 & 0.01383 & 9.452e-3 & 4.827e-3 & 0.01610\\
    \bottomrule
    \end{tabular}
    \end{center}
    \caption{Algorithm \ref{algo:fixed_control} solving for $J(\cdot;0)$ with trace class covariance noise in \eqref{eq:stochastic_heat}}
    \end{table}
    
    \begin{table}[H]
    \begin{center}
    \begin{tabular}{@{}llllllllll@{}}
    \toprule
     Gradient & Critic ME & Critic RMSE & Critic RE1 & Critic RE2 & Actor ME & Actor RMSE & Actor RE1 & Actor RE2 \\ \midrule
     DHGM & 1.026 & 1.133 & 0.8448 & 0.5209 & 0.8841 & 1.060 & 2.389 & 1.481\\
     QHPDE & 8.796e-3 & 9.042e-3 & 8.739e-3 & 4.162e-3 & 5.918e-3 & 6.146e-3 & 0.01953 & 8.595e-3\\
    \bottomrule
    \end{tabular}
    \end{center}
    \caption{Algorithm \ref{algo:ac} solving for $V$ and $u^*$ with trace class covariance noise in \eqref{eq:stochastic_heat}}
    \end{table}

    \begin{table}[H]
    \begin{center}
    \begin{tabular}{@{}lllll@{}}
    \toprule
    Problem & $\nor{v - v^{d,\theta}}_{L^4(H;\mu)}$ &  $\nor{D(v^{d,\theta} - v)}_{L^4(H;\mu)}$  & $\nor{D^2(v^{d,\theta}-v)}_{(4;\mu,\mu)}$ & $\nor{D^2(v^{d,\theta}-v)}_{(4;\mathrm{Op})}$ \\ \midrule
    Kolmogorov ($u = 0$) & 9.084e-3 & 0.01762 & 0.02035 & 0.4151 \\
    HJB & 8.875e-3 & $0.01409$ & 0.01647 & $0.3480$\\
    \bottomrule
    \end{tabular}
    \end{center}
    \caption{Accuracy of derivatives learned by Algorithms \ref{algo:fixed_control} and \ref{algo:ac} (QHPDE) for trace class covariance noise in \eqref{eq:stochastic_heat}} \label{tab:derivatives}
    \end{table}

    \paragraph{1D noise}

    We set $Q \in \g(H;H)$ in \eqref{eq:stochastic_heat} to be such that $Q^{1/2}e_1  = \sum_{n=1}^\infty \frac{2\sqrt 2}{n\pi}e_n = \frac{1}{\sqrt{2\pi}}$ and $Q^{1/2} e_n = 0$ for all $n > 1$.
    Thus, the underlying stochastic evolution dynamics are infinite-dimensional, but the extrinsic noise has a low-dimensional structure in the sense that $\mathrm{Rank}(Q^{1/2}) = \mathrm{Rank}(Q) = 1$. We use $\mu = \mu_{\mathrm{WN}} = \mathcal{N}(0, \mathrm{diag}(1/{2\lambda_n})_{n=1}^\infty)$, the stationary distribution of \eqref{eq:stochastic_heat} under white noise (i.e.~covariance operator $\tilde Q = I$), as the reference training and evaluation measure for all four simulations in this example.

    \begin{table}[H]
    \begin{center}
    \begin{tabular}{@{}lllll@{}}
    \toprule
     Gradient & Critic ME & Critic RMSE & Critic RE1 & Critic RE2 \\ \midrule
     DHGM  & 4.909e-3 & 6.329e-3 & 3.738e-3 & 2.214e-3\\
     QHPDE & 4.128e-3 & 4.711e-3 & 3.164e-3 & 1.652e-3\\
    \bottomrule
    \end{tabular}
    \end{center}
    \caption{Algorithm \ref{algo:fixed_control} solving for $J(\cdot;0)$ with 1D noise in \eqref{eq:stochastic_heat}}
    \end{table}

    \begin{table}[H]
    \begin{center}
    \begin{tabular}{@{}lllllllll@{}}
    \toprule
     Gradient & Critic ME & Critic RMSE & Critic RE1 & Critic RE2 & Actor ME & Actor RMSE & Actor RE1 & Actor RE2\\ \midrule
     DHGM & 9.417e-3 & 0.06472 & 8.858e-3 & 0.02973 & 0.03513 & 0.07985 & 0.08747 & 0.1073 \\
     QHPDE & 2.205e-3 & 2.757e-3 & 1.861e-3 & 1.263e-3 & 0.01654 & 0.01689 & 0.04121 & 0.02265 \\
    \bottomrule
    \end{tabular}
    \end{center}
    \caption{Algorithm \ref{algo:ac} solving for $V$ and $u^*$ with 1D noise in \eqref{eq:stochastic_heat}}
    \end{table}
    
    \paragraph{No noise (deterministic heat evolution).}
    We set $Q = 0$. The underlying dynamics of \eqref{eq:stochastic_heat} are thus entirely deterministic and the PDE \eqref{eq:HJB_stochastic_heat} is first-order. We use $\mu = \mu_{\mathrm{TCC}}$ (given in the first example) as the reference training and evaluation measure for all four simulations in this example.

    \begin{table}[H]
    \begin{center}
    \begin{tabular}{@{}lllll@{}}
    \toprule
     Gradient & Critic ME & Critic RMSE & Critic RE1 & Critic RE2 \\ \midrule
     DHGM & 3.791e-3 & 9.171e-3 & 0.01301 & 3.719e-3 \\
     QHPDE & 4.932e-3 & 5.653e-3 & 8.545e-3 & 2.295e-3  \\
    \bottomrule
    \end{tabular}
    \end{center}
    \caption{Algorithm \ref{algo:fixed_control} solving for $J(\cdot;0)$ with no noise ($Q = 0$) in \eqref{eq:stochastic_heat}}
    \end{table}

    \begin{table}[H]
    \begin{center}
    \begin{tabular}{@{}lllllllll@{}}
    \toprule
     Gradient & Critic ME & Critic RMSE & Critic RE1 & Critic RE2 & Actor ME & Actor RMSE & Actor RE1 & Actor RE2 \\ \midrule
     DHGM & 5.049 & 6.768 & 12.19 & 3.605 & 2.312 & 2.976 & 5.503 & 3.996 \\
     QHPDE & 5.140e-3 & 5.766e-3 & 0.1551 & 3.067e-3 & 0.01609 & 0.01642 & 0.04003 & 0.02203 \\
    \bottomrule
    \end{tabular}
    \end{center}
    \caption{Algorithm \ref{algo:ac} solving for $V$ and $u^*$ with no noise ($Q = 0$) in \eqref{eq:stochastic_heat}}
    \end{table}

    Overall, Algorithm \ref{algo:fixed_control} performs well on Kolmogorov problems when utilizing either the DHGM or QHPDE gradient. Algorithm \ref{algo:ac}, on the other hand, performs dramatically better with the QHPDE gradient. It is thus safe to generally prefer the QHPDE variants of Algorithms \ref{algo:fixed_control} and \ref{algo:ac} when solving time-discounted problems in the form \eqref{eq:HJB_control}. Further, Table \ref{tab:derivatives} empirically demonstrates the QHPDE variants of Algorithms \ref{algo:fixed_control} and \ref{algo:ac} learning not just the proper function values, but also gradients and Hessians (when measured in the right norm), as we would expect.
 
\subsection{Optimal control of deterministic and  stochastic Burgers equations}\label{subsec:burgers}
    Formally, consider the stochastic partial differential equation on the space domain $[0,2\pi]$ and time domain $[0,\infty)$ with the following dynamics:
    \begin{align} 
    \label{eq:stochastic_burgers}
        &\frac{\partial x}{\partial t}(t,\xi) = \frac{\partial^2x}{\partial \xi^2}(t,\xi) + x(t,\xi)\frac{\partial x}{\partial \xi}(t,\xi) + u_t(\xi) + \frac{\partial^2 W^Q}{\partial t \partial \xi}, \quad 
        &x(t,0) = x(t,2\pi) = 0,\quad x(0,\xi)= x_0(\xi),
    \end{align}
    Let $\lambda, \gamma > 0$. The objective to minimize is
    \begin{align*}\label{eq:stochastic_burgers_objective}
        J(x_0;u) &= \E\brac{\int_0^\infty e^{-\gamma t}\int_0^{2\pi}\paren{\frac{\partial x}{\partial \xi}(s,\xi)}^2 + \lambda u_s(t,\xi)^2d\xi ds}, \quad
        V(x_0) = \inf_{u \in \mathcal{U}}J(x_0;u),
    \end{align*}
    Like before, we take $H = \Xi = \tilde{U} = U = L^2([0,2\pi]).$ The SPDE \eqref{eq:stochastic_burgers}, known as the stochastic Burgers equation, is substantially more analytically challenging than the stochastic heat equation. In particular, the map $x \mapsto x'' + xx'$ is nonlinear and does not generate a $C_0$-semigroup, meaning that we must work beyond the previous mild SPDE solution framework. 
    Nonetheless, we may define the operators $A : H^2([0,2\pi]) \cap H_0^1([0,2\pi]) \to L^2([0,2\pi])$ and $B : H^1([0,2\pi]) \to L^2([0,2\pi])$ by $Ax = x''$ and $Bx = xx'$ with domains dense in $H$ such that \eqref{eq:stochastic_burgers} can be given the meaning of an SDE on $H$:
    \begin{equation}\label{eq:burgers_sde}
        dX_t = (AX_t + B(X_t) + u_t)dt + dW^Q_t,\quad X_0 = x_0.
    \end{equation}
    The SDE \eqref{eq:burgers_sde} admits a unique solution in a qualified sense, see \cite[Theorem 4.229]{fabbri2017book}. For more analytic details, see \cite[Section 4.9.1]{fabbri2017book}.

    We are interested in a stationary analogue of the HJB equation in e.g.~\cite{fabbri2017book, daprato_debussche2000}, namely,
    \begin{align}\label{eq:stochastic_burgers_hjb}
        -\gamma v + \inprod{Dv}{Ax + B(x)} +{\modu{(-A)^{1/2}x}^2} + \inf_{u \in H}\curlbrac{\inprod{Dv}{u} + \lambda \modu{u}^2} + \frac{1}{2} \mathrm{Tr}[Q D^2v] = 0, \quad x\in H.
    \end{align}
    The HJB equation \eqref{eq:stochastic_burgers_hjb} should be understood in a formal sense and is very challenging. Notably, it contains the nonlinear unbounded operators $B(\cdot),\modu{(-A)^{1/2}(\cdot)}^2,$ in addition to $A$, and does not admit a classical solution in the sense of Definition \ref{def:classical_sol}, but only a suitable mild solution \cite{fabbri2017book, daprato_debussche2000}. Thus, this numerical example is intended as a stress-test case. Algorithm \ref{algo:ac} is {designed} to solve \eqref{eq:stochastic_burgers_hjb}, {and we provide runnable code and trained models for this on GitHub}.\footnote{Available at \url{https://github.com/JacksonHebner/Deep-Hilbert-Galerkin-Methods}.} However, we have no immediate way to evaluate their accuracy. Thus, we limit ourselves here to the resulting Kolmogorov equation for a fixed actor $u$, where accuracy may be evaluated by comparison against a Monte Carlo finite difference scheme. We note that according to this benchmark, Algorithm \ref{algo:fixed_control} succeeds in empirically learning the solutions to Kolmogorov problems beyond the framework established in the theory portion of this paper with remarkable accuracy.

    \subsubsection{Numerical results}
    We provide results for two relatively simple example problems. The HGNOs $v^{d,\theta}$ have $d = 100$ input frequencies and points in $H$ are sampled up to $N = 500$ elements in the basis $\{\sin(n\xi/2) /\sqrt{\pi}\}_{n=1}^\infty$. {Rather than a feedforward neural network, we use the LSTM-like architecture developed in \cite{sirignano2018} containing about $3.5\times$ more parameters than in the previous example.\footnote{In the language of \cite[Section 4.2]{sirignano2018}, we use $L + 1 = 4$ hidden units and $M = 128$ hidden dimensions in each unit.}} Training occurs for $T = 2\times 10^{6}$ iterations at a learning rate $\alpha_t = \frac{0.05}{20 + t^{0.5}}$ with the QHPDE gradient and an ADAM optimizer \cite{kingma2014}. Each training iteration samples $M = 2000$ points from the Hilbert space $H$ according to the distribution $\mu = \mathcal{N}(0, \mathrm{diag}(1/n^4)_{n=1}^\infty)$. All computations are done without need for numerical integration.\footnote{While $B(\cdot)$ is not linear, its action upon functions of the form $x = \sum_{n=1}^N x_n \sin(n\xi/2)/\sqrt{\pi}$ is regular enough to be computed directly in the basis $\{\sin(n\xi/2)/\sqrt{\pi}\}_{n=1}^\infty$. Namely, $$(B(x))_n = \frac{n}{8\sqrt{\pi}} \sum_{m=1}^{n-1}x_m x_{n-m} - \frac{n}{4\sqrt{\pi}}\sum_{m=1}^{N-n} x_m x_{n+m}, \quad n\in\{1,\dots,N\}.$$ We may thus exactly calculate each of the terms in \eqref{eq:stochastic_burgers_hjb} without need for spatial numerical integration.}

    To provide comparison values, we use a Monte Carlo finite difference scheme with 251 points evenly spaced on the interval $[0,2\pi]$. Each time increment is $\Delta t = 10^{-4}$ and there are $10^{5}$ steps in the finite difference scheme. The value given for any starting point $x \in H$ is the average of 50,000 Monte Carlo finite difference simulations. This discretization scheme is stable \cite{alabert2006}. However, it does add some error. This, combined with the fact that the standard error in the Monte Carlo estimate of the value function remains nontrivial even with 50,000 simulations in the stochastic example, means that the ``ground truth'' of Monte Carlo finite difference estimates should not be taken uncritically.

    In the problems below, we fix $\lambda  = \gamma = 1$, $u = 0$, and solve the Kolmogorov equation that \eqref{eq:stochastic_burgers_hjb} reduces to.

    We provide point-by-point comparisons for a selection of interesting functions as well as overall evaluations against ensembles of 100 randomly sampled functions. In particular, for the stochastic example, we sample from both the training measure $\mu$ and the empirical stationary distribution $\mu_{\mathrm{stat}}$ generated by forward-simulating the SPDE \eqref{eq:stochastic_burgers}.\footnote{This is done by representing the solution in the first $N = 500$ dimensions of the basis $\{e_n\}_{n=1}^\infty$ and forward-simulating the dynamics using $\Delta t =  10^{-5}$ for $2\times10^6$ steps, the diagonal representation of $A(\cdot)$, and the above representation of the Burgers operator $B(\cdot)$.}
    
    \paragraph{1D noise} We set $Q \in \g(H;H)$ in \eqref{eq:stochastic_burgers} to be such that $Q^{1/2}e_1  = \sum_{n=1}^\infty \frac{2\sqrt 2}{n\pi}e_n = \frac{1}{\sqrt{2\pi}}$ and $Q^{1/2} e_n = 0$ for all $n > 1$.

    \begin{table}[H]
    \begin{center}
    \begin{tabular}{@{}lllllllll@{}}
    \toprule
     Evaluation point $x \in H$ & $0$ & $\frac{\sin\xi}{\sqrt{\pi}}$ & $\frac{\sin(2\xi)}{\sqrt{\pi}}$ &$\frac{\sin(3\xi)}{\sqrt\pi}$ & $\frac{\xi}{2\pi}(2\pi-\xi)$ & $1-\cos\xi$ & $1-\cos(2\xi)$ & $\frac{1}{\sqrt{2\pi}}$\\ \midrule
     Algorithm \ref{algo:fixed_control} (QHPDE)  & $0.2248$ & $0.5709$ & $0.6634$ & $0.6633$ & $1.720$ & $2.028$ & $2.679$ & $0.4408$ \\
     Monte Carlo Finite Difference & 0.2216 & 0.5663 & 0.6668 & 0.6941 & 1.702 & 2.012 & 2.741 & 0.4417 \\ \bottomrule
    \end{tabular}
    \end{center}
    \caption{Point values of approximated PDE solution for 1D noise in \eqref{eq:stochastic_burgers}}
    \end{table}

    \begin{table}[H]
    \begin{center}
    \begin{tabular}{@{}lllll@{}}
    \toprule
     Sampling regime & Critic ME & Critic RMSE & Critic RE1 & Critic RE2 \\ \midrule
     Training measure $\mu$ & 3.427e-3 & 3.973e-3 & 8.999e-3 & 8.320e-3 \\
     Stationary measure $\mu_{\mathrm{stat}}$ & 3.430e-3 & 4.365e-3 & 8.989e-3 & 7.474e-3 \\ \bottomrule
    \end{tabular}
    \end{center}
    \caption{100 point ensemble accuracy for 1D noise in \eqref{eq:stochastic_burgers}}
    \end{table}

    \paragraph{No noise (deterministic Burgers evolution)} We set $Q = 0$, thus reducing \eqref{eq:stochastic_burgers} to a deterministic evolution and \eqref{eq:stochastic_burgers_hjb} to a first-order PDE on $H$.

    \begin{table}[H]
    \begin{center}
    \begin{tabular}{@{}lllllllll@{}}
    \toprule
     Evaluation point $x \in H$ & $0$ & $\frac{\sin\xi}{\sqrt{\pi}}$ & $\frac{\sin(2\xi)}{\sqrt{\pi}}$ & $\frac{\sin(3\xi)}{\sqrt{\pi}}$& $\frac{\xi}{2\pi}(2\pi-\xi)$ & $1-\cos\xi$ & $1-\cos(2\xi)$ & $\frac{1}{\sqrt{2\pi}}$ \\ \midrule
     Algorithm \ref{algo:fixed_control} (QHPDE) & 9.242e-3 & 0.3434 & 0.4512 & 0.4636 & 1.475 & 1.781 & 2.428 & 0.2213 \\
     Monte Carlo Finite Difference & 0.000 & 0.3338 & 0.4446 & 0.4739 & 1.456 & 1.755 & 2.504 & 0.2184 \\ \bottomrule
    \end{tabular}
    \end{center}
    \caption{Point values of approximated PDE solution for no noise ($Q = 0$) in \eqref{eq:stochastic_burgers}}
    \end{table}

    \begin{table}[H]
    \begin{center}
    \begin{tabular}{@{}lllll@{}}
    \toprule
     Sampling regime & Critic ME & Critic RMSE & Critic RE1 & Critic RE2 \\ \midrule
     Training measure $\mu$ & 9.227e-3 & 9.295e-3 & 0.1604 & 0.03145 \\
     \bottomrule
    \end{tabular}
    \end{center}
    \caption{100 point ensemble accuracy for no noise ($Q = 0$) in \eqref{eq:stochastic_burgers}}
    \end{table}

    The performance of Algorithm \ref{algo:fixed_control} in both the stochastic and deterministic cases is satisfactory. This is best exemplified by the accuracy of the trained models on the 100 point ensembles sampled from the training measure $\mu$ and stationary distribution $\mu_{\mathrm{stat}}$. It is also noteworthy that the value at the point $x(\xi) = \frac{1}{\sqrt{2\pi}}$, which one might expect to be difficult because $x \in H$ falls outside the domain of all the unbounded operators in \eqref{eq:stochastic_burgers}, is learned quite well.
    
\appendix
\section{Appendix}
\subsection{Notation}\label{subsec:notation}
\paragraph{Basic notation.} Throughout the paper $C>0$ indicates a constant, which may change from line to line.

 Given a Banach space $X$ we will denote by $|\cdot|_X$, or simply $|\cdot|$, its norm. 
If $H$ is a Hilbert space, we denote the scalar product and the induced norm, respectively, by 
$\langle \cdot ,\cdot \rangle,$ $ |\cdot|:=(\langle \cdot ,\cdot \rangle)^{1/2}.$
Throughout the whole paper we identify $H$ with its dual $H^*$.
If $X,Y$ are Banach spaces, we denote by $\mathcal L(X,Y)$ the Banach space of linear bounded operators from $X$ to $Y$, endowed with the operator norm $\|L\| :=\sup_{|x|_X=1}|Lx|_Y$ for $ L  \in \mathcal L(X,Y)$.
When $Y=X$ we simply write $\mathcal{L}(X)=\mathcal{L}(X,X).$

Let $H$ be a Hilbert space. We denote by $S(H)$ the space of self-adjoint operators in $\mathcal{L}(H)$, endowed with the norm in $\mathcal L(H)$. We define $S_M:=\{Y\in S(H):\|Y\|\leq M\}$.
If $L \in \mathcal L(H)$ is such that $\langle L x, x\rangle_X \geq 0$ for every $x \in H$, it is called positive or non-negative. If $L \in \mathcal L(H)$ is such that $\langle L x, x\rangle_X > 0$ for every $x \in H$, it is called strictly positive. Let $Y$ be another Hilbert space.
If $H,Y$ are separable, we denote by $\mathcal{L}_1(H,Y)$ the Banach space of trace-class operators endowed with the norm  
$\|L\|_{\mathcal{L}_1(H,Y)}:=\inf \left \{\sum_{i \in \mathbb N} |a_i|_H |b_i|_Y : \{ a_i \}\subset H, \{b_i \}\subset  Y , Lx=\sum_{i \in \mathbb N} b_i  \langle a_i,x \rangle_H, \forall x\in H  \right \} ,  $ $ L \in \mathcal{L}_1(H,Y).$ 
We denote by  $\mathcal{L}_2(H,Y)$ the Hilbert space of Hilbert--Schmidt operators from $H$ to $Y$. The scalar product in $\mathcal{L}_2(H,Y)$ and its induced norm are respectively given by
$\langle L,T \rangle_{\mathcal{L}_2(H,Y)}:= \sum_{k \in \mathbb N} \langle Le_k,Te_k \rangle$ and $   \|L\|_{\mathcal{L}_2(H,Y)}:=\left (\sum_{k \in \mathbb N} |L  e_k|^2 \right)^{1/2} $ for $L,T \in \mathcal{L}_2(H,Y),$
where  $\{ e_k\}$  is any orthonormal basis of $H$.
  When $Y=H$, we simply write $\mathcal{L}_1(H)=\mathcal{L}_1(H,H)$, $\mathcal{L}_2(H)=\mathcal{L}_2(H,H)$. 
  We denote by $\mathcal{L}^+_1(H)\subset \mathcal{L}_1(H)$ the subspace of positive operators. If $L\in \mathcal{L}_1(H)$ we can define its trace by  $\operatorname{Tr}(T):=\sum_{k \in \mathbb{N}}\left\langle T e_k, e_k\right\rangle_H$, where $\operatorname{Tr}(T)$ is independent of the   orthonormal basis $\{ e_k\}$.

\begin{definition}[Coordinate, embedding, and projection operators]\label{def:embedding_op}
    Given an orthonormal basis $\{ e_i\}_{i\in\mathbb N}$ of $H$, we have $x=\sum_{i\in \mathbb N}x_ie_i$, where $x_i=\langle x,e_i \rangle$. For $d\in \mathbb N,$ we
define  the coordinate operator
        $\mathcal{E}_d^{H} \in \mathcal L( H , \R^d)$, $ \mathcal{E}_d^{H}(x) =  (\inprod{x}{e_i})_{i=1}^d,$ the  embedding operator
        $\widehat{\mathcal E}_d^{H}\in \mathcal L(\R^d , H )$, $\widehat{\mathcal{E}}_d^{H}((x_i)_{i=1}^d) = \sum_{i=1}^d x_i e_i,$ and the  projection operator $P_d^H \in \mathcal L( H  )$, $ P_d^H(x) = \sum_{i=1}^d x_i e_i=\widehat{\mathcal{E}}_d^{H}(\mathcal{E}_d^{H} (x)).$
    When there is no ambiguity on the Hilbert space $H$, we will denote them simply by $\mathcal{E}_d,\widehat{\mathcal E}_d,P_d$.
\end{definition}

\paragraph{Spaces $C^k(O;Y)$.}
  Let $X,Y$ be Banach spaces and $O$ be an open subset of  $X$ endowed with the induced topology. For $k \in \mathbb{N}$, we denote by $C^k(O; Y)$ the set of all functions $\phi: X \rightarrow Y$ which are  $k$-times continuously Fréchet differentiable on $X$. Hence, $D^k\phi\in C^0(O; \mathcal L^k(X;Y))$, where $\mathcal L^k(X;Y)$ is the space of $k$-multilinear bounded operators from $X$ to $Y$. When $X=H$ is a Hilbert space and $Y=\mathbb{R}$, then thanks to the Riesz Representation Theorem, we can identify the linear functional $D \phi(\bar{x})$ with the element $y \in H$ such that $\langle y, h\rangle=D\phi(x)h$ for all $h \in X$. The second-order derivative $D^2 u(x)$ can be identified with a symmetric bilinear form in $\mathcal{L}^2(H, \mathbb{R})$, so we can identify $D^2 u(x)$ with the unique $T \in \mathcal{S}(H)$ having the property that
$
\langle T h_1, h_2\rangle=D^2 u(x)(h_1, h_2) $, for all $h_1, h_2 \in H .
$
We will  denote $T:=D^2 u(x)$. Therefore we will see $D^2\phi\in C^0(O,S(H))$. The subscript $_b$ (such as $C^k_b$) will indicate that all derivatives are bounded up to order $k$. We will denote $C(O,Y)=C^0(O,Y)$.

\paragraph{$L^p$ Spaces.}
Let $(\Omega, \mathcal F, \mu)$ be a measure space with positive measure $\mu$ and  $p \in[1,+\infty)$. We define 
    $\|\mu\|_{p}:= \left(\int_\Omega |x|^{p}\mu(dx)\right)^{1/p}$
for $p>0$. If $Y$ is a separable Banach space, we denote by $L^p(\Omega; Y)$ the set  of equivalence classes, with respect to the  equivalence relation $=$ a.e., of  $\mathcal F / \mathcal B(Y)$-measurable functions $ f: \Omega \rightarrow Y$ such that $\int_\Omega |f|_Y^p d \mu<+\infty$, where $\mathcal B(Y)$  denotes the Borel $\sigma$-algebra on $Y$.
$L^p(\Omega;Y;\mu)$ is a Banach space with norm
$\|f\|_{L^p(\Omega;Y;\mu)}:=\left[\int_{\Omega}|f|^p d \mu\right]^{1 / p}.$
If $Y$ is a separable Hilbert space,  $L^2(\Omega; Y;\mu)$ is a separable Hilbert space with scalar product
$\langle f,g \rangle_{L^2(\Omega;Y;\mu)}=\int_\Omega \langle f,g \rangle d \mu.
$
For $f \in L^1(\Omega; Y)$,  its Bochner integral   is well defined and denoted by $\int_\Omega f d \mu$. If $(\Omega, \mathcal F, \mathbb P)$ is a probability space, then for $X \in L^1(\Omega, \mathcal F, \mathbb P)$, we write $\mathbb E[X]$ in place of $\int_\Omega X d \mathbb P$. When $Y=\mathbb R$ we  also use the notation $L^p(\Omega; \mathbb R)=L^p(\Omega)$.

\paragraph{Sobolev spaces.} For $s=k+\alpha$ with $k \in \mathbb{N}_0$ and $\alpha \in(0,1)$, we denote by $W^{s, p}\left(\Omega ; \mathbb{R}^n\right)$ the fractional Sobolev space of functions $u \in W^{k, p}\left(\Omega ; \mathbb{R}^n\right)$ such that
$
\left[D^k u\right]_{W^{\alpha, p}(\Omega)}^p:=\int_{\Omega} \int_{\Omega}\frac{\left|D^k u(x)-D^k u(y)\right|^p}{|x-y|^{d+\alpha p}} d x d y<\infty .
$
When $p=2$, we write $H^s\left(\Omega ; \mathbb{R}^n\right)=W^{s, 2}\left(\Omega ; \mathbb{R}^n\right)$. We denote $ H_0^s\left(\Omega ; \mathbb{R}^n\right):={\overline{C_c^{\infty}\left(\Omega ; \mathbb{R}^n\right)}}^{H^s\left(\Omega ; \mathbb{R}^n\right)}$, where $C_c^{\infty}\left(\Omega ; \mathbb{R}^n\right)$ are functions in $C^{\infty}\left(\Omega ; \mathbb{R}^n\right)$ with compact support. 
When $\Omega=(0,r)^d$, $>0$, we define
$
H_{\mathrm{per}}^s(\Omega;\mathbb R^n)
:=\{u\in H^s(\Omega;\mathbb R^n): u \text{ is periodic in each variable}\}.$ 
As usual, when $n=1$, we drop $\mathbb R^n$ in the notation, e.g.~we write $W^{s, p}\left(\Omega \right)=W^{s, p}\left(\Omega ; \mathbb{R}\right)$.
Given a  measure $\mu$ on $\mathbb R^d$, we  denote by $\|\cdot\|_{W^{k,w}(\mathbb R^d;\mathbb R^p;\mu)}$ the corresponding Sobolev norm for functions in $C^k(\mathbb R^d;\mathbb R^p)$. When $w=2$, we write $\|\cdot\|_{H^{k}(\mathbb R^d;\mathbb R^p;\mu)}$. If $\mu$ is the the Lebesgue measure, we drop it in the notation. If $d\mu=\nu dx$ for some weight function $\nu$, we use the subscript $_\nu$ in the Sobolev spaces and norms.

Let $H$ be a Hilbert space and $\mu$ be a Borel probability measure on $H$. The following standard Sobolev norm on $C^2(H)$ will be too strong for some of our convergence results:
\begin{equation}\label{eq:sobolev_norm}
    \|v\|_{\mathcal W^{2,w}_{\mu}(H)}:=\left(\max\left[\int_H |v(x)|^w\mu(dx), \int_H |D v(x)|^w\mu(dx),   \int_{H} \|D^2v(x)\|^w\mu(dx) \right] \right)^{1/w}.
\end{equation}
To this purpose, let $\mu,\mu'$ be Borel probability measures on $H$ such that  $\|\mu\|_{q},\|\mu'\|_{w} < \infty$ for given $q,w\geq 1$. Then, we define  the following subspace\footnote{We index $C^{2}_{w,q}$ with two indexes $w,q$ (instead of just one) to hint at $L^w$-integrability for a measure $\|\mu\|_{q}< \infty$.} of $C^2(H)$  
\begin{equation}\label{eq:growth_v_UATsobolev} C^{2}_{w,q}(H):=\{v\in C^2(H):\exists C=C(v)>0 :|v(x)|^w+|Dv(x)|^w+\nor{D^2v(x)}^w\leq C(1+|x|^q), \forall x \in H\}\end{equation}
 and we endow it with the following Sobolev-type norm (up to quotients by an appropriate equivalence class)
\begin{equation}\label{eq:sobolev_seminorm}
    \|v\|_{\mathcal  W^{2,w}_{\mu,\mu'}(H)}:=\left(\max\left[\int_H |v(x)|^w\mu(dx), \int_H |D v(x)|^w\mu(dx),   \int_{H\times H} |D^2v(x)h|^w\mu(dx)\mu'(dh) \right] \right)^{1/w}.
\end{equation}
If $\mu,\mu'$ have full support on $H$ then $\|\cdot\|_{\mathcal  W^{2,w}_{\mu,\mu'}(H)}$ is a norm on $C^{2}_{w,q}(H)$\footnote{That is, if $v=v',Dv=Dv',D^2v(\cdot)(\cdot)=D^2v'(\cdot)(\cdot)$, $\mu\otimes\mu'-$a.e., by continuity, equality holds everywhere.}; however, the resulting normed space is not complete and a completion may not behave as well as standard Sobolev spaces; however, a normed space will be enough for us to state density results, e.g, in Theorem \ref{th:UAT_L2_noL}.
This definition allows for different $\mu,\mu'$ for measuring {norms of Hessians in different ways in the variables $x,h$, respectively.
Note that  $$\|v\|_{\mathcal  W^{2,w}_{\mu,\mu'}(H)}\leq \max(1,\|\mu'\|_w)  \  \|v\|_{\mathcal  W^{2,w}_{\mu}(H)}<\infty,\quad \forall v\in C^{2}_{w,q}(H).$$ 

\subsection{Compact-open topologies}
We will often consider the compact-open topology for locally convex topological vector spaces, so we give a general definition and then specialize it when needed. For a short introduction to this topology, see \cite[Appendix B.2]{schmeding2022}.

Let $W$ be a locally convex topological vector space whose topology is generated by a family of seminorms  $\left\{\rho_i\right\}_{i \in I}$. A net $w^\alpha\to w$ if and only if $\rho_{i}(w^\alpha-w)\to 0$ for all $i\in I.$ The family  $\left\{\rho_i\right\}_{i \in I}$ is called directed (or fundamental) if, for all $i_1, \ldots, i_n \in I$, $n\in \mathbb N$ there is a $\gamma \in I$ and $C>0$ so that
$
\max(\rho_{i_k}(x),k=1,\ldots,n)\leq C \rho_\gamma(x), $ for all $x \in Z$, or equivalently, by induction, restricting $n=2$. The existence of such a family  can be assumed without loss of generality \cite[p. 126]{reed_simon}. If $\left\{\rho_i\right\}_{i \in I}$ is a directed family, then $\left\{\left\{x \in W:\rho_i(x)<\varepsilon\right\} : i \in I, \varepsilon>0\right\}$ is a neighborhood base at $0$ and $\left\{\left\{x\in W :\rho_i(x-w)<\varepsilon\right\} : i \in I, \varepsilon>0\right\}$ is a neighborhood base at $w\in W$.
In this case, a subset $\mathcal G$ is dense in $W$ if and only if  for all $w\in W$, $i\in I$, $\epsilon>0$, there exists $g\in \mathcal G$ such that $\rho_{i}(g-w)<\epsilon.$ Let $k\in \mathbb N$ and $W=W_1\times \ldots W_k$, where each $W_j$, $j\leq k$ is a locally convex topological vector space whose topology is generated by a family of directed seminorms $\{\rho_{i_j}^j\}_{i_j \in I_j}$.  We endow $W$ with the product topology generated by the family of directed seminorms $\{\rho_i\}_{i \in I}$, where $\rho_i(w):=\max(\rho_{i_j}^j(w_j),j\leq k)$, $w=(w_1,\ldots,w_k)$, $i=(i_1,\ldots,i_k)\in I:=I_1\times\ldots \times I_k$.
\begin{definition}[Compact-open topology]\label{def:compact-op}Let $X$ be a topological space and $Y$ be a locally convex Hausdorff topological vector space, whose topology is generated by a directed  family of seminorms  $\left\{q_i\right\}_{i \in I}$.
Denote the space of continuous functions from $X$ to $Y$ by $C^0(X;Y)$. 
We define the compact-open topology on $C^0(X;Y)$, as   the  locally convex (Hausdorf) topology generated by the directed family\footnote{To see that $\mathcal P$ is directed, define 
$
\bar K:=K \cup K', 
$
for compacts $K,K'.$
Choose $\gamma$ and $C>0$ from directedness of $\left\{q_i\right\}$. Then, e.g.,
$
p_{K, i_1}(f) =\sup _{x \in K} q_{i_1}(f(x)) \leq  \sup _{x \in \bar K}\left(q_{i_1}(f(x))+q_{i_2}(f(x))\right)\leq C \sup _{x \in \bar K}  q_\gamma(f(x))  =
C p_{\bar K, \gamma}(f)
$.} of seminorms $\mathcal P=\{p_{K,i}:K\subset X\text { compact},i\in I\}$, where 
$p_{K,i}(f):=\sup _{x \in K} q_i(f(x))$, $f\in C^0(X;Y)$. A net $f^\alpha\to f$ in the compact-open topology if and only if $p_{K,i}(f^\alpha-f)\to 0$, for all $K\subset X$ compact, for all $i\in I.$ A set $\mathcal G$ is dense in the compact-open topology of $C^0(X;Y)$ if and only if,  for all $f\in C^0(X;Y)$,  $K\subset X$ compact, $i\in I$, $\epsilon>0$, there exists $g\in \mathcal G$ such that $p_{K,i}(g-f)<\epsilon.$
\end{definition}

Now let $H$ be a separable Hilbert space and recall the notation $S_M=\{Z\in S(H):\|Z\|\leq M\}$.
{\begin{definition}[Topologies on $S(H)$]\label{rem:compact_open_LX}
   In parts of this paper, we endow $\mathcal{L}(H)$, (resp. $S(H)$) with the compact-open topology\footnote{Observe that  $\mathcal{L}(H)$ and  $S(H)$ are subsets of $C^0(H;H)$, and so can inherit the compact-open topology defined earlier.}, i.e.~the locally convex topology generated by  the directed family of seminorms $\mathcal{S}^{co}:=\left\{s^{co}_{K}:K \subset H\right.$ compacts$\},$ with $s^{co}_{K}(Z)=\sup _{h \in K}|Z h|$ for $Z\in \mathcal{L}(H) $ (resp. $S(H)$). Comparing with the strong operator topology, i.e.~the locally convex  topology generated by  the  family of seminorms $\mathcal S^{so}=\{s^{so}_{h},h\in H\},$ with $s^{so}_{h}(Z)=|Z h|$, $Z\in \mathcal{L}(H)$ (resp. $S(H)$), the compact-open topology is stronger.  However, both  topologies are weaker than the topology induced by the operator norm.
   
If $H$ is infinite-dimensional, both these topologies are not sequential. 
However,  the  restrictions of  the compact-open topology and the strong operator topology to bounded subsets of $\mathcal{L}(H)$ (resp. $S(H)$) coincide, since, by  Lemma \ref{lemma1:uniformcompactsfromstrongandequibounded}, they have the same converging nets. If $H$ is separable, when restricted to bounded subsets of $\mathcal{L}(H)$, (resp. $S(H)$), both the compact-open topology and the strong operator topology are metrizable and therefore sequential  \cite[p.35]{simon}, \cite[Section 2]{weaver2004}: Choose a countable dense set $\left(h_k\right)_{k \geq 1}$ in the unit ball $\{h \in H:|h| \leq 1\}$ and define
$
d_M(T, Z):=\sum_{k=1}^{\infty} 2^{-k} |(T-Z) h_k| ,
$
$T, Z \in S_M$,
then $d_M$ is a standard metric on $S_M$ that generates the strong operator topology and the compact-open topology restricted to $S_M$ \cite[Section 2]{weaver2004}.

A net $Z_\alpha \to Z\in \mathcal L(H)$  in the compact-open topology iff $\sup_{x\in K}|(Z_\alpha -Z )h|\to 0$ for all $K\subset H$ compact; $Z_\alpha \to Z$  in the strong operator topology iff $|(Z_\alpha -Z )h|\to 0$ for all $h \in  H$.
\end{definition}}
\begin{definition}[Topologies on $C^0(H;S(H))$]\label{rem:compact_open_C_SH}
    Consider the space $C^0(H;S(H))$. The standard compact-open topology on this space is the one  generated by  the directed family of seminorms $\mathcal{P}^{co}:=\{ p^{co}_{K}: K \subset H\text{ compact}\},$ with $ p^{co}_{K}(Z)=\sup _{x \in K}\|Z(x) \|$ for $Z\in C^0(H ; S(H)) $. However, this topology will be too strong for some of our convergence results.
    
    To our purposes, we may endow $C^0(H;S(H))$ with the (weaker) compact-open topology, where  $S(H)$ is also endowed with the compact-open topology (see Definition \ref{rem:compact_open_LX}. By Definition \ref{def:compact-op}, the resulting topology is the one generated by  the directed family of seminorms $\mathcal{P}^{COCO}:=\{p^{COCO}_{K, K'}: K, K' \subset H]\text{ compact}\}$ with $p^{COCO}_{K, K'}(Z)=\sup _{x \in K}s^{CO}_{K'}(Z(x))=\sup _{x \in K, h \in K'}|Z(x) h|$ for $Z\in C^0(H ; S(H))$. Equivalently, this topology can be constructed by considering the subspace compact-open topology of $ C^0(H;S(H))\subset C^0(H\times H;H)$.
    \end{definition}
    We  extend the previous idea to define a  compact-open topology on $C^2(H)$ differing from the standard one generated by the operator norm on $S(H)$. 
\begin{definition}[Topologies on $C^2(H)$]\label{rem:compact_open_C2}
The usual compact-open topology on $C^2(H)$ is defined by the one generated by  the directed family of seminorms $ {\mathbf {\mathcal {\mathbf P}}}^{CO}=\{{\mathbf p}_{K}^{CO}=\max( p^{0,CO}_{K}, p^{1,CO}_{K}, p^{2,CO}_{K}): K \subset H \text{ compact}\}$, where $ p^{0,CO}_{K}(v)=\sup _{x \in K}|v(x)|$, $p^{1,CO}_{K}(v)= \sup _{x \in K}|D v(x)|$, and $p^{2,CO}_{K}(v)= \sup _{x \in K}\left\|D^2 v(x)\right\|$. However, this will be too strong for our needs. 

 For our purposes, we will often endow $C^2(H)$ with following weaker compact-open topology generated by  the directed family of seminorms ${\mathbf {\mathcal {\mathbf P}}}^{COCO}=\{\mathbf p^{COCO}_{K,K'}=\max(p^{0,CO}_{K},p^{1,CO}_{K},p^{2,COCO}_{K,K'}): K, K' \subset H \text{ compact}\}$, with $p^{0,CO}_{K}(v)=\sup _{x \in K}|v(x)|$, $p^{1,CO}_{K}(v)= \sup _{x \in K}|D v(x)|$, and $p^{2,COCO}_{K,K'}(v)= \sup _{x \in K, h \in K'}\left|D^2 v(x) h\right|$.
\end{definition}
   Notice that the topology generated by seminorms $\mathcal{P}^{COCO}$ is stronger than the topology described in \cite[Remark 2.2 (c)]{schmeding2022} (with $k=0,1,2$ and $M=E=H, N=\mathbb R$ there). The one there, suitably adapted to $C^2(H)$, would correspond to the weaker topology described by the maximum of seminorms $p^{0,CO}_{K}(v)=\sup _{x \in K}|v(x)|$, $p^{1,COCO}_{K,K'}(v)= \sup _{x \in K,h'\in K'}|\langle D v(x),h'\rangle |$, and $p^{2,COCOCO}_{K,K',K''}(v)= \sup _{x \in K, h' \in K',h''\in K''}\left|\langle D^2 v(x) h',h''\rangle\right|$. The key difference is that, since we are on Hilbert spaces,  we use Riesz theorem for identifying the gradient and the Hessians, respectively, as elements of $H$ and $S(H)$ and therefore we can use directly the norm in $H$ and not the absolute value of scalar products. This  yields a stronger topology, which is compatible with Assumptions \ref{ass:F_estimates}, \ref{ass:F_seq_cont}.
\begin{small}
\subsection{Functional analysis facts}
\begin{lemma}\label{lemma1:uniformcompactsfromstrongandequibounded}
    Let $X, Y$ be Banach spaces and $\{T_\alpha\}\subset  L(X, Y)$ be a net such that\footnote{Of course, by the Banach--Steinhaus theorem it is enough to check that $\sup_{\alpha }|T_\alpha h|<\infty$, for any $h\in H.$} $\sup_{\alpha }\|T_\alpha \|<\infty$. If $T_\alpha \to T\in L(X, Y)$ in the strong operator topology, then $T_\alpha  \to T$ converges in the compact-open topology.

Therefore, by the Banach--Steinhaus theorem, if a sequence $\{T_n\}\subset  L(X, Y)$,  $T_n \to T \in L(X, Y)$ in the strong operator topology, then $T_n \to T$ in the compact-open topology.
\end{lemma}
\begin{proof}
    The net $T_\alpha$ is equicontinuous (since $|T_\alpha(x-y)|\leq C|x-y|$, $C=\sup_{\alpha}\|T_\alpha\|$) and pointwise convergent. Therefore it converges uniformly on compacts. The first claim follows. The second follows from the first, since by the Banach--Steinhaus theorem, we have $\sup_{n}\|T_n\|<\infty$.
\end{proof}
\begin{remark}\label{rem:relatively-compact_Pd-I}
    For any compact $K \subset H$, the set $\bigcup_{d \in\mathbb N} P_d(K)={\{P_d(x),x \in K, d \in \mathbb N\}}$ is relatively compact, i.e.~its closure, denoted $\tilde K := \overline{\bigcup_{d=1}^\infty P_d(K)}$, is compact (e.g.~see \cite[Lemma B.77]{fabbri2017book}). Note that $\tilde K \supset K$.
\end{remark}
A related result is the following:
\begin{lemma}\label{rem:relatively-compact_metric}Let $(X,d)$ be a metric space and  $K\subset X$ be a compact set. Let $f^n,f:X\to X$ continuous  such that $f^n\to f$ uniformly on $K$. Then:
\begin{enumerate}[leftmargin=*,nosep]
    \item  The set $C=\bigcup_{n \in\mathbb N} f^n(K)={\{f^n(x),x \in K, n \in \mathbb N\}}$ is relatively compact, i.e.~its closure $\overline C$ is compact. Clearly, $\overline C \supset f(K)$
    \item Let $(Y,m)$ be another metric space and let $F:X\to Y$ continuous. Then, for all $\epsilon>0$, there exists $N$ such that  $\sup_{x\in K}m(F(f^n(x)),F(f(x)))<\epsilon$, for all $n\geq N$.
\end{enumerate}
\end{lemma}\begin{proof}
Claim 1. Since metric spaces are sequential spaces, we show that $C$ is  relatively sequentially compact.  Let $\left\{y_j\right\}_{j\in \mathbb N}\subset \bigcup_{n \in\mathbb N} f^n(K)$, i.e.~we have $y_j=f^{n_j} ( x_j)$ for  $\{x_j\}\subset K$ and $\{n_j\}\subset \mathbb N$. From compactness of $K$, there exists a subsequence $j_k$ and $\bar{x} \in K$ such that
$
x_{j_k} \rightarrow \bar{x}, \text { as } k \rightarrow+\infty .
$
Next we have two cases.

(i) If there exists $\bar n \in \mathbb N$ such that $n_{j_k} \leq \bar n$, for all $j_k$,  we can suppose $n_{j_k}=\bar n$, eventually in $k$;  then, by continuity,
$
\lim _{k \rightarrow+\infty} y^{j_k}=\lim _{k \rightarrow+\infty} f^{n_{j_k}} (x_{j_k})=\lim _{k \rightarrow+\infty} f^{\bar{n}} (x_{j_k})=f^{\bar{n}} (\bar{x} ).
$

(ii) Otherwise, let us suppose $\lim _{k \rightarrow+\infty} n_{j_k}=+\infty$. Then by uniform convergence of $f^n$ over $K$ and continuity of $f$, we have 
$
|y^{j_k}-f(\bar x)| =\left|f^{n_{j_k}}( x_{j_k})-f(\bar{x})\right| \leq\left|f^{n_{j_k}}( x_{j_k})-f( x_{j_k})\right|+\left|f( x_{j_k}) -f(\bar{x})\right|    \xrightarrow{k \rightarrow+\infty} 0  .
$
The claim of 1 follows.

Claim 2. 
Let $\epsilon>0$. By uniform continuity of $F$ on the compact set $\overline C\supset f(K)$, there exists $\delta>0$ such that if $y,z\in \overline C$ are such that  $d(y,z)<\delta$, then $m(F(y),F(z))<\epsilon$. By uniform convergence of  $f^n\to f$ over $K$, there exists $N$ such that for all $n\geq N$, we have $d(f^n(x),f(x))<\delta$, for all $x\in K$. Then $m(F(f^n(x)),F(f(x)))<\epsilon$, for all $n\geq N,$ $x\in K.$ 
\end{proof}
\subsection{Linear unbounded operators and $C_0$-semigroups}\label{app:semigroups}
Throughout the subsection, let $Y$ be a  Banach space.	
Consider a linear, possibly unbounded, $A \colon  D(A) \subset Y \to Y$. We recall that the domain is a crucial part of the definition of the operator $A$, as  different domains yield different realizations of the operators,  e.g.~see Example \ref{ex:laplacian}.

\begin{definition}[$C_0$-semigroup]\label{def:C0_semigroup}A family
$
\{S(t)\}_{t\geq 0}\subset \mathcal L(X) 
$
is called a $C_0$-semigroup   (or strongly continuous semigroup) of linear bounded operators on $Y$ if: $S(0)=I$;  for every  $ t, \tau \geq 0$ it holds that
  $S(t+\tau)=S(t) S(\tau)$; for every $x \in Y$ it holds that
$ \lim _{t \rightarrow 0} S(t) x =x.$

A $C_0$-semigroup $S(t)$ is $\omega$-contractive if $\|e^{At}\|\leq e^{\omega t}$ for $\omega \in \mathbb R$. If $\omega=0$ it is called a contractive $C_0$-semigroup, if $\omega<0$, it is called a strictly contractive $C_0$-semigroup.
\end{definition}
Let $S(t)$ be a $C_0$-semigroup of linear bounded operators on $Y$. 
\begin{definition}[Infinitesimal generator]\label{def:infinitesimal_generator}
The infinitesimal generator of $S(t)$  is the linear (unbounded) operator $A \colon  D(A) \subset Y \to Y$ defined by
$A x=\lim _{t \rightarrow 0^+} \frac{S(t) x-x}{t}  , $ $ D(A)=\left\{x \in Y: \ \lim _{t \rightarrow 0} \frac{S(t) x-x}{t} \textit{ is well-defined} \right\}.$
\end{definition}
We will use then the (formal) notation
$
S(t)=e^{t A} .
$
We recall that the infinitesimal generator $A$ turns out to be a closed, densely defined operator which characterizes the $C_0$-semigroup. For such results, for further properties of $C_0$-semigroups and infinitesimal generators, and for results of generation of $C_0$-semigroups (e.g.~the Hille--Yosida theorem or Lumer--Philips theorem) we refer the reader, e.g., to  \cite{engel_nagel}.
\paragraph{Linear evolution equations.}
Consider the following autonomous linear Cauchy problem in a Banach space $Y$:
\begin{equation}\label{eq:linear_evolution_eq}
    X^{\prime}(t)=A X(t), \quad \forall t \geq 0, \quad 
X(0)=x_0 \in X
\end{equation}
where $A: D(A) \subset Y \rightarrow Y$ is the infinitesimal generator of a $C_0$-semigroup $e^{At}$ on $Y.$ Then $x(t):=e^{t A}x_0$ is a mild solution of \eqref{eq:linear_evolution_eq}. If $x_0\in D(A)$, then it  is the unique classical solution to the Cauchy problem \eqref{eq:linear_evolution_eq}, i.e.~$X \in C^1([0, \infty) ; Y), X(t) \in D(A), $ for all $ t \geq 0$ and \eqref{eq:linear_evolution_eq} is satisfied. 

The above  covers many important models, including most important differential operators. We discuss the case of the Laplacian.
\begin{example}\label{ex:laplacian}
    Let $\Omega \subset \mathbb{R}^n$ be a bounded domain with sufficiently smooth boundary $\partial \Omega$. Consider the heat equation
$
\partial_t x(t, \xi)=\Delta x(t, \xi), $ $(t, \xi) \in(0, \infty) \times \Omega,$ $ x(0, \xi)=u_0(\xi), $ $ \xi \in \Omega
$
with boundary conditions specified below. Then the following different realizations of the Laplacian generate three different $C_0$-semigroups of contractions (moreover, they turn out to be analytic semigroups):
\begin{enumerate}[leftmargin=*]
    \item Homogeneous Dirichlet boundary conditions, i.e.
$
x(t, \xi)=0, $ $(t, \xi) \in(0, \infty) \times \partial \Omega .
$
Then we set $Y=L^2(\Omega)$,
$
A=\Delta, $ $ D(A)=H^2(\Omega) \cap H_0^1(\Omega) .
$
\item Homogeneous Neumann boundary conditions, i.e.~$
\frac{\partial u}{\partial n}(t, \xi)=0, $ $(t, \xi) \in(0, \infty) \times \partial \Omega,
$
where $\nu$ denotes the outward unit normal. Then we set $Y=L^2(\Omega)$, $
A=\Delta, $ $ D(A)=\left\{x\in H^2(\Omega): \partial_\nu x=0 \text { on } \partial \Omega\right\} .
$
\item Let $\Omega=(0, r)^n$, $r>0$, and impose periodic boundary conditions in each variable, i.e.
$
x\left(t, \xi+r e_i\right)=x(t, \xi), $ $ i=1, \ldots, n ,
$ with $e_i$ canonical basis of $\mathbb R^n$.
Then we set 
$
Y=L^2(\Omega), $ $ A=\Delta,
$ $ D(A)=H_{\mathrm{per}}^2(\Omega).$
\end{enumerate}
\end{example}

\subsection{Gaussian measures on Hilbert spaces}
\begin{definition}\label{def:gaussian}
    Let $(H, \langle\cdot,\cdot\rangle)$ be a separable Hilbert space and $(H, \mathcal{B}(H))$ be its Borel measurable space. A measure space $(H, \mathcal{B}(H), \mu)$ is called Gaussian if for any $x \in H$, the map $\inprod{\cdot }{x} : y \mapsto \inprod{y}{x} \in \R$ has a pushforward that is Gaussian. That is to say, $\inprod{\cdot}{x}_{\#}\mu$ is a Gaussian distribution on $(\R, \mathcal{B}(\R))$ for any $x \in H$.
\end{definition}

\begin{proposition}
    For a Gaussian random variable $X$ with distribution $\mu$ on $H$ to be well-defined, there must exist a unique $m \in H$ and a unique positive semi-definite, trace-class operator $Q \in \g(H;H)$ diagonalized by an orthonormal eigenbasis $\{e_i\}_{i=1}^\infty$ with eigenvalues $\{\lambda_i\}_{i=1}^\infty$ such that $X \overset{\text{dist}}{=} m + \sum_{i=1}^\infty \xi_i\sqrt{\lambda_i}e_i,$
    where $\{\xi_i\}_{i=1}^\infty$ are i.i.d. standard Gaussians on $\R$. We call $m \in H$ the mean and $Q \in \g(H;H)$ the covariance.
\end{proposition}
To sample from $\mu$ on a computer, it suffices to know the mean $m \in H$ and a diagonalization $\{e_i\}_{i=1}^\infty, \{\lambda_i\}_{i=1}^\infty$ of $Q \in \g(H;H)$. We then select an $N \in \N$ large, sample $\{\xi\}_{i=1}^N$ i.i.d. standard Gaussians on $\R$, and approximate $X \approx m + \sum_{i=1}^N \xi_i \sqrt{\lambda_i}e_i.$
We add that Gaussian measures on Hilbert spaces have finite moments, which is helpful for verifying the regularity conditions of e.g.~Theorem \ref{th:DGM_residual}.
\begin{proposition}
    Let $\mu$ be a Gaussian measure on $H$. Then for any $q > 0$,
        $\nor{\mu}_q = \paren{\int \modu{x}_H^qd\mu(x)}^{1/q} < \infty.$
\end{proposition}
For more details about Gaussian measures on Hilbert spaces, we refer to \cite[Chapter 2]{da1992stochastic} and \cite{fabbri2017book,liu_rockner}.

\subsection{Closed-form solution of controlled stochastic heat equation}\label{sec:stochastic_heat_appendix}
    Let $x_n = \inprod{x}{e_n}, f_n = \inprod{x_0}{e_n}$, and $\overline{x}_n = \inprod{\overline x}{e_n}$ for each $n \in \N$. Denote the noise in each eigenmode by $\{\sigma_n\}_{n=1}^\infty$ (i.e.~$\sigma_n = \sqrt{\sum_{i=1}^\infty \inprod{Qe_i}{e_n}}$ for each $n \in \N$).
    
    We begin with the Kolmogorov problem. Let $u \in L^2([0,2\pi])$ be fixed and set $u_n = \inprod{u}{e_n}$ for each $n \in \N$. Then each coefficient $x_n(t) = \inprod{x(t)}{e_n}_{L^2}$ in the eigenbasis expansion of the mild SPDE solution $x(t)\in L^2([0,2\pi])$ follows the Ornstein--Uhlenbeck process
\begin{align*}
    &dx_n(t) = (-\lambda_n x_n(t) + u_n)dt + \sigma_ndB_t^{(n)}, \qquad x_n(0) = f_n\quad  \textrm{i.e. } x_n(t) = e^{-\lambda_n t}f_n + \paren{1-e^{-\lambda_nt}}\frac{u_n}{\lambda_n} +  \sigma_n \int_0^t e^{-\lambda_n(t-s)}dB_s^{(n)}
\end{align*}
for a Brownian motion $B_t^{(n)}$. Using the It\^{o} isometry, we have
\begin{align*}
    \E\brac{(x_n(t)-\overline{x}_n)^2} &= \paren{e^{-\lambda_n t}f_n + (1-e^{-\lambda_nt})\frac{u_n}{\lambda_n} - \overline{x}_n}^2 + \sigma_n^2 \int_0^t e^{-2\lambda_n(t-s)}ds \\
    & = \paren{e^{-\lambda_n t}f_n + (1-e^{-\lambda_nt})\frac{u_n}{\lambda_n} - \overline{x}_n}^2 + \frac{1-e^{-2\lambda_nt}}{2\lambda_n}\sigma_n^2,\\
 J(x_0;u) &= \sum_{n=1}^\infty \int_0^\infty e^{-\gamma t} \paren{\paren{e^{-\lambda_n t}f_n + (1-e^{-\lambda_nt})\frac{u_n}{\lambda_n} - \overline{x}_n}^2 + \frac{1 - e^{-2\lambda_nt}}{2\lambda_n}\sigma_n^2 + \lambda u_n^2}dt \\
    & = \sum_{n=1}^\infty \brac{\frac{(f_n - \frac{u_n}{\lambda_n})^2}{\gamma+2\lambda_n} + \frac{2(f_n-\frac{u_n}{\lambda_n})(\frac{u_n}{\lambda_n}-\overline x_n)}{\gamma+\lambda_n}+\frac{(\frac{u_n}{\lambda_n}-\overline x_n)^2 + \lambda u_n^2}{\gamma} + \frac{\sigma_n^2}{\gamma(\gamma+2\lambda_n)}}.
\end{align*}
Since $\{\overline{x}_n\}_{n=1}^\infty, \{u_n\}_{n=1}^\infty$ are square summable and $\lambda_n = n^2/4$, it is easy to see that $J(\cdot;u) \in C^2(H)$. Further, dominated convergence allows for interchanging the second derivative $A^*$ and $DJ(\cdot;u)$, from which we see that $A^*DJ(\cdot;u)$ is continuous. Thus, $J(\cdot;u)$ is classical. Further, Theorem \ref{th:DGM_residual} applies. Assumptions \ref{ass:F_estimates} and \ref{ass:F_seq_cont} hold, \eqref{ass:APdDv_measure} is true by Condition 3 in Section \ref{subsec:UAT_unbounded} with $k = 2$, $q = 8$, and $\mu$ any Gaussian measure, and \eqref{eq:est_classical_sol} is true with $k = 2$.\end{small}

\begin{small}This line of reasoning can also be extended to obtain a solution to the full HJB equation (i.e.~considering the infimum over controls rather than fixing a control). Namely, each eigenspace/mode $n$ of the system corresponds to a 1-dimensional LQR problem. It is thus natural to formulate a quadratic ansatz for each mode and sum these to obtain the overall answer. In particular, for each mode $n$, we have the dynamics, cost, and HJB equation
\begin{align*}
    &dx_n(t) = (-\lambda_n x_n(t) +  u_n)dt + \sigma_ndB_t^{(n)},\quad v_n(z_n) = \int_0^\infty e^{-\gamma t}\paren{(x_n(t) - \overline{x}_n)^2 + \lambda u_n^2}dt,\\
    & -\gamma v_n(x_n) + \inf_{u_n \in \R} \curlbrac{(x_n - \overline{x}_n)^2 + \lambda u_n^2 + v_n'(x_n)(-\lambda_n x_n + u_n) + \frac{1}{2}\sigma_n^2v_n''(x_n)}=0.
\end{align*}
The first-order condition for $u_n$ gives that $u_n^*(x_n) = \frac{-1}{2\lambda}v_n'(x_n).$ Thus
\begin{equation*}
    0 = -\gamma v_n(x_n) + (x_n - \overline{x}_n)^2 + \frac{1}{4\lambda}(v_n'(x_n))^2 + v_n'(x_n)\paren{-\lambda_nx_n - \frac{1}{2\lambda}v_n'(x_n)} + \frac{1}{2}\sigma_n^2v_n''(x_n).
\end{equation*}
We use the quadratic ansatz $v_n(x_n) = M_nx_n^2 + Q_nx_n + R_n.$
Substituting this ansatz and taking the first and second derivatives yields a system of three equations
\begin{align*}
    -\frac{M_n^2}{\lambda} - (2\lambda_n + \gamma) M_n + 1 = 0, \quad  Q_n\paren{\lambda_n+\gamma+\frac{M_n}{\lambda}} + 2\overline{x}_n = 0, \quad
    -\gamma R_n + M_n\sigma_n^2 + \overline{x}_n - \frac{Q_n^2}{4\lambda} &= 0.
\end{align*}
This system admits a straightforward algebraic solution
\begin{align*}
    M_n &= \frac{2\lambda}{\lambda(2\lambda_n+\gamma)+\sqrt{\lambda^2(2\lambda_n+\gamma)^2 + 4\lambda}},\quad 
    Q_n = \frac{-2\overline{x}_n}{\lambda_n + \gamma + M_n/\lambda},\quad
    R_n = \frac{1}{\gamma}\paren{M_n\sigma_n^2 + \overline{x}_n^2 - \frac{Q_n^2}{4\lambda}}.
\end{align*}
Then
\begin{equation*}
    u^*(x)(\xi) = \sum_{n=1}^\infty u_n^*(x_n)\sin\paren{\frac{n\xi}{2}} \quad \text{ and } \quad V(x) = \sum_{n=1}^\infty M_nx_n^2 + Q_nx_n + R_n,
\end{equation*}
with $u_n^*(x_n) = \frac{-M_n}{\lambda}x_n - \frac{Q_n}{2\lambda}.$
This makes sense for any reasonable noise structure\footnote{For example, cylindrical/white noise, infinite-dimensional trace class covariance noise, and finite-dimensional noise all work because they each imply the existence of a uniform bound $\modu{\sigma_n} < K$ for all $n \geq 1$, and so $\sum_{n=1}^\infty R_n < \infty.$}. Further, the solution $V$ is classical. We see immediately that $V \in C^2(H)$, and $M_n = O(n^{-2})$ and $Q_n = O(n^{-2})$, which implies that
\begin{align*}
    A^* D V(x) &= A^* \sum_{n=1}^\infty (2M_n x_n + Q_n)e_n = \sum_{n=1}^\infty \frac{n^2}{4}(2M_nx_n + Q_n)e_n
\end{align*}
is continuous. Like with the Kolmogorov problem, Theorem \ref{th:DGM_residual} applies, since Assumptions \ref{ass:F_estimates} and \ref{ass:F_seq_cont} hold, \eqref{ass:APdDv_measure} is true by Condition 3 in Section \ref{subsec:UAT_unbounded} with $k = 2$, $q = 8$, and $\mu$ any Gaussian measure, and \eqref{eq:est_classical_sol} is true with $k = 2$.

\paragraph{Acknowledgements.} The authors are grateful to Andrzej Święch for helpful comments regarding Assumption \ref{ass:F_seq_cont}.

 Samuel N. Cohen, Jackson Hebner, Justin Sirignano acknowledge the support of His Majesty's Government.  S.N.C.\ also acknowledges the support of the UKRI Prosperity Partnership Scheme (FAIR) under EPSRC Grant EP/V056883/1, and EPSRC Grant EP/Y028872/1 (Mathematical Foundations of Intelligence: An Erlangen Programme for AI). Filippo de Feo acknowledges funding by the Deutsche Forschungsgemeinschaft (DFG, German Research Foundation) – CRC/TRR 388 ``Rough Analysis, Stochastic Dynamics and Related Fields'' – Project ID 516748464.

\bibliography{refs}
\bibliographystyle{plain}
\end{small}
\end{document}